\documentclass{article}

\usepackage[english]{babel}

\usepackage[letterpaper,top=2cm,bottom=2cm,left=3cm,right=3cm,marginparwidth=1.75cm]{geometry}

\usepackage{authblk}
\usepackage{times}
\usepackage{soul}
\usepackage{url}
\usepackage[hidelinks]{hyperref}
\usepackage[utf8]{inputenc}
\usepackage[small]{caption}
\usepackage{graphicx}
\usepackage{amsmath}
\usepackage{amsthm}
\usepackage{booktabs}
\usepackage{algorithm}
\usepackage{algorithmic}

\newtheorem{theorem}{Theorem}[section]

\newtheorem{lemma}{Lemma}[section]

\newcommand{\mdp}{\mathcal{M}}
\newcommand{\state}{\mathcal{S}}
\newcommand{\action}{\mathcal{A}}
\newcommand{\reward}{R}

\def\approxcorrect{\checkmark\kern-1.1ex\raisebox{.89ex}{$\times$}}

\newcommand{\reals}{{\mathbb{R}}}

\newcommand{\Expect}{\mathbb{E{}}}

\newcommand{\argmin}{\mathop{\rm argmin}}
\newcommand{\argmax}{\mathop{\rm argmax}}
\newtheorem{asmp}{Assumption}[section]
\newtheorem{defn}{Definition}[section]

\usepackage{amsmath,amsfonts,bm}

\def\eqref#1{equation~\ref{#1}}
\def\1{\bm{1}}

\newcommand{\ga}{\gamma}

\DeclareMathAlphabet{\mathsfit}{\encodingdefault}{\sfdefault}{m}{sl}
\SetMathAlphabet{\mathsfit}{bold}{\encodingdefault}{\sfdefault}{bx}{n}

\usepackage{float}
\usepackage{color}
\usepackage{float}
\usepackage{xcolor}
\usepackage{amsthm}
\usepackage{pdfpages}
\usepackage[square,numbers]{natbib}
\title{Differentially Private Temporal Difference Learning with Stochastic  Nonconvex-Strongly-Concave Optimization}
\author[1]{Canzhe Zhao}
\author[1]{Yanjie Ze}
\author[2]{Jing Dong}
\author[2]{Baoxiang Wang}
\author[1]{Shuai Li }

\affil[1]{Shanghai Jiao Tong University}
\affil[2]{The Chinese University of Hong Kong, Shenzhen}

\affil[ ]{\{canzhezhao,zeyanjie,shuaili8\}@sjtu.edu.cn, jingdong@link.cuhk.edu.cn, bxiangwang@cuhk.edu.cn}
\date{}

\begin{document}
\maketitle

\begin{abstract}
Temporal difference (TD) learning is a widely used method to evaluate policies in reinforcement learning.
While many TD learning methods
have been developed in recent years, 
little attention has been paid to preserving privacy and
most of the existing approaches might face the concerns of data privacy  from users.
To enable complex representative abilities of policies, 
in this paper,
we consider preserving privacy in TD learning with nonlinear value function approximation.
This is challenging because  such
a nonlinear problem
is usually studied in the formulation of
stochastic nonconvex-strongly-concave optimization to gain finite-sample analysis,
which would require simultaneously preserving the privacy on primal and dual sides.
To this end, 
we employ a momentum-based stochastic gradient descent ascent to achieve a 
single-timescale algorithm, and achieve a good trade-off between meaningful privacy and utility guarantees of both the primal and dual sides by perturbing the gradients on both sides using well-calibrated Gaussian noises.
As a result, our DPTD algorithm could provide $(\epsilon,\delta)$-differential privacy (DP) guarantee for the sensitive information encoded in transitions and retain the original power of TD learning,
with the utility upper bounded by 
$\widetilde{\mathcal{O}}(\frac{(d\log(1/\delta))^{1/8}}{(n\epsilon)^{1/4}})$
\footnote{The tilde in this paper hides the log factor.},
where $n$ is the trajectory length and $d$ is the dimension. 
Extensive experiments conducted in OpenAI Gym show the advantages of our proposed algorithm. 
\end{abstract}

\section{Introduction}
% \zhao{
% \begin{itemize}
%     \item P1:RL is successful + DP
%     \item P2:rl-pe-td (
%     little TD methods pay attention to privacy preserving + why privacy preserving is important (concrete examples))
%     \item P3:Our method + our results
% \end{itemize}
% }

Reinforcement learning  (RL) has shown great success in a series of scenarios such as robot control tasks, planning tasks and games 
% \cite{peters2006policy,silver2014deterministic,mnih2015human,haarnoja2018soft,sutton2018reinforcement}. 
\cite{silver2014deterministic,haarnoja2018soft,sutton2018reinforcement}. 
However, despite their superior empirical performance, 
most of these works do not consider privacy concerns regarding user data and
many applications of RL algorithms are hindered due to data leakage \cite{balle2016differentially}. As a motivating example, in medical research, users' treatment records should remain confidential while RL policies are trained upon them. 
Without considering the data privacy, previous works  have shown that
the user historical information can be inferred by recursively interacting with the released policies 
% \cite{zhang2020secret,wu2016methodology}. 
\cite{wu2016methodology}. 

Policy evaluation (PE), which aims to approximate a value function, is an essential step in many RL algorithms. For instance, in actor-critic \cite{zhang2020provably}, the resulting value function could be used to estimate the expected return of the states for a given policy, which can be further used in a policy improvement step. 
The first algorithm for PE achieving differential privacy (DP) is proposed by
\cite{balle2016differentially}, which originates from Monte-Carlo methods.  However, Monte-Carlo methods need a full trajectory before updating the estimation,  which might be impractical when the task 
% would never end and the full trajectory can not be obtained. 
incurs a long trajectory for an episode.

% As another more widely used policy evaluation method, temporal difference (TD) learning updates the approximation to the value function at the end of each step \cite{sutton1988learning}. 
% An alternative to the Monte Carlo method is the Temporal Difference (TD) method, which enjoys incremental updates without full trajectory information. 
% An alternative to the Monte Carlo method is the temporal difference (TD) learning \cite{sutton1988learning}, 
% \zhao{tbd}
Another classical PE method is the temporal difference (TD) learning \cite{sutton1988learning}, 
which allows incremental updates without using full trajectory information. 
% To allow TD to best approximate the value function and enable it in large or continuous state space,
% To allow TD to approximate the value function more effectively and to enable it in large or continuous state space,
To enable TD learning to approximate the value function well in large or continuous state space,
function approximation is employed. 
% Large amounts of works 
A large amount of works \cite{sutton2008convergent,sutton2009fast,bhandari2018finite,sun2020adaptive} focus on the analysis of TD learning with linear function approximation. 
To make the TD learning more effective in many RL tasks where the value function is more complex and can not be simply approximated by linear functions, 
% \citet{maei2009convergent}
% \citeauthor{maei2009convergent}\citeyear{maei2009convergent}
% \citeauthor{maei2009convergent} \citeyear{maei2009convergent}
\citet{maei2009convergent}
build up the first framework for the analysis of TD learning with nonlinear value function approximation and a great number of advances  \cite{WaiHYWT19,qiu2020single,wang2021finitesample} have been made for the effectiveness of nonlinear TD learning. 
% \zhao{i will fix it}.
Though the effectiveness has been extensively studied, the importance of privacy in TD learning has long been ignored.

In this paper, we propose the first differentially private temporal difference (DPTD) learning algorithm
to preserve privacy in TD learning  with nonlinear  value  function approximation in the formulation of 
stochastic nonconvex-strongly-concave optimization.
% To analyze the sensitivity and achieve DP in nonlinear TD learning, we
% consider perturbing the gradients on both the primal and dual sides, which requires injecting
% the noises on the primal and dual sides simultaneously since protecting the privacy of the gradients on both the primal and dual sides in nonlinear TD learning is needed. 
To analyze the sensitivity and achieve DP in nonlinear TD learning, we consider perturbing the gradients on both the primal and dual sides by injecting noise to the primal and dual sides simultaneously.
However, different from canonical tasks of preserving privacy in stochastic gradient descent, devising such noises in the formulation of nonlinear TD learning is more challenging since the noises on the primal side will also suppress
the convergence of the dual side and vice versa.
%and
%hence the noises injected on both sides are coupled with each
%other. 
We overcome this challenge by employing the momentum-based stochastic gradient descent ascent to achieve a  single-timescale  algorithm, which enables us to update the parameters of both the primal and dual sides with the learning rates of the same order.
In this way, it is possible to preserve the privacy of both the primal and dual sides using noises with the same variances to avoid the large  privacy  cost.
Finally, we perturb the gradients on primal and dual sides using Gaussian noises with the same and carefully chosen 
% variances to make a good trade-off to achieve the privacy and utility guarantees.
variances to efficiently preserve the privacy and make a good trade-off between the privacy and utility guarantees.
% As a result, DPTD is able to protect the single state transition, \textit{i.e.}, a triple $ (s,a,s') $ in a given trajectory. This means our algorithm can make two neighboring  trajectories that only differ in one state transition indistinguishable.

% In summary, we make the following three-fold contributions: 
In summary, we make the following contributions.
% \zhao{to be fixed}
\begin{itemize}
    \item We propose the first TD learning method that achieves DP with nonlinear function approximation, named DPTD. 
    % We prove that our algorithm satisfies $ (\epsilon,\delta) $-DP and protects the single state transition. 
    We prove that our algorithm could protect the single state transition with $(\epsilon,\delta) $-DP guarantee. 
    % \item We propose the first TD learning method that achieves DP, named DPTD (Algorithm \ref{alg: DPTD}). We prove that our algorithm satisfies $ (\epsilon,\delta) $-DP and protects the single state transition. 
    % \zhao{we give two results about dp?}\ze{Should this be mentioned?}
    
    \item 
    We prove that the utility of our algorithm is upper bounded by $\widetilde{\mathcal{O}} \left(\frac{ (d\log (1/\delta) ) ^\frac{1}{8}}{ (n\epsilon) ^\frac{1}{4}}\right) $, where the tilde hides the log factor. 
    % $\sqrt{\log (n \epsilon) }$. 
    % We prove our algorithm's utility, upper bounded by $\widetilde{\mathcal{O}} \left(\frac{ (d\log (1/\delta) ) ^\frac{1}{8}}{ (n\epsilon) ^\frac{1}{4}}\right) $, where the tilde hides the factor $\sqrt{\log (n \epsilon) }$. 
    % Second, we prove our algorithm's utility, upper bounded by $\widetilde{\mathcal{O}} (\frac{ (d\log (1/\delta) ) ^\frac{1}{8}}{ (n\epsilon) ^\frac{1}{4}}) $, where the tilde hides the factor $\sqrt{\log (T) }$. 
    % \zhao{seems a bit strange (there is no $T$ but we `hide' $\log T$ (i know the reason but the reviewer might be confused))}
    % \ze{just delete this: The utility of DPTD outperforms the previous DP results $\mathcal{O}\left (\frac{ (d \log  (1 / \delta) ) ^{1 / 4}}{ (n \epsilon) ^{1 / 2}}\right) $} \cite{wang2019efficient}. \dong{I think ``at least the same convergence rate'' is a bit confusing..is it the same? up to log factor? maybe we should give the converegence rate for TD here?}\ze{solved} \zhao{is this the SOTA results?}
    
    \item 
    % Third, we report  experiment results conducted on the simulation environments, \textit{i.e.}, OpenAI Gym, to test the utility of our algorithm. We compare our method with several advanced algorithms and show the efficiency of our algorithm. \zhao{i will rewrite here}
    % We conduct extensive experiments in OpenAI Gym environments. The experimental results validate our theoretical analysis of utility and show significant improvements over previous approaches. 
    We conduct extensive experiments in OpenAI Gym environments. The experimental results show clear improvements  against previous approaches.

\end{itemize}

\paragraph{Notations} 
Throughout this paper, we use $\|\cdot\|$ to denote the $\ell_2$ norm of the vectors and $ (x,y) $ to denote the concatenation of two vectors $x$ and $y$. For a given set $\mathcal{X}$, let $\mathcal{P}_{\mathcal{X}}(\cdot)$ be the projection to the set $\mathcal{X}$. 
We denote $[n]=\{1,\cdots,n\}$ for $n\in \mathbb{N}^+$. 
% Let $\tau$ represent a trajectory and $\xi_i=(s_i,a_i,s^\prime_i)$ be the $i$-th state-action-state pair in a given trajectory. 
Let $\tau$ represent a trajectory and $\xi_i=(s_i,a_i,s^\prime_i)$ be the $i$-th state transition in a given trajectory.

\section{Related Work}
In this section, we present the works for studying TD learning and 
the recent advances in achieving DP in RL.
\paragraph{Temporal Difference Learning}
Policy evaluation (PE), which approximates the value function of a given policy, is a fundamental part of RL. 
One of the most widely used policy evaluation methods is temporal difference (TD) learning which is first proposed in  \cite{sutton1988learning} and 
% aims to minimize the Bellman error by approximating the value function of a given policy. 
aims to solve PE by minimizing the Bellman error. 
While most of the existing works focus on analyzing the  convergence rate of TD learning with linear value function approximation \cite{sutton2008convergent,sutton2009fast,bhandari2018finite,sun2020adaptive}, nonlinear function approximation might be more preferable which can
tackle the complex learning objectives in some complex tasks better. 
The most notable example might be using neural networks with nonlinearities to approximate the value functions. 
% Hence TD learning with nonlinear function approximation has attracted significant attention in recent years.
% This motivates researchers to study on a more general case, smooth non-linear \zhao{tbd} value functions.  
\citet{maei2009convergent} present the first framework for TD learning with smooth nonlinear value functions.
% by minimizing a Bellman error objective function. 
\citet{WaiHYWT19} reformulate the 
% minimization optimization problem in 
nonlinear TD learning as
a primal-dual finite-sum optimization problem via Fenchel's duality, where the primal side is nonconvex and the dual side is strongly-concave, and propose a TD learning method with variance reduction technique in the offline setting. 
% A TD learning method with variance reduction technique is proposed in \cite{WaiHYWT19} in the offline setting.
% Via Fenchel's duality, the minimization problem is 
% reformulated into a primal-dual finite-sum optimization by  \cite{WaiHYWT19}, where the primal side is nonconvex and the dual side is strongly-concave. Also,  \cite{WaiHYWT19} propose a  temporal difference method achieving variance reduction, while the offline setting is required. 
% Further, to bridge the gap between the offline setting required in variance reduction and the online algorithm,  \cite{qiu2020single} propose  primal-dual online TD algorithms, and the convergence rate for the algorithm without variance reduction is $\widetilde O (\frac{1}{T^\frac{1}{4}}) $.  \cite{wang2021finitesample} improve the convergence rate of temporal difference learning with gradient correction  to $O (\frac{1}{T^\frac{1}{2}}) $, up to a factor $O (\log T) $.
Further, \citet{qiu2020single} propose primal-dual online TD algorithms based on the variance reduction technique in the online setting.

\paragraph{Differential Privacy and Applications}
% \zhao{i will rewrite here}
% Differential privacy (DP)  is a system \zhao{maybe a `concept'?}\ze{both ?} for publicly sharing information about a dataset by describing the patterns of groups within the dataset while withholding information about individuals in the dataset.  
Differential privacy (DP) is first formally introduced by \cite{dwork2006differential} which aims to provide rigorous privacy-preserving guarantee of the systems. In recent years, privacy-preserving machine learning algorithms have been extensively studied in empirical risk minimization (ERM) \cite{wang2019efficient}, deep learning (DL) \cite{AbadiCGMMT016} and RL \cite{balle2016differentially}.

We briefly discuss the DP RL algorithms.
The first DP RL algorithm for PE is presented in  \cite{balle2016differentially}, 
motivated by the protection of user records in medical research. However, their methods originate from Monte-Carlo methods, which require at least one full trajectory for updating the value function approximation once. 
% \ze{and they apply an output perturbation mechanism to achieve DP (which is a trivial way)} \zhao{i think this sentence might offend the audience..} 
Private Q-learning algorithm is given by  \cite{wang2019privacy}, achieving DP by protecting the reward function.  \citet{lebensold2019actor} focus on  how actor-critic methods perform when initialized with a privatized first-visit Monte-Carlo estimate in  \cite{balle2016differentially}. 
% Off-policy evaluation is studied in  \cite{xie2019privacy}, where an algorithm called gradient perturbed off-policy evaluation (GPOPE)  is proposed. 
% \citeauthor{xie2019privacy} \citeyear{xie2019privacy} propose an algorithm of gradient perturbed off-policy evaluation (GPOPE)  to study the privacy-preserving off-policy evaluation. 
% \cite{vietri2020private} develop a private optimism-based learning algorithm  under a relaxation of DP called joint differential privacy (JDP) and provide the upper bound and the lower bound of the utility. 
\citet{vietri2020private} establish both the PAC and regret utility
guarantees of an optimism-based private RL algorithm for episodic tabular MDPs.
%  \zhao{what a kind of regret bound}\ze{is this ok?}

\section{Preliminaries}

%In this section, we introduce necessary definitions and assumptions in the field of DP and PE. 

Before we formally present our algorithm, we first introduce  PE, some definitions in DP and necessary assumptions.

\subsection{Policy Evaluation}
In RL, a discounted Markov decision process  (MDP)  is denoted by a tuple $\mdp =  ( \state, \action, P, \reward, \ga) $, where $\state$ is the state space, $\action$ is the action space, $P(\cdot | s, a)  $ is the transition probability kernel, $\reward:\mathcal{S}\times\mathcal{A}\to\mathbb{R}$ is the reward function, and $\ga$ is the discount factor. A policy $\pi$ takes state $s \in \state$ as an input and gives a distribution over actions $\action$. 

% J\left(\pi_{\theta}\right)=\int_{\mathcal{S}} \rho^{\pi}(s) \int_{\mathcal{A}} \pi_{\theta}(s, a) r(s, a) \mathrm{d} a \mathrm{~d} s
 
 We consider the PE problem, where the value function is learned for a policy. For a given  policy $\pi$, the corresponding reward function is  defined as $R^\pi(s) = \Expect_{a \sim \pi (\cdot | s) } \left[R (s, a) \right]$ and the induced transition matrix  is  
%  $P^\pi  (s, s’)  = \sum_{a \in \action}\pi (a | s) P (s, a, s') $. 
  $P^\pi  (s, s’)  = \int_{\action}\pi (a | s) P (s, a, s') \mathrm{d}a$. 
 The value function is defined as $V^\pi: \state \rightarrow \reals$ representing the long term expected discounted reward under the policy $\pi$, which  is formally defined as
\begin{align*}
    V^\pi  (s)  
    = \Expect \left[ \sum^\infty_{t=0} \gamma^t R^\pi (s_t)  \mid s_0 = s, \pi \right] \,.
\end{align*}

To simplify the notations, we use $R^\pi, V^\pi$ through stacking up $R^\pi(s), V^\pi(s)$ for all $s$. By definition, $V^\pi$ satisfies the Bellman equation
\begin{align*}
    V^\pi = R^\pi + \gamma P^\pi V^\pi \,.
\end{align*}

Since the true value function is intractable, it is common to proceed the policy evaluation by minimizing the mean squared Bellman error (MSBE). 
% \zhao{We assume the MDP is aperiodic and irreducible for all $\pi$.}
% We assume there exists a stationary distribution $\mu^\pi$ of the Markov chain induced by policy $\pi$ and the matrix $D$ is constructed as $D = \operatorname{Diag}  (\{\mu^\pi (s) \}_{s \in \state})$.
We assume there exists a stationary distribution $\mu^\pi$ of the Markov chain induced by policy $\pi$. Let $D = \operatorname{Diag}  (\{\mu^\pi (s) \}_{s \in \state})$.
Then the MSBE could be formulated as 
\begin{align*}
    \text{MSBE} = \frac{1}{2} \| V^\pi - R^\pi - \gamma P^\pi V^\pi \|_D^2 \,. 
\end{align*}
% In practice, however, we cannot directly optimize the above objective as the approximated value functions live in a subspace \cite{sutton2018reinforcement}. Thus a projection step is needed. 
% When $\mathcal{S}$ is  large or infinite, it is inefficient or even unrealistic to access $V^\pi$ through a tabular form and thus the function approximation is
% needed. 
When $\mathcal{S}$ is  large or infinite, it is inefficient or even unrealistic to access $V^\pi$ through a tabular form and thus the function approximation is
needed. 
In practice, however, we can not directly optimize the above objective as the approximated value functions usually lie in subspaces \cite{sutton2018reinforcement}. Thus a projection step is needed. 
We assume $V^\pi$ is parameterized by some parameter $\theta\in\mathbb{R}^d$ where $d$ is the dimension \cite{sutton2018reinforcement}.
In the case where linear function approximation is used, \textit{i.e.}, $V^\pi = \Phi\theta $ with $\Phi\in \mathbb{R}^{ |\mathcal{S}| \times d}$ as the feature matrix, the projection $\Pi = \Phi (\Phi^\top D \Phi) ^{-1}\Phi^\top D $ is well defined and well studied.
% It is not until recently that a general projection is proposed for nonlinear function approximations. 
For twice-differentiable nonlinear function approximation,  \cite{WaiHYWT19} propose a general projected Bellman error (MSPBE) as follows
% \begin{align} \label{def:PMSBE}
%     \text{MSPBE} =& \Expect\left[\delta (s) \Psi (s)  \right]^\top \Expect\left[\Psi (s)  \Psi (s) ^\top\right] \Expect\left[\delta (s) \Psi (s)  \right] \,,
% \end{align}
\begin{align} \label{def:PMSBE}
    \text{MSPBE} =& \frac{1}{2}\Expect\left[\delta (s) \Psi (s)^\top  \right] G_{\theta}^{-1} \Expect\left[\delta (s) \Psi (s)  \right] \,,
\end{align}
where $V^\pi_\theta$ denotes the value function under policy $\pi$ parametrized by $\theta$, $\Psi (s)= \nabla_\theta V^\pi_\theta (s)$ is the gradient evaluated at state $s$, $G_{\theta}=\mathbb{E}_s\left[\Psi (s)  \Psi (s) ^\top\right]\in\mathbb{R}^{d\times d}$, $\delta (s)  = R^\pi (s)  + \gamma P^\pi V^\pi_\theta (s')  - V^\pi_\theta (s) $ is the TD error and the expectation is taken over $s \in \state, a \sim \pi (\cdot | s) , s' \sim P (s, a) $.
% Via the Fenchel's duality that $1 / 2 \cdot\|x\|_{A^{-1}}^{2}=\max _{y \in \mathbb{R}^{d}}\langle x,y\rangle-1 / 2 \cdot y^{\top} A y$, we have a primal-dual formulation of MSPBE minimization problem as 
Via the Fenchel's duality that $\frac{1}{2}\|x\|_{A^{-1}}^{2}=\max _{y \in \mathbb{R}^{d}}\langle x,y\rangle-\frac{1}{2} y^{\top} A y$, 
% we have a primal-dual formulation of MSPBE minimization problem as 
the MSPBE minimization problem has a primal-dual formulation as 
\begin{equation}
\label{problem: min-max RL}
\begin{aligned}
&\min _{\theta \in \Theta} \text{MSPBE} (\theta) \\
&=\min _{\theta \in \Theta} \max _{\omega \in \Omega}\left\{\mathcal{L} (\theta, \omega) :=\mathbb{E}_{s, a, s^{\prime}}\left[\ell\left (\theta, \omega ; s, a, s^{\prime}\right) \right]\right\}\,,
\end{aligned}
\end{equation}
% where we define
% \begin{align*}
% &\ell\left (\theta, \omega ; s, a, s^{\prime}\right) \\
% &\quad:=\left\langle\delta \cdot \nabla_{\theta} V_{\theta} (s) , \omega\right\rangle-\frac{1}{2} \omega^{\top}\left[\nabla_{\theta} V_{\theta} (s)  \nabla_{\theta} V_{\theta} (s) ^{\top}\right] \omega
% \end{align*}
% with $
% \delta:=R\left (s, a\right) +\gamma V_{\theta}\left (s^{\prime}\right) -V_{\theta} (s) \,.
% $
where
\begin{align*}
\ell\left (\theta, \omega ; s, a, s^{\prime}\right):=\left\langle\delta(s) \Psi(s) , \omega\right\rangle-\frac{1}{2} \omega^{\top}\left[\Psi (s)  \Psi (s) ^\top\right] \omega\,.
\end{align*}
% with $
% \delta:=R\left (s, a\right) +\gamma V_{\theta}\left (s^{\prime}\right) -V_{\theta} (s) \,.
% $

% \zhao{$\delta_{\theta}(s,a,s^\prime)$?}

% More generally,  we denote $f (\theta,\omega;\xi) =\ell (\theta,\omega;s,a,s') $.
% % , where $\xi$ represents a state-action-state triple in RL, \textit{i.e.}, $ (s,a,s')$. 
% Then we define $F (\theta,\omega) :=\mathbb{E}_{\xi\sim \Xi}[f (\theta,\omega;\xi]$ and the original min-max problem (\textit{i.e.}, Eq.(\ref{problem: min-max RL})) is transformed into the following form
More generally, let $f (\theta,\omega;\xi) =\ell (\theta,\omega;s,a,s') $ and $F (\theta,\omega) :=\mathbb{E}_{\xi\sim \Xi}[f (\theta,\omega;\xi)]$. Then the original minimax problem in Eq. (\ref{problem: min-max RL}) is transformed into the following form
% \begin{equation}
% \label{problem: NCSC min-max}
% \min _{\theta \in \Theta} \max _{\omega \in \Omega} F (\theta, \omega) =\min_{\theta \in \Theta} \text{MSPBE} (\theta)\,.
% \end{equation}
\begin{align}\label{problem: NCSC min-max}
\min_{\theta \in \Theta} \text{MSPBE} (\theta)=\min _{\theta \in \Theta} \max _{\omega \in \Omega} F (\theta, \omega)\,.
\end{align}
% \zhao{justify the assumptions}

The difficulty of solving the above objective largely arises from the fact that it may be nonconvex in $\Theta$ but concave in $\Omega$. 
% Moreover, as RL proceeds in a Markov chain, the gradient sampling is not i.i.d. \dong{whats the relationship between i.i.d and the following assumptions?} 
% Thus, we maintain the following assumptions about the objective function, gradients, and the Markov chain.
% Thus, we maintain the following assumptions which are common in the field of nonconvex-strongly-concave primal-dual optimization \cite{luo2020stochastic,qiu2020single,WaiHYWT19} and DP empirical risk minimization
% (ERM) problem \cite{wang2019efficient,Wang2017differentially,wang2019differentially}.
Like previous works, we need the following assumptions which are common in the field of nonconvex-strongly-concave primal-dual optimization \cite{luo2020stochastic,qiu2020single,WaiHYWT19} and DP ERM problem \cite{wang2019efficient,Wang2017differentially,wang2019differentially}.

% The first assumption is about the existence of solutions. If there is no solution, then it is impossible for the optimization algorithm to work
The first assumption guarantees the existence of a solution, which hence ensures the feasibility of the problem
\cite{luo2020stochastic,qiu2020single}.
% \begin{asmp}[Existence of solutions]
% \label{asmp: existence of solution}
% The solution $\theta^*=\argmin_{\theta\in \Theta}\emph{MSPBE} (\theta) $ exists.  For notational convenience, we denote $J (\theta) =\max_{\omega\in \Omega} F (\theta,\omega) $. We assume $J (\theta^*) > -\infty$, such that  $J (\theta) \geq J (\theta^*) >-\infty$. 
% \end{asmp}
\begin{asmp}[Existence of solutions]
\label{asmp: existence of solution}
The solution $\theta^*=\argmin_{\theta\in \Theta}\emph{MSPBE} (\theta) $ exists.  Let $J (\theta) :=\max_{\omega\in \Omega} F (\theta,\omega) $. We assume $J (\theta^*) > -\infty$. 
\end{asmp}
% This is a standard assumption in  nonconvex-strongly-concave primal-dual optimization literature \cite{luo2020stochastic,qiu2020single}.
% \shuai{don't need to mention 'standard' any more, since we already state it above}\zhao{fixed}

% % \zhao{
% \shuai{give motivations for assumptions first, then it can hold}

%\zhao{
% The next assumption is a mild one that constrains the gradient of $F$ and ensures smoothness of the function. 
The next assumption is about continuity of the gradient, which holds when the parametric family of functions has bounded, smooth gradient and Hessian \cite{qiu2020single,WaiHYWT19}. Furthermore, this assumption implies that $F (\theta,\cdot) $ and $F (\cdot, \omega) $ are both $L_F$-Lipschitz smooth.
%}
% The next\shuai{Next} assumption \ze{constrain the gradient of $F$ and smooth the function.} is a mild \shuai{one} one and it holds when the parametric family of functions has bounded, smooth gradient and Hessian \cite{qiu2020single,WaiHYWT19}. Furthermore, this assumption implies that $F (\theta,\cdot) $ and $F (\cdot, \omega) $ are both $L_F$-Lipschitz smooth.
% % }

% \begin{asmp}[Lipschitz continuity of $\nabla F$]
% \label{asmp:wlipschitz}
% There exists some constant $L_F>0$ such that for all $ (\omega, \theta) ,  (\omega', \theta') $ 
% \begin{align*}
%     \left\| \nabla F (\omega, \theta)  - \nabla F (\omega', \theta')  \right\| 
%     \leq L_F \left\| (\omega, \theta)  -  (\omega', \theta') \right\|  \,.
% \end{align*}
% \end{asmp}
\begin{asmp}[Lipschitz continuity of $\nabla F$]
\label{asmp:wlipschitz}
There exists some constant $L_F>0$ such that for any $\theta,\theta'\in\Theta$,  $\omega,\omega'\in\Omega$, the gradient $\nabla F (\theta,\omega)=(\nabla_{\theta} F (\theta,\omega),\nabla_{\omega} F (\theta,\omega))$ satisfies
\begin{align*}
    \left\| \nabla F (\theta,\omega)  - \nabla F (\theta',\omega' )  \right\| 
    \leq L_F \left\| (\theta,\omega)  -  (\theta',\omega') \right\|  \,.
\end{align*}
\end{asmp}
The third assumption upper bounds the stochastic gradient, which is critical for bounding the sensitivity in the analysis of DP \cite{wang2019efficient,Wang2017differentially}. 
% This assumption also 
% motivates a key lemma
% that helps bound the variance of stochastic gradients.
% \zhao{tbd}
 \begin{asmp}[Stochastic G-Lipschitz]
 \label{asmp: G-lispchitz}
%  For any stochastic function $f$, for all $ (\omega,\theta) , (\omega',\theta') $
% \begin{align*}
% \left\|  f (\omega, \theta)  -  f (\omega', \theta')  \right\| 
%     \leq G \left\| (\omega, \theta)  -  (\omega', \theta') \right\|  \,.
% \end{align*}
For any $\xi$, $ (\omega,\theta)$ and $(\omega',\theta') $, the stochastic function $f$ satisfies
\begin{align*}
\left\|  f (\omega, \theta;\xi)  -  f (\omega', \theta';\xi)  \right\| 
    \leq G \left\| (\omega, \theta)  -  (\omega', \theta') \right\|  \,.
\end{align*}
 
 \end{asmp} 
%  Assumption \ref{asmp: G-lispchitz} upper bounds the stochastic gradient by the constant $G$, thus critical to DP's analysis. 

% Though stochastic G-lipschitz is required in most DP papers, it is a strong assumption which is not easily satisfied in Markovian Sampling. We  consider a weaker assumption to bound the variance of stochastic gradients, shown in Assumption \ref{asmp: bounded variance}. It will be shown that if we aim to get DP under trajectory rather than under state-action-state, Assumption \ref{asmp: G-lispchitz}, \ref{asmp:bounded} and Lemma \ref{lemma: bounded variance} are directly replaced by Assumption \ref{asmp: bounded variance}.
% \begin{asmp}[Bounded variance]
% \label{asmp: bounded variance}
%  The variance of the stochastic gradient $\nabla f (\theta, \omega; \xi) =\left (\nabla_{\theta} f (\theta, \omega;\xi) , \nabla_{\omega} f (\theta, \omega;\xi) \right) $ is bounded as $\mathbb{E}_{\xi \sim \Xi }\|\nabla f (\theta, \omega;\xi) -\nabla F (\theta, \omega) \|^{2} \leq \sigma^{2}$, where $\sigma^2$ is a given constant. 
% \end{asmp}

The fourth assumption 
% gives the feasible sets, 
restricts the feasible sets of the parameter to be convex, which is
common in TD learning \cite{qiu2020single,sun2020adaptive,bhandari2018finite}. 
% If we set $\Theta$ and $\Omega$ as $\mathbb{R}^d$, the projection step in our algorithm can be viewed as eliminated.
% \shuai{this sentence is almost the same with the assumption content. then why do we need to write again??}\ze{ok}
\begin{asmp}[Convex sets]
\label{asmp: feasible convex set}
The feasible sets $\Theta$ for the primal variable $\theta$ and $\Omega$ for the dual variable $\omega$
% \shuai{these notations are first mentioned here? then why the reader would know what does this mean}\ze{ok}
are closed convex sets. 
\end{asmp}

The next assumption guarantees the existence and uniqueness of the solution  $\omega^\ast=\max_{\omega\in \Omega}F(\theta,\omega)$, for any fixed $\theta\in \Theta$. It holds when $G_{\theta}$ defined in Eq. (\ref{def:PMSBE}) is positive definite \cite{qiu2020single,WaiHYWT19}. 
% \zhao{tbd}
%\shuai{we do not need to repeat}

% \zhao{The last assumption guarantees the existence and uniqueness of the solution  $\omega^\ast=\max_{\omega\in \Omega}F(\theta,\omega)$, for any fixed $\theta\in \Theta$. It holds when $G_{\theta}$ defined in Eq.(\ref{def:PMSBE}) is positive definite \cite{qiu2020single,WaiHYWT19}.}
\begin{asmp}[Strong concavity]
\label{asmp: strongly concave}
For any given $\theta \in \Theta$, the function $F (\theta, \cdot) $ is $\mu$-strongly concave, i.e., 
% $\forall \theta \in \Theta$ and 
$\forall\ \omega, \omega^{\prime} \in \Omega$, $F (\theta, \cdot) $ is concave and $\left\|\nabla_{\omega} F (\theta, \omega) -\nabla_{\omega} F\left (\theta, \omega^{\prime}\right) \right\| \geq\mu\left\|\omega-\omega^{\prime}\right\|$.
\end{asmp}
% The last assumption assumes data is iid. This assumption is impractical for DP under SAS since data points in a single trajectory are dependent, though it may hold more naturally with DP under trajectory, which we discuss in Section \ref{sec: under trajectory}. 
% Moreover, this is standard in DP relevant analysis \cite{Wang2017differentially,  wang2019efficient, wang2019privacy,qiu2020single}
% and it remains to explore how to replace the iid assumption. 
The last assumption assumes data is i.i.d. 
% This assumption might be impractical for 
% DP under state-action-state in Definition \ref{def: SASDP2}
% since data points in a single trajectory might be correlated, though it may hold more naturally with DP under trajectory in Definition \ref{def: TDP2}. 
Though this assumption might be impractical for 
DP under state-action-state in Definition \ref{def: SASDP2}
since data points in a single trajectory might be correlated, it may hold more naturally with DP under trajectory in Definition \ref{def: TDP2}. 
Moreover, this is standard in DP-relevant analysis \cite{Wang2017differentially,wang2019efficient,wang2019privacy}.
% and it remains to explore how to replace the iid assumption. 

\begin{asmp}[Sampling i.i.d. data]
\label{asmp: iid sampling}
For a given dataset $S$,  data points in $S$ are independent and identical distributed (i.i.d.). Further, the algorithm samples the data points uniformly.
\end{asmp}
% Armed with the above assumptions, we provide our algorithm DPTD and the theoretical analysis in the following section.

\subsection{Differential Privacy}
% Two datasets $X$ and $X'$ are neighboring if they only differ in one data point. Then the formal definition of DP is given as follows.
Two datasets $X$ and $X^\prime$ are neighboring if they only differ in one data point. Then the DP is defined as follows.
% which is introduced in  \cite{dwork2006differential}.

\begin{defn}[$ (\epsilon,\delta) $-DP \cite{dwork2006differential}]
\label{def: DP}
A randomized mechanism $\mathcal{M}:\mathcal{X}\rightarrow\mathcal{Y}$ satisfies $ (\epsilon,\delta) $-differential privacy if for any two neighbouring
% \shuai{do we have the def for 'neighbouring'?} \zhao{we give two definitions of `neighbouring' later}
inputs $X,X^\prime\subseteq \mathcal{X}$ and any subset of outputs $Y\subseteq \mathcal{Y}$, it holds that
\begin{align*}
    \mathbb{P} (\mathcal{M} (X) \in Y) \leq e^\epsilon\ \mathbb{P} (\mathcal{M} (X') \in Y) +\delta\,.
\end{align*}
% $$
% P (\mathcal{M} (x) \in y) \leq e^\epsilon P (\mathcal{M} (x') \in y) +\delta\,.
% $$
\end{defn}
To achieve $ (\epsilon,\delta) $-DP, we consider using Gaussian mechanism \cite{dwork2014algorithmic} which adds a $d$-dimensional Gaussian noise $u_t\sim N (0,\sigma^2_t\mathbf{I}_d) $ to the output at time $t$.  
% Further, to achieve $ (\epsilon,\delta) $-DP,  \shuai{'can use' does not mean 'should use'}\ze{solved} we consider using Gaussian mechanism\cite{dwork2014algorithmic}, which adds noise $u_t\sim N (0,\sigma^2_t\mathbf{I}^d) $ to a vector with $d$ dimensions at time $t$, parametrized by $\sigma_t$.  
% And one relevant concept is $\ell_2$-sensitivity,  used to select the variance $\sigma_t$ of Gaussian noise.
% And one relevant concept is $\ell_2$-sensitivity,  used to select the variance $\sigma_t$ of Gaussian noise.
% The variance of the noise used in Gaussian mechanism depends on the $\ell_2$-sensitivity given in Definition \ref{def: l2-sensitivity}.
% The magnitude of the variance of the noise used in Gaussian mechanism depends on the $\ell_2$-sensitivity of the query function, which is formally defined in Definition \ref{def: l2-sensitivity}.
The magnitude of the noise variance depends on the $\ell_2$-sensitivity of the query function, which is formally defined in Definition \ref{def: l2-sensitivity}.
% \begin{defn}[$\ell_2$-sensitivity \cite{dwork2014algorithmic}]
% For two neighbouring datasets $S$ and $S'$\shuai{this means to fix $S,S'$}, the $\ell_2$-sensitivity $\Delta (g) $ of a function $g$ is defined as $\Delta (g) =\sup _{S, S^{\prime}}\left\|g (S) -g\left (S'\right) \right\|$.
% \end{defn}

% To analyse the composition of different Gaussian mechanisms,  a natural relaxation of differential privacy based on the R\'enyi divergence is proposed in  \cite{wang2019subsampled}. The new definition named RDP shares many important
% properties with the standard definition of differential privacy and allows tighter analysis of composite heterogeneous mechanisms in addition.
% To analyze the composition of different Gaussian mechanisms, R\'enyi differential privacy (RDP) is proposed in  \cite{mironov2017renyi} which is based on the R\'enyi divergence and is a natural relaxation of DP. RDP shares many important properties\shuai{be specific} with the standard definition of DP and allows a tighter analysis of composite heterogeneous mechanisms.
% When the dataset is accessed by a sequence of randomized mechanisms iteratively,
% To analyze the composition of different Gaussian mechanisms, 
To analyze the mechanism of a sequence of randomized mechanisms more effectively, 
R\'enyi differential privacy (RDP) is proposed in  \cite{mironov2017renyi} based on the R\'enyi divergence, which is a natural relaxation of DP. 
% RDP shares many important properties with the standard definition of DP and allows a tighter analysis of composite heterogeneous mechanisms.

\begin{defn}[$ (\alpha,\rho) $-RDP \cite{mironov2017renyi} ]
% \zhao{this paper only has one author}
\label{def: RDP}
A randomized mechanism $\mathcal{M}:\mathcal{X}\rightarrow\mathcal{Y}$ satisfies $ (\alpha,\rho) $-R\'enyi differential privacy if for any two neighbouring inputs $X,X^\prime\subseteq \mathcal{X}$ and any subset of outputs $Y\subseteq \mathcal{Y}$, it holds that
% \begin{align*}
% D_\alpha (\mathcal{M} (X) \|\mathcal{M} (X') ) :=\log \mathbb{E}\left (\mathcal{M} (X) /\mathcal{M} (X') \right) ^\alpha /  (\alpha-1)  \leq \rho\,.
% \end{align*}
\begin{align*}
D_\alpha (\mathcal{M} (X) \|\mathcal{M} (X') ) :=
\frac{\log \mathbb{E}\left (\mathcal{M} (X) /\mathcal{M} (X') \right) ^\alpha}{  (\alpha-1)}  \leq \rho\,.
\end{align*}
% $$
% D_\alpha (\mathcal{M} (X) \|\mathcal{M} (X') ) :=\log \mathbb{E}\left (\mathcal{M} (X) /\mathcal{M} (X') \right) ^\alpha /  (\alpha-1)  \leq \rho\,.
% $$
\end{defn}

% \begin{defn}[Trajectory DP (TDP) ]
% \label{def: TDP}
% We consider a dataset $S$ consisting of $m$ trajectories and each trajectory is a ordered set containing at least $\frac{n}{2}$ $ (s,a,s') $ triples and at most $n$ $ (s,a,s') $ triples. We define $S$ and $S'$ are neighbouring iff $S$ and $S'$ differ in only one trajectory, i.e.
% $\exists \text{ only one } traj\in S\ \text{and only one}\ traj'\in S'$ satisfy
%  $\exists  (s,a,s') \in traj$ but $\notin traj'$, or $\exists  (s,a,s') \in traj'$ but $\notin traj$.
% If a randomized mechanism $\mathcal{M}$ satisfies $ (\epsilon,\delta) $-DP in this setting, we call $\mathcal{M}$ satisfies $ (\epsilon,\delta) $-TDP.
% \end{defn}

%Since the agent updates the approximation of the value function after the agent experiences a new state transition in TD learning, in this paper, our main focus is how to protect the privacy of single state transition. To formalize this goal,  we give the following definition called \textit{DP under state-action-state}, which is a realization of DP under our setting. 
% \textit{DP under state-action-state} claims that  two neighbouring trajectories differ in only one state transition.
%In DP under state-action-state, two trajectories are considered to be neighboring if these two trajectories differ in only one state transition, \textit{i.e.}, a state-action-state triple .
% It is common for an RL algorithm to involve some sensitive information through the learning process, which is encoded by the experiences, \textit{i.e.}, the state-action-state triples. 
When the RL algorithm is deployed online in applications such as recommender systems, sensitive user information is often encoded through experiences, \textit{i.e.}, the state-action-state triples.
Our goal to protect the sensitive information in RL is realized by making the state-action-state triple approximately indistinguishable for attackers, which leads to our specification of neighboring datasets. This definition is applicable to our approach and other pure online RL algorithms.

For notational convenience, we use $ \xi_i=(s_i,a_i,s^\prime_i) $ and $\hat{\xi}_i=(\hat{s}_i,\hat{a}_i,\hat{s}^\prime_i) $ to denote the state-action-state triples. 

\begin{defn}[DP under state-action-state]
\label{def: SASDP2} 
% \zhao{improved}
Let $S=\{ \xi_i \}^n_{i=1}$ and $\hat S=\{ {\hat \xi}_i \}^n_{i=1}$ be two trajectories of the same length.
$S$ and $\hat{S}$ are neighbouring if there exists a unique $i\in [n]$ such that $ \xi_i \neq {\hat\xi}_i $. 
%Let $S$ be a dataset consisting of only one trajectory $\tau=\{ (s_i,a_i,s^\prime_i) \}^n_{i=1}$.
% which can be either an unordered set or an ordered set. 
%Similarly, let $\hat{S}$ be another dataset consisting of only one trajectory $\hat{\tau}=\{ (\hat{s}_i,\hat{a}_i,\hat{s}^\prime_i) \}^n_{i=1}$.
If a randomized mechanism $\mathcal{M}$ is $ (\epsilon,\delta) $-DP under this definition of neighbourhood, this mechanism is $ (\epsilon,\delta) $-DP under state-action-state.
\end{defn}

% Though our main focus is on \textit{DP under state-action-state},  it is also discussed in our theoretical analysis that  two datasets which consist of $m$ trajectories and only differ in one trajectory. Here we first give a formal definition under this setting, named \textit{DP under trajectory}.
% Some previous works \dong{I guess we can start the sentence with just balle2016?}\cite{balle2016differentially} consider to project the trajectory-wise DP \dong{I am not quite sure what does it mean by project the trajectory-wise DP?}. Specifically, two datasets are considered to be neighboring if these two datasets are consist of $m$ trajectories and only differ in one trajectory in \cite{balle2016differentially} \dong{I don't think we need to cite again here?}, which is referred to as \textit{DP under trajectory} in this paper. The formal definition of DP under trajectory is as follows.
% \dong{I guess we can start the sentence with just balle2016?}\zhao{fixed}

% When the RL algorithm involves offline learning components, the setting where one trajectory composes a dataset is no longer feasible. 
When the RL algorithm is deployed offline, the above definition may be insufficient to provide privacy guarantee since the setting where one trajectory composes a dataset is no longer feasible.  In the case where the dataset is composed of multiple trajectories, we introduce
% In this case, we introduce 
a more general definition of DP under trajectory that allows at most one trajectory to differ in neighbouring datasets. 
\begin{defn}[DP under trajectory]
\label{def: TDP2}
Let $S=\{\tau_i\}^m_{i=1}$ and $\hat{S}=\{\hat{\tau}_i\}^m_{i=1}$ be two datasets consisting of $m$ trajectories
where $\tau_i=\{\xi_j\}^{|\tau_i|}_{j=1}$ with $|\tau_i|\leq n$ and $\hat{\tau}_i=\{\hat{\xi}_j\}^{|\hat{\tau}_i|}_{j=1}$ with $|\hat{\tau}_i|\leq n$.
% with length $\frac{n}{2}\leq|\tau_i|\leq n$.
$S$ and $\hat{S}$ are neighbouring if there exists a unique $i\in [m]$ such that $\tau_i\neq\hat{\tau}_i$. 
If a randomized mechanism $\mathcal{M}$ is $ (\epsilon,\delta) $-DP  under this definition of neighbourhood, 
this mechanism is $ (\epsilon,\delta) $-DP  under trajectory.
\end{defn}
% \zhao{length $|\tau_i|\leq n$}
% In DP under trajectory, two datasets are neighboring if these two datasets are consisting of $m$ trajectories and only differ in one trajectory, which
% is also considered in \cite{balle2016differentially}.
% Specifically, they consider two datasets to be neighboring if these two datasets are consisting of $m$ trajectories and only differ in one trajectory,
% which is referred to as \textit{DP under trajectory} in this paper. 
% The formal definition of DP under trajectory is as follows.

% \begin{defn}[State-action-state DP  (SASDP) ]
% \label{def: SASDP}
% We consider a dataset $S$ consisting of only one trajectory , which can be either an unordered set or an ordered set, with $n$ $ (s,a,s') $ triples.  We define $S$ and $S'$ are neighbouring iff $S$ and $S'$ differ in only one $ (s,a,s') $ triple, i.e.,
%  $\exists!  (s_i,a_i,s'_i) \in S\ \text{and}\  (s_{i'},a_{i'},s'_{i'}) \in S'$, $ (s_i,a_i,s'_i) \neq  (s_{i'},a_{i'},s'_{i'}) $. 
%  They are indexed by $i$ and $i'$ for the convenience.
%  If a randomized mechanism $\mathcal{M}$ satisfies $ (\epsilon,\delta) $-DP in this setting, we call $\mathcal{M}$ satisfies $ (\epsilon,\delta) $-SASDP.
% \end{defn}

\section{Algorithm}

 We now present our algorithm, differentially private temporal difference learning (DPTD), detailed in Algorithm \ref{alg: DPTD}.

DPTD takes the adaptive step size $\nu_t$, and the constant parameters $\kappa$, $\eta$, $\alpha$, $\beta$ as the input.
These constant parameters are used to adjust the step sizes when updating the 
primal and dual variables with the
momentum-based gradient estimators. 
At each iteration, DPTD performs stochastic gradient descent and ascent of $\theta_t$ and $\omega_t$ respectively and then projects the updates to the feasible sets $\Theta$ and $\Omega$ (line \ref{algo:sgda}).
% After the projection step,
% DPTD obtains $\theta_{t+1}$ by taking a step from $\theta_{t}$ to $\widetilde{\theta}_{t+1}$ with step size $\nu_t$ 
% and obtains $\omega_{t+1}$  in the similar way (line 4). 
Then
DPTD obtains $\theta_{t+1}$ by taking a step from $\theta_{t}$ to $\widetilde{\theta}_{t+1}$ with step size $\nu_t$ 
and obtains $\omega_{t+1}$  in the similar way (line \ref{algo:moving_average}). 
% Specifically, $p'_{t+1}$ is taken as the convex combination of the  differentially private gradient estimator $p_{t}$ released at the $t$-th iteration and the stochastic gradient associated with $\xi_{t+1}$, which can be further viewed as the  exponentially weighted average over the gradients of history data points.
% The gradient estimator on the dual side  $d'_{t+1}$ is constructed in the same manner. 
Then DPTD computes the stochastic momentum-based gradient estimator $p'_{t+1}$ and $d'_{t+1}$ (line \ref{algo:momentum}), which are perturbed via the Gaussian noises with moderate variances to achieve DP (line \ref{algo:Gaussian}).
One of the main technical challenges lie in controlling privacy noises for primal and dual sides simultaneously. The common two-timescale framework implies an imbalance of privacy noises on the two sides, hence leading to an inefficient convergence rate and an unnecessarily large privacy cost. To overcome this challenge, we employ a single-timescale framework via the momentum-based stochastic gradient descent ascent \cite{qiu2020single}, which despite being more complicated to analyze the simultaneous descent dynamics, achieves desirable utility and privacy guarantees.
% It is worth noting that devising the variance of noises in the formulation of nonlinear TD is still not trivial since 
% the noise on the primal side will also suppress
% the convergence of the dual variable and vice versa and
% hence the noises injected on both sides are coupled with each
% other. 
% However, devising the variance of noises in the formulation of nonlinear TD is still challenging since 
% the noise on the primal side will also suppress
% the convergence of the dual variable and vice versa and
% hence the noises injected on both sides are coupled with each
% other. 
% The choice of the variance of the Gaussian noises $\sigma_{t+1}$ in DPTD is crucial to achieve  is detailed in Section \ref{sec:theo_res}.
% The choice of the variance of the Gaussian noises $\sigma_{t+1}$ is the key
% for DPTD to achieve DP while keeping the fast convergence rate
% so as to simultaneously guarantee privacy and utility, which is detailed in Section \ref{sec:theo_res}.
The other key challenge is the choice of the variance of the Gaussian noises $\sigma_{t+1}$, which is detailed in Section \ref{sec:theo_res}.

\begin{algorithm}[tb]
\caption{Differentially Private Temporal Difference Learning}
\label{alg: DPTD}
\textbf{Input}: $\nu_t>0$, $\kappa>0$, $\eta>0$, $\alpha>0$, $\beta>0$, $\theta_0\in \Theta$, $\omega_0\in \Omega$.\\
\textbf{Initialize}: 
\begin{align*}
u_0^p\sim N(0,\sigma_0^2\mathbf{I_d})\,,\ p_0=\nabla_\theta f(\theta_0,\omega_0;\xi_0)+u_0^p\,,
\end{align*}
\begin{align*}
 u_0^d\sim N(0,\sigma_0^2\mathbf{I_d})\,,\ d_0= \nabla_\omega f(\theta_0,\omega_0;\xi_0)+u_0^t\,.
\end{align*}
\begin{algorithmic}[1] %[1] enables line numbers
\FOR{$t=1,2,\cdots$}
    \STATE Perform stochastic gradient descent and ascent and project the updates to the feasible sets: \label{algo:sgda}
    \begin{align*}
        \widetilde \theta_{t+1} = \mathcal{P}_\Theta(\theta_t -\kappa p_t)\,,\
        \widetilde \omega_{t+1} = \mathcal{P}_{\Omega} (\omega_t + \eta d_t)\,.
    \end{align*}
    \STATE 
    Update the primal variable $\theta_{t+1}$ and dual variable $\omega_{t+1}$: \label{algo:moving_average}
    \begin{align*}
        \theta_{t+1} = \theta_t + \nu_t (\widetilde\theta_{t+1} - \theta_t)\,,\
        \omega_{t+1} = \omega_t + \nu_t (\widetilde\omega_{t+1} - \omega_t)\,.
    \end{align*}
    \STATE Compute the momentum-based gradient estimator on primal side $p'_{t+1}$ and on dual side  $d'_{t+1}$: \label{algo:momentum}
    \begin{align*}
    p'_{t+1} &=  (1-\alpha\nu_t)p_t +\alpha\nu_t\nabla_\theta f(\theta_{t+1},\omega_{t+1};\xi_{t+1})\,,\\
    d'_{t+1} &= (1-\beta\nu_t)d_t +\beta\nu_t\nabla_\omega f(\theta_{t+1},\omega_{t+1};\xi_{t+1})\,.
    \end{align*}
    \STATE Draw the Gaussian noises with variance $\sigma_{t+1}$: \label{algo:Gaussian}
    \begin{align*}
        u_{t+1}^p\sim N(0,\sigma_{t+1}^2\mathbf{I_d})\,,\
        u_{t+1}^d\sim N(0,\sigma_{t+1}^2\mathbf{I_d})\,,
    \end{align*}
    and release the differentially private gradient estimator $p_{t+1}$, $d_{t+1}$:
    \begin{align*}
    p_{t+1} =  p'_{t+1}+u_{t+1}^p\,,
    d_{t+1} = d'_{t+1}+u_{t+1}^d\,.
    \end{align*}
\ENDFOR
\STATE \textbf{Output:} ($\bar{\theta},\bar{\omega}$) sampled uniformly at random  from $\{(\theta_t,\omega_{t})\}^{T-1}_{t=0}$.
\end{algorithmic}
\end{algorithm}

\section{Theoretical Results}\label{sec:theo_res}
% In this section, we first provide the main theoretical results of privacy and utility in terms of DP under state-action-state defined in Definition \ref{def: SASDP2}. 
% Then we give further discussions on the theoretical results of privacy and utility in terms of DP under trajectory defined in Definition \ref{def: TDP2}.
In this section, we provide the main theoretical results of privacy and utility with DP under state-action-state. 
% We also give further discussions on the theoretical results of privacy and utility in terms of DP under trajectory defined in Definition \ref{def: TDP2} in Appendix \ref{sec:under_trajectory}.
% We also give further discussions on the theoretical results of privacy and utility in terms of DP under trajectory defined in Definition \ref{def: TDP2} in Appendix .
The presentation and discussions with DP under trajectory is deferred to Appendix \ref{sec:under_trajectory}.

% \subsection{Privacy and Utility under State-Action-State}
\subsection{Privacy Analysis}
Theorem \ref{theorem: privacy guarantee} provides a privacy guarantee in terms of DP under  state-action-state for Algorithm \ref{alg: DPTD}. 
\begin{theorem}[Privacy under state-action-state]\label{theorem: privacy guarantee}
Consider the DP defined in Definition \ref{def: SASDP2}.
Under Assumption \ref{asmp: G-lispchitz}, \ref{asmp: feasible convex set}, \ref{asmp: iid sampling}, given the total number of iterations $T$, for any $\delta>0$ and the privacy budget $\epsilon$, Algorithm \ref{alg: DPTD} satisfies $(\epsilon,\delta) $-DP under state-action-state with the variance
\begin{align*}
\sigma^2_t = \frac{14G^2T\alpha'}{n^2\left(\epsilon - \frac{\log(1/\delta)}{\alpha'-1}\right)}=\frac{14G^2T\alpha'}{n^2\beta'\epsilon}\,,\ \forall t\geq 0\,,
\end{align*}
where $\sigma'^2=\frac{\sigma^2_t}{4G^2}\geq 0.7$, $\alpha'=\frac{\log(1/\delta)}{(1-\beta') \epsilon}+1\leq 2\sigma^2\log(\frac{n}{\alpha'(1+\sigma'^2) }) /3+1$ and $\beta'\in (0,1) $.
\end{theorem}

\begin{proof}[Proof Sketch of Theorem \ref{theorem: privacy guarantee}]
% We only provide the proof sketch of the gradient estimator on primal side and the proof sketch of
Consider the randomized mechanisms on primal side induced by the update rule of the gradient estimator in Algorithm \ref{alg: DPTD}
\begin{align*}
    \mathcal{M}_{t}^p =
    \begin{cases}
    \nabla_\theta f (\theta_{0},\omega_{0};\xi_{0}) +u_{0}^p &t=0\\
      (1-\alpha\nu_{t-1}) p_{t-1} +\alpha\nu_{t-1}\nabla_\theta f (\theta_{t},\omega_{t};\xi_{t}) +u_{t}^p &t> 0\,.
    \end{cases}
\end{align*}
% We show $\mathcal{M}_{t}^p$ is RDP and then transform the RDP guarantee to RDP guarantee using Lemma \ref{lemma: RDP_to_DP}. 
% def: l2-sensitivity
We show $\mathcal{M}_{t}^p$ satisfies RDP and the privacy guarantee of DP could be transformed from privacy guarantee of RDP using Lemma \ref{lemma: RDP_to_DP}. 
Notice that
$\mathcal{M}_{t}^p$ is the composition of a series of randomized mechanisms $(\mathcal{G}^p_0,\cdots,\mathcal{G}^p_t)$ where
% \mathcal{G}_t^p = \alpha\nu_{t-1}\nabla_\theta f  (\theta_{t},\omega_{t};\xi_{t}) +u_{t}^p\,.
\begin{align*}
    \mathcal{G}_t^p =
    \begin{cases}
       \nabla_\theta f  (\theta_{0},\omega_{0};\xi_{0}) +u_{0}^p &t=0\\
      \alpha\nu_{t-1}\nabla_\theta f  (\theta_{t},\omega_{t};\xi_{t}) +u_{t}^p &t>0\,.
    \end{cases}
\end{align*}
It remains to show that $\mathcal{G}_t^p$ achieves RDP so as to show $\mathcal{M}_{t}^p$ achieves RDP by Lemma \ref{lemma: composition of RDP}.
To this end, in the case when $t=0$, we first consider the Gaussian mechanism $\widetilde{\mathcal{G}}_0^p =\sum_{i=0}^{n-1} \nabla_\theta f  (\theta_{0},\omega_{0};\xi_{i}) +u_{0}^p$ which takes the whole trajectory $\tau$ as the input instead of one state transition $\xi_0$ of $\tau$.
Gaussian mechanism $\widetilde{\mathcal{G}}_0^p$ consists of the Gaussian noise $u_{0}^p$ and the query $\widetilde{q}_0^p  (S) =\sum_{i=0}^{n-1}\nabla_\theta f  (\theta_{0},\omega_{0};\xi_{i})$ whose $\ell_2$-sensitivity could be shown to satisfy $\widetilde{\Delta}_{0}^p\leq 2G$. 
Thus $\widetilde{\mathcal{G}}_0^p$ and $\mathcal{G}_0^p$ satisfy RDP by Lemma \ref{lemma: RDP subsampling transformation} if the variance of Gaussian noise $u_{0}^p$ takes the value as suggested in Theorem \ref{theorem: privacy guarantee}. 
In the similar manner, we can prove that $\mathcal{G}_t^p$ satisfies RDP for the case $t>0$.
The proof sketch of the randomized mechanisms on the dual side is similar to that of the primal side.
\end{proof}

\subsection{Utility Analysis}
% \zhao{ }
We first introduce the utility metric to measure the  
nonconvex-strongly-concave optimization of TD learning and then present the 
utility analysis of our algorithm.

\paragraph{Utility Metric} To simultaneously measure the convergence on the primal and dual sides of our algorithm, we adopt the following metric to measure the utility and similar metrics are also adopted in the previous works \cite{WaiHYWT19,ghadimi2020single,qiu2020single}, which is 
% \zhao{tbd}
%  \begin{align*}
% \mathfrak{M}_{t}:=\kappa^{-1}\left\|\widetilde{\theta}_{t+1}-\theta_{t}\right\|+\left\|\nabla_{\theta} F\left(\theta_{t}, \omega_{t}\right) -p_{t}\right\|+L_{F}\left\|\omega_{t}-\omega^{*}\left(\theta_{t}\right) \right\|\,.
%  \end{align*}
% $
% \mathfrak{M}_{t}:=\kappa^{-1}\left\|\widetilde{\theta}_{t+1}-\theta_{t}\right\|+\left\|\nabla_{\theta} F\left(\theta_{t}, \omega_{t}\right) -p_{t}\right\|+L_{F}\left\|\omega_{t}-\omega^{*}\left(\theta_{t}\right) \right\|\,.
% $
\begin{align}\label{eq:metric}
\mathfrak{M}(\theta_t,\omega_t):=&\kappa^{-1}\left\|\widetilde{\theta}_{t+1}-\theta_{t}\right\|+\left\|\nabla_{\theta} F\left(\theta_{t}, \omega_{t}\right) -p_{t}\right\|\notag\\
&+L_{F}\left\|\omega_{t}-\omega^{*}\left(\theta_{t}\right) \right\|\,.
\end{align}
% \zhao{case 1: stationary point case 2: local minimizer at boundary}
The first two terms  of RHS in Eq. (\ref{eq:metric}) are used to measure the convergence of the primal variable $\theta$. 
If the first two terms $\kappa^{-1}\left\|\widetilde{\theta}_{t+1}-\theta_{t}\right\|+\left\|\nabla_{\theta} F\left(\theta_{t}, \omega_{t}\right) -p_{t}\right\|\approx 0$, then $\nabla_{\theta} F\left(\theta_{t}, \omega_{t}\right) \approx p_{t}$ and $\widetilde{\theta}_{t+1}\approx\theta_{t}$, which further indicates that $\widetilde{\theta}_{t+1}=\mathcal{P}_\Theta(\theta_t-\kappa p_t)\approx\mathcal{P}_\Theta(\theta_t-\kappa \nabla_{\theta} F\left(\theta_{t}, \omega_{t}\right))\approx \theta_t$ due to the update rules in Algorithm \ref{alg: DPTD}. In this circumstance, $\theta_t$ will be a stationary point if $\nabla_{\theta} F\left(\theta_{t}, \omega_{t}\right)=0$ and a local minimizer on the boundary of $\Theta$ otherwise. In either situation, $\theta_t$ could be considered convergent in constrained nonconvex optimization \cite{ghadimi2020single,qiu2020single}. The convergence of $\omega_t$ to the optimal maximizer $\omega^\ast(\theta_t)$ is measured by the third therm of RHS in Eq. (\ref{eq:metric}).

% Let's consider two conditions that  the primal variable converges.
% \begin{itemize}
%     \item If $\theta_t$ converges to the global or local minimizer. By the update rule, we have
% \begin{align*}
%     \widetilde{\theta}_{t+1}&=\mathcal{P}_\Theta(\theta_t-\kappa p_t) \,,\\
% p_{t+1}&=  (1-\alpha\nu_t) p_t +\alpha\nu_t\nabla_\theta f(\theta_{t+1},\omega_{t+1};\xi_{t+1}) +u_{t+1}^p\,.
% \end{align*}
% Though $p_{t+1}$  is updated with the Gaussian noise, we currently ignore the noise. Since $\nabla_\theta f(\theta_{t+1},\omega_{t+1};\xi_{t+1}) =0$ in this condition, we have 
% \zhao{maybe $\nabla_\theta F(\theta_{t+1},\omega_{t+1}) =0$ ?}
% \begin{align*}
% p_{t+1} =  (1-\alpha\nu_t) p_t +u_{t+1}^p\approx (1-\alpha\nu_t) p_t \,.
% \end{align*}
% Then we know $p_t$ converges to $0$ by the geometric series and $\widetilde\theta_{t+1}\approx\theta_t$. Thus the two terms converge to $0$. \zhao{i will revise here}
% \item If $\theta_t$ converges to one point at the border of the convex set $\Theta$. Then $p_t$ does not converge to $0$ but it still holds that $p_{t+1}\approx p_t$. And $\widetilde {\theta}_{t+1}$ also converges by the projection. Thus the two terms sill converge to $0$.
% \end{itemize}

% Under this metric, we present the utility under state-action-state achieved by our algorithm with the specified Gaussian noises in Theorem \ref{theorem: privacy guarantee}. 
Under this metric, we present the utility under state-action-state achieved by our algorithm
in the following theorem, whose proof is deferred to \ref{app:sec:pf_theorem_utility}, with the specified Gaussian noises in Theorem \ref{theorem: privacy guarantee}.

\begin{theorem}[Utility under state-action-state]
\label{theorem: utility}
% Under Assumptions \ref{asmp: differentiable}, \ref{asmp: existence of solution}, \ref{asmp:wlipschitz},  \ref{asmp:bounded}, setting the parameters $\alpha = \beta = 3$, $0<\eta \leq \mu /(4L^2_F) $, $0<\kappa\leq \eta \mu^2/(9L^2_F) $, and $\nu_t = 1/4(t+b) ^\frac{1}{2}$
% with $b\geq \max\{(2\kappa L_F^2/\mu) ^2,3\}$, with the updating rules in Algorithm \ref{alg: DPTD} and the Gaussian variance  in Theorem \ref{theorem: privacy guarantee}, we set 
% \begin{align*}
% T = \frac{Cn\epsilon}{\sqrt{d\log(1/\delta) }}\,,
% \end{align*}
Under Assumptions \ref{asmp: existence of solution}, \ref{asmp:wlipschitz}, \ref{asmp: G-lispchitz}, \ref{asmp: feasible convex set}, \ref{asmp: strongly concave}, \ref{asmp: iid sampling}, if we set the parameters $\alpha = \beta = 3$, $0<\eta \leq \mu /(4L^2_F) $, $0<\kappa\leq \eta \mu^2/(9L^2_F) $, $\nu_t = 1/4(t+b) ^\frac{1}{2}$ with $b\geq \max\{(2\kappa L_F^2/\mu) ^2,3\}$ and choose the number of iterations $T = \frac{Cn\epsilon}{\sqrt{d\log(1/\delta) }}$ where $C$ is a constant, then with the Gaussian noises in Theorem \ref{theorem: privacy guarantee}, the output of Algorithm \ref{alg: DPTD} satisfies
% \begin{align*}
% \frac{1}{T}\sum_{t=0}^{T-1}\mathbb{E}\left\|\mathfrak{M}_t\right\|\leq \widetilde{\mathcal{O}}\left(\frac{(d\log(1/\delta) ) ^\frac{1}{8}}{(n\epsilon) ^\frac{1}{4}}\right) \,.
% \end{align*}
\begin{align*}
\mathbb{E}\left\|\mathfrak{M}(\bar{\theta},\bar{\omega})\right\|\leq \widetilde{\mathcal{O}}\left(\frac{(d\log(1/\delta) ) ^\frac{1}{8}}{(n\epsilon) ^\frac{1}{4}}\right) \,.
\end{align*}
% $$
% \frac{1}{T}\sum_{t=0}^{T-1}\mathbb{E}\left\|\mathfrak{M}_t\right\|\leq \widetilde O(\frac{(d\log(1/\delta) ) ^\frac{1}{8}}{(n\epsilon) ^\frac{1}{4}}) 
% $$ 
% \ze{with the gradient complexity equal to $2(T+1)=\mathcal{O}(\frac{n\epsilon}{\sqrt{d\log(1/\delta) }})$}
% \ze{with the gradient complexity equal to }
Moreover, the  total gradient complexity of Algorithm \ref{alg: DPTD} is $2(T+1)=\mathcal{O}\left(\frac{n\epsilon}{\sqrt{d\log(1/\delta) }}\right)$.
\end{theorem}

\paragraph{Discussion}
Compared to \cite{qiu2020single},
% However, devising the variance of noises in the formulation of nonlinear TD is still challenging since 
% the noise on the primal side will also suppress
% the convergence of the dual variable and vice versa and
% hence the noises injected on both sides are coupled with each
% other. 
the main hardness to develop DPTD is to choose
a well-calibrated Gaussian noise, where good trade-offs are
needed to simultaneously achieve meaningful privacy and
the utility guarantee. It is worth noting that devising such a
Gaussian noise in our formulation is not trivial since injecting
the noises on the primal and dual side simultaneously is
required, where the noise on the primal side will also suppress
the convergence of the dual variable and vice versa. We address this challenge by Lemma \ref{lemma: forth in main proof}.
% Besides, it is also non-trivial to find a good trade-off point in terms of $T$ and $\sigma_t$ since the injected Gaussian noises influence the convergence significantly as shown in  Eq. (\ref{ineq: main proof 7}), 
% where the numerator of the RHS is dominated by $\sigma_t$ and $\nu_t$ and the denominator of the RHS is dominated by  and $T$ respectively.
Besides, to achieve a good trade-off between the privacy and utility guarantees, it is crucial to find a good trade-off point in terms of $T$ and $\sigma_t$.
This is also nontrivial  since the injected Gaussian noises influence the convergence significantly as shown in  Eq. (\ref{ineq: main proof 7}), 
where the numerator of the RHS is dominated by $\sigma_t$ and $\nu_t$ and the denominator of the RHS is dominated by $T$ respectively.

% ---------------------------------------  fig 1 ---------------------------------------------------
\begin{figure*}[h]
\begin{minipage}[t]{0.33\textwidth}
  \includegraphics[width=0.95\linewidth]{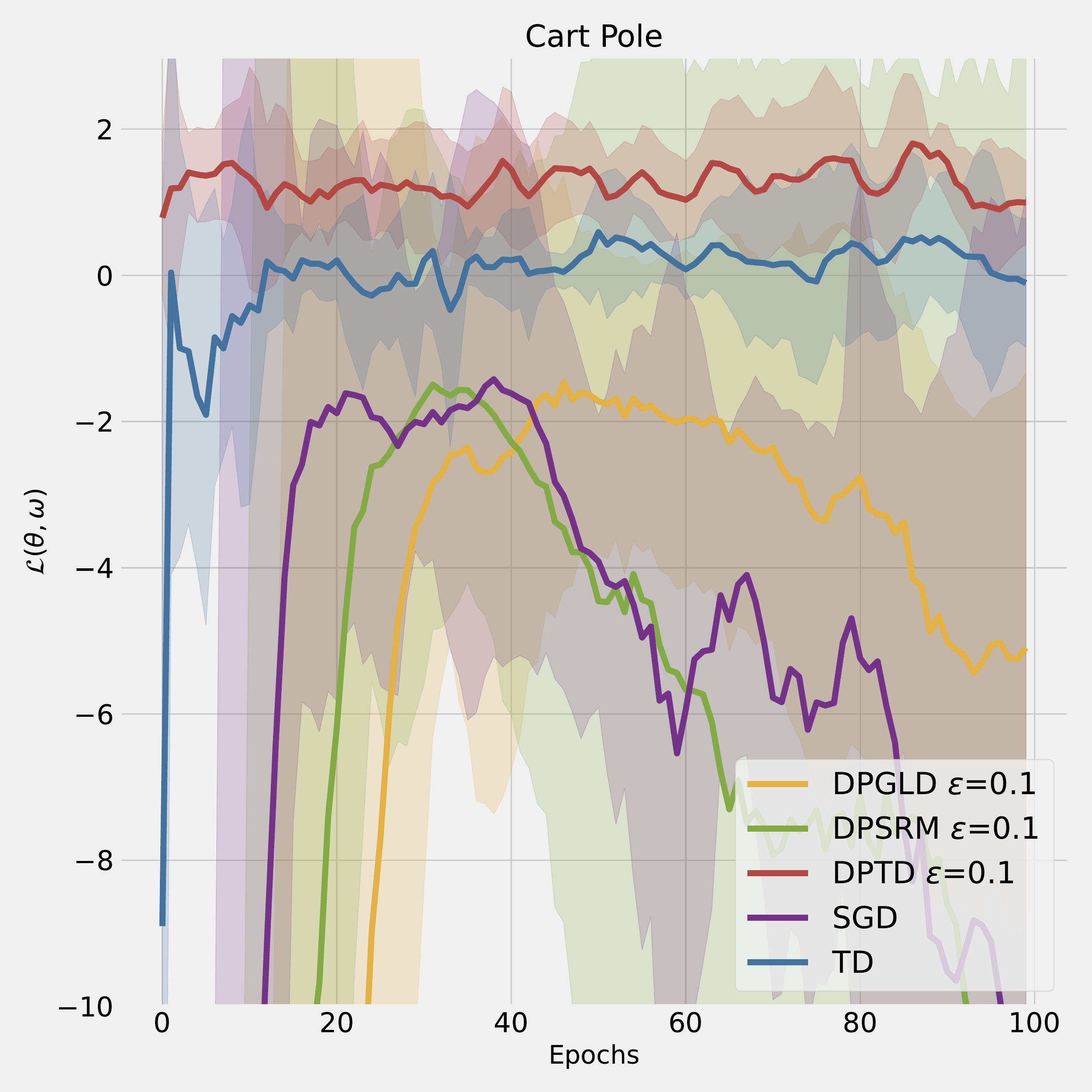}
%   \subcaption{}
   \label{fig:fig1_cp}
\end{minipage}%
\hfill
\begin{minipage}[t]{0.33\textwidth}
  \includegraphics[width=0.95\linewidth]{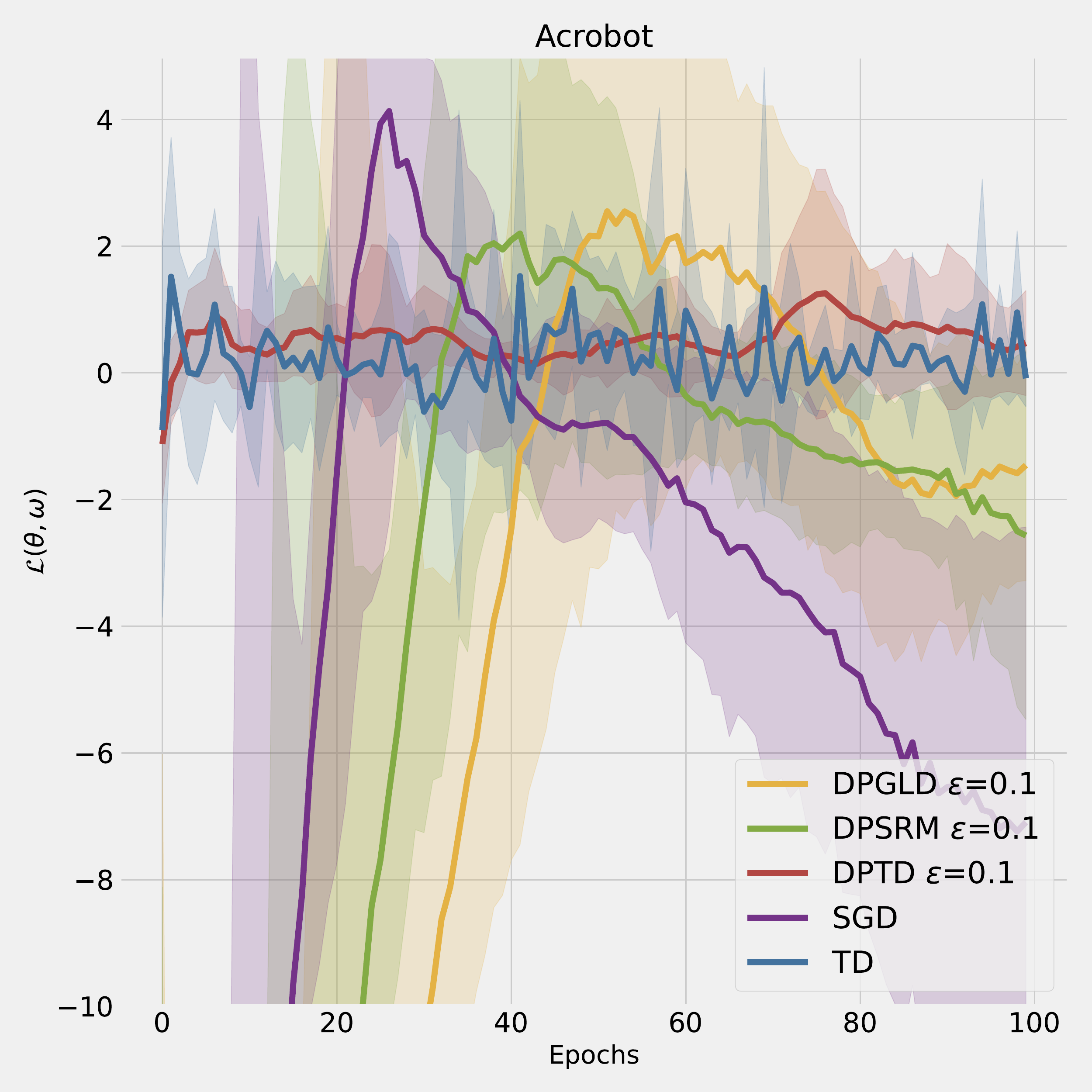}
%   \subcaption{}
   \label{fig:fig1_acro}
\end{minipage}%
\hfill
\begin{minipage}[t]{0.33\textwidth}
  \includegraphics[width=0.95\linewidth]{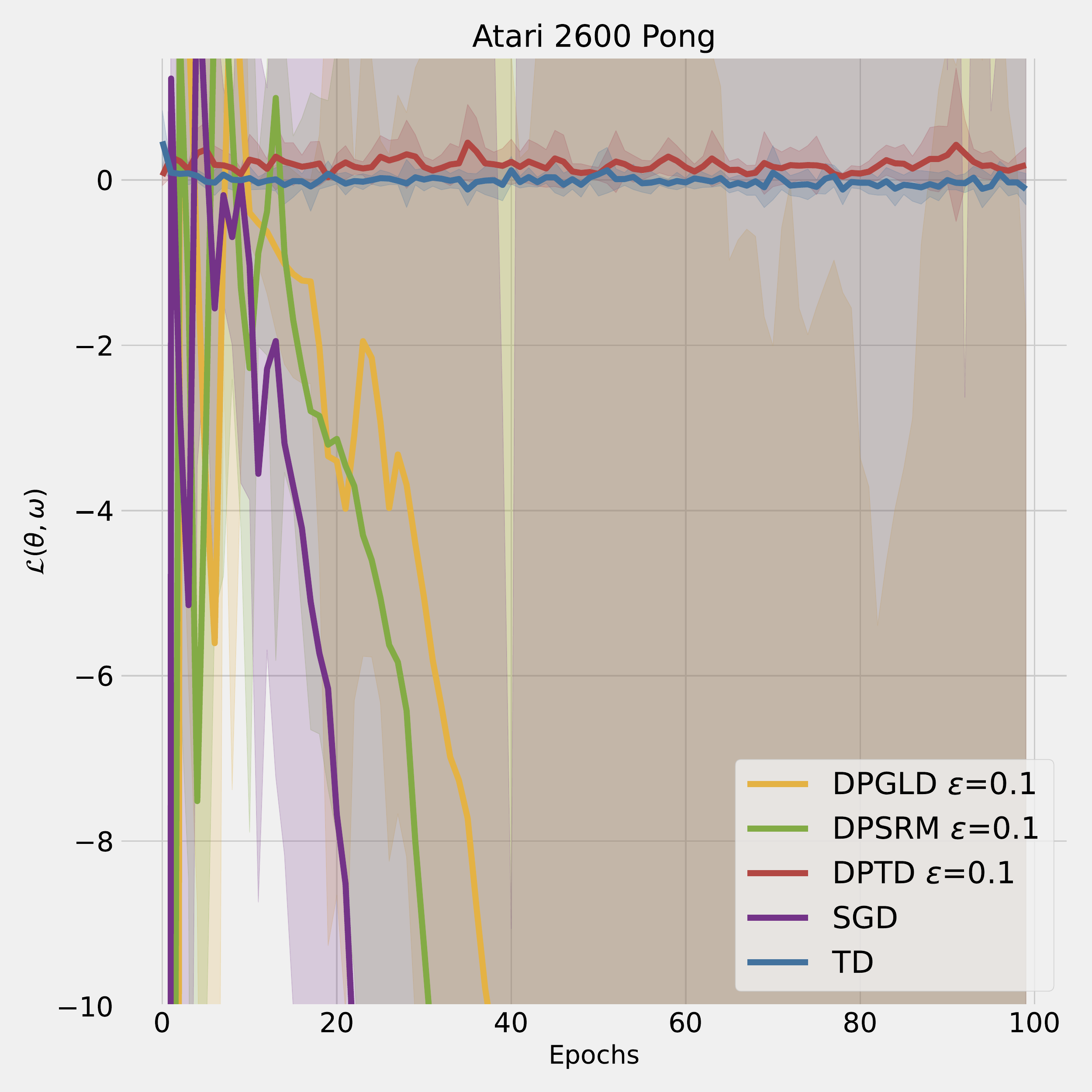}
%   \subcaption{}
   \label{fig:fig1_atari}
\end{minipage}
\caption{Compare DPTD with different algorithms (DPGLD, DPSRM, SGD, TD) on Cart Pole (a), Acrobot (b) and Atari 2600 Pong (c). 
% Figure \ref{fig: compare mountain car}, \ref{fig: compare cart pole} and \ref{fig: compare acrobot} 
Figure (a), (b) and (c) 
% These figures
show the value of objective function versus the number of epochs.  Each epoch has $5$ finite trajectories. The shadow denotes $1$-std. The learning curves are averaged over 10 random seeds and are generated without smoothing.}
\label{fig:fig_1}
\end{figure*}
% ---------------------------------------  fig 1 ---------------------------------------------------

\section{Experiments}
To validate the effectiveness of our algorithm, we conduct comprehensive empirical evaluations and present the experiment results in this section.

\subsection{Setting}

We justify our proposed algorithms empirically through classical control tasks: Cart Pole \cite{BartoSA83}, Acrobot \cite{GeramifardDKDH15} and Atari 2600 Pong
in  OpenAI Gym \cite{BrockmanCPSSTZ16} environments. 
All the algorithms 
are evaluated with data generated from Sarsa for Cart Pole and Acrobot and DQN for Atrari. To ensure that the generated trajectories are of good quality, we sample 5 trajectories for each environment.
% Since the high-dimensional inputs
% The maximum length of a trajectory depends on the simulation environment. For example, the length is at most $200$ in Mountain Car.
% \zhao{traj num \& avg. length}\ze{ok}

% DPTD outperforms DPGLD and DPSRM on three environments and converges fast. SGD performs slightly better than DPTD on Acrobot, which is accepted for SGD is not effected by the noise.

\subsection{Baselines}
% To evaluate the utility of DPTD, we select several algorithms as our baseline methods for comparison, including stochastic gradient descent (SGD) \zhao{cite}, differentially private gradient Langevin dynamics (DPGLD) \cite{wang2019differentially}, differentially private stochastic recursive momentum (DPSRM)  \cite{wang2019efficient}. We implement DPTD and other baseline algorithms on a local server using Pytorch 1.7.0  \cite {paszke2019pytorch} and Gym 0.18.0  \cite {brockman2016openai}. 
% The implementation of  other algorithms   under our setting is described as follows.  
Since our algorithm is the first differentially private temporal difference method, we have no relevant TD algorithms which can also achieve DP to compare. Thus,  we evaluate DPTD against several baseline methods in the DP ERM literature including differentially private gradient Langevin dynamics (DPGLD) \cite{wang2019differentially}, and differentially private stochastic recursive momentum (DPSRM)  \cite{wang2019efficient}. 
% To study the utility where there is no need to inject noises to achieve DP, 
% To study the utility where there is no need to achieve DP and thus no need to inject noises, 
% we also include stochastic gradient descent (SGD) \cite{GhadimiL13a,Nesterov04} as our baseline which is not injected by any noise and thus there is no privacy guarantee of it.
% We also include 
% the non-private TD as one of the baselines, which is a variant of our algorithm with no noises injected.
To study the utility where there is no need to achieve DP and thus no need to inject noises, 
we also include the non-private TD and 
stochastic gradient descent (SGD) \cite{GhadimiL13a,Nesterov04} 
as our baselines, which are not injected by any noise and thus there are no privacy guarantees of them.
% We also include 
%  as one of the baselines, which is a variant of our algorithm with no noises injected.
Though DPGLD, DPSRM, and SGD are designed for solving nonconvex optimization problems instead of nonconvex-strongly-concave primal-dual optimization problems, for a fair comparison, 
we also implement these baselines in the primal-dual form for comparing their performance with DPTD.
% Specifically, we implement these algorithms on primal side to minimize the objective function (\textit{i.e.}, Eq. (\ref{problem: min-max RL})) and also implement these algorithms on dual side to maximize the objective function. \zhao{to be fixed}
% \zhao{
Specifically, at each iteration, these algorithms are implemented to take a gradient descent step to  minimize the objective function (\textit{i.e.}, Eq. (\ref{problem: min-max RL})) on the primal side and simultaneously take a gradient ascent step to  maximize the objective function on the dual side. 
The value functions of all the algorithms are parameterized by a two-layer  fully-connected neural network with  $50$ hidden neurons and ELU activation function \cite{djork2016elu}.
Other implementation details are deferred to Appendix \ref{app:sec:exp_imple}.

% The parameters of all the algorithms are introduced as follows. The value function is parameterized by a two-layer  fully-connected neural network with  $50$ hidden neurons and ELU activation function \cite{djork2016elu}.
% The discount factor $\gamma$ is set to $0.95$ as in \cite{WaiHYWT19}. We set the feasible sets as  $\Theta=[-1,1]^d$ and $\Omega =[-1,1]^d$ where $d$ is the dimension of the neural network's parameters. For DPTD and TD, we set $\alpha=3$, $\beta=3$, $\kappa=2$, $\eta=2$, $\nu_t=\frac{1}{4(t+3)^{1/2}}$ as suggested in Theorem \ref{theorem: utility}.
% The step sizes of DPGLD and DPSRM are also taken as the suggested theoretical values in their original papers. For SGD, the step size 
% is maintained in the same order with other algorithms, ranging from $10^{-3}$ to $10^{-4}$.
% We implement all the algorithms in PyTorch 1.5.1  \cite{paszke2019pytorch} with Ubuntu 18.04 and an NVIDIA GTX 2080Ti GPU. 

% \zhao{subsections are needed here}
\subsection{Results and Analysis}
We report the experiment results in terms of utility in Figure \ref{fig:fig_1}, where the y-axis indicates the value of $\mathcal{L} (\theta, \omega)$ in Eq. (\ref{problem: min-max RL}) in the optimization process. 
% \zhao{tbd} 
The following conclusions are drawn in order. First, we observe that DPSRM, DPGLD and SGD can not converge well in all the three tasks, even though the gradients in SGD are not perturbed, since 
these methods
are not able to leverage the property of the primal-dual optimization problem inherently. In particular, one can see that the performance of DPSRM, DPGLD and SGD degrades heavily in Figure \ref{fig:fig_1} (c), perhaps due to the  high-dimensional state space  and the increasing complexity of the policies in the Atari task.
% This is reasonable since we do not inject any noise into SGD. 
Furthermore, DPTD converges faster in three tasks compared to DPSRM, DPGLD and SGD, which shows that DPTD has a better utility. 
% Finally, TD which is not injected by any noises has the best utility compared to all the other methods in all the tasks, whose values of Eq. (\ref{problem: min-max RL}) converge to $0$ rapidly.
Finally, TD without injected by any noises has the best utility compared to all the other methods in three tasks, whose values of Eq. (\ref{problem: min-max RL}) converge to $0$ rapidly. This is reasonable 
since TD is a non-private version of our algorithm.
% Besides, when the number of epochs is large, DPTD has a better utility compared to DPGLD and DPSRM under the same privacy budget and is still comparable with the non-noisy algorithm SGD.

% \zhao{tbd}
Furthermore, to study the impact of different privacy budgets on convergence, we conduct experiments to show the utility of DPTD with varying $\epsilon$ and report the experiment results in Figure \ref{fig: DP}, defered to Appendix \ref{app:sec:exp_varying_eps}. One can see that as the privacy parameter $\epsilon$ decreases from $100.0$ to $0.1$, the variance of Gaussian noises 
increases and the performance of DPTD begins to degrade,
matching our theoretical analysis.
% which matches our theoretical analysis. 
% This phenomenon matches our theoretical analysis. 
% The experiment results show that DPTD ensures a better utility under the same privacy budget compared to DPGLD and DPSRM and has the similar or better convergence rate compared to the non-noisy algorithm, \textit{i.e.}, SGD. 
% For the second set,  different values of $\epsilon$ in DPTD are tested shown in Figure \ref{fig: DP}. 
% As the privacy parameter $\epsilon$ decreases from $10$ to $0.1$, the variance of Gaussian noises 
% increases and the corresponding performance of DPTD becomes less stable. The phenomenon matches our theoretical analysis. 

% provide strong evidence that variance

% \zhao{the analysis to the exp results seems not insightful enough}

% As $\epsilon$ increases, the variance of the noises decreases, which leads to a better performance. 

\section{Conclusions}
% In this paper, we develop a rigorous and efficient algorithm for differentially private primal-dual temporal difference (TD) learning, which protects the critical state transitions in reinforcement learning (RL) and makes two neighboring trajectories indistinguishable. By adding Gaussian noise to the momentum-based gradient at each iteration, we achieve differential privacy (DP) and the utility upper bounded by $\widetilde{\mathcal{O}}(\frac{(d\log(1/\delta))^{1/8}}{(n\epsilon)^{1/4}})$. Another natural case that the full trajectory needs to be protected is also discussed, armed with the theoretical analysis \zhao{this sentence seems a bit confusing}. Comprehensive experiments corroborate our theoretical results and show the superiority of our proposed algorithm \zhao{some sentences seem to coincide with the abstract}. In our future work, we will further consider how to achieve DP under Markovian sampling \zhao{need to be more specific}. \ze{can we mention this? to be discussed}
In this paper, we make the first step to develop an efficient algorithm for differentially private primal-dual temporal difference (TD) learning, which protects the critical state transitions in reinforcement learning (RL) so as to make two neighboring trajectories indistinguishable and simultaneously achieve fast convergence rate. 
% Another natural case that the full trajectory needs to be protected is also discussed, armed with the theoretical analysis \zhao{this sentence seems a bit confusing}. 
We also show that our algorithm can achieve differential privacy (DP) with a bounded utility under the
case where the full trajectory needs to be protected.
% By adding the Gaussian noise with a carefully chosen variance to the momentum-based gradient estimator at each iteration, our algorithm achieves differential privacy (DP) with the utility upper bounded by $\widetilde{\mathcal{O}}(\frac{(d\log(1/\delta))^{1/8}}{(n\epsilon)^{1/4}})$.
The privacy guarantee and the utility guarantee of our algorithm are validated by both the rigorous theoretical analysis and comprehensive experiments conducted in three OpenAI Gym environments.
% \zhao{some sentences seem to coincide with the abstract}. 
% In our future work, we will further consider how to achieve DP under Markovian sampling 
% In our future work, we will further study how to simultaneously achieve DP and keep a fast convergence rate of TD learning under Markovian sampling.
In our future work, we are interested in how to simultaneously achieve DP and keep a fast convergence rate of TD learning with nonlinear smooth function approximation under Markovian sampling.
% \zhao{need to be more specific}. 

% \bibliographystyle{alpha}
% \bibliography{sample}

\bibliographystyle{named}
\bibliography{ijcai22_comp}

\appendix
\onecolumn
% \section{Appendix}
% Proofs of Main Results

\section{Discussion on Privacy and Utility under Trajectory}
\label{sec:under_trajectory}
\subsection{Privacy and Utility Analysis}
%  Although the aim of protecting one single state transition motivates  us to study the utility and the privacy under state-action-state, 
%  it is natural to consider two neighbouring datasets  which consist of several trajectories and only differ in one trajectory, such as \cite{balle2016differentially}. \zhao{i will revise here} We confirm that this case should be discussed for the completeness of our differentially private algorithm and the formal definition of privacy under trajectory has been described in Definition \ref{def: TDP2}.

% Although the aim of protecting one single state transition motivates  us to study the utility and the privacy under state-action-state, 
% in medical applications individual patients generate full trajectories \cite{balle2016differentially} and thus it is necessary to consider two neighboring datasets which consist of several trajectories and only differ in on trajectory (\textit{i.e.}, DP under trajectory defined in Definition \ref{def: TDP2}).

 While our algorithm runs online with the utility and the privacy under state-action-state as discussed before, a more general definition is required for the offline setting where multiple trajectories are presented in one dataset, which motivates us to consider DP under trajectory in Definition \ref{def: TDP2}. 
 %Thus, the first advantage of DP under trajectory is the extension to the offline setting for our algorithm. \zhao{`the first advantage of DP...' seems not very natural}

 Another motivation of DP under trajectory is about the assumption, including two aspects. From the first aspect, making data points in one trajectory independent and identically distributed (\textit{i.e.}, Assumption \ref{asmp: iid sampling}) cannot be easily satisfied in practice, caused by the property of the Markov chain. However, sampling different trajectories independently under the identical distribution can be achieved for less dependency between trajectories. From the second aspect, constraining the stochastic gradients (\textit{i.e.}, Assumption \ref{asmp: G-lispchitz}) is not necessary for DP under trajectory and this assumption can be replaced by a weaker one, shown below.

The following assumption gives a weaker version of Assumption \ref{asmp: G-lispchitz}, helping bound the  averaged stochastic gradients. This assumption is a necessary  but not sufficient condition for Assumption \ref{asmp: G-lispchitz}.
\begin{asmp}[Averaged G-Lipschitz]
\label{asmp: upper bound of summation}
Given a full trajectory $\tau$ with length $|\tau|\leq n$, 
% for all $(\theta, \omega)$,
$\forall \theta\in\Theta$ and $\forall \omega\in\Omega$,
\begin{align*}
\left\|\sum_{i=0}^{ |\tau| -1}\nabla_\theta f (\theta,\omega;\xi_{i}) \right\|\leq nG\,,\ \left\|\sum_{i=0}^{ |\tau| -1}\nabla_\omega f (\theta,\omega;\xi_{i}) \right\|\leq nG
\end{align*}
% \begin{align*}
% \left\|\sum_{i=0}^{ |\tau| -1}\nabla_\theta f (\theta,\omega;\xi_{i}) \right\|\leq nG\,,
% \end{align*}
% \begin{align*}
% \left\|\sum_{i=0}^{ |\tau| -1}\nabla_\omega f (\theta,\omega;\xi_{i}) \right\|\leq nG\,.
% \end{align*}
% where $\xi_i$ is a state-action-state pair in $\tau$, $n$ is the maximum length of $\tau$, and $G$ is a constant.
holds for some $G>0$.
\end{asmp}
% We confirm that this case should be discussed for the completeness of our differentially private algorithm and the formal definition of privacy under trajectory has been described in Definition \ref{def: TDP2}.
% \zhao{needs to clarify }
 
% The following two theorems give the privacy guarantee and the utility under trajectory. The variance of Gaussian noises injected depends on the max length of one trajectory $n$ and the number of trajectories $m$. The utility under trajectory of Algorithm \ref{alg: DPTD}, \textit{i.e.}, $\widetilde{\mathcal{O}} (\frac{n^\frac{7}{4} (d\log (1/\delta) ) ^\frac{1}{8}}{ (m\epsilon) ^\frac{1}{4}}) $, maintains the almost same order as the utility under state-action-state, \textit{i.e.}, $ \widetilde{\mathcal{O}}(\frac{(d\log(1/\delta) ) ^\frac{1}{8}}{(n\epsilon) ^\frac{1}{4}}) $. The difference in $m$ and $n$ is caused by the definition and  the value of $T$ set in Theorem \ref{theorem: utility} and \ref{theorem: utility in TDP}. 

% From the utility we know that Algorithm \ref{alg: DPTD} still achieves differential privacy with an acceptable utility. \zhao{i will revise here} However, the utility's factor $n^\frac{7}{4}$ should be emphasised, which tells that if $n$ is great large, the utility will be bad. 

% \zhao{i will revise here}
% Armed with the above assumption, the theorems about the privacy and the utility under trajectory are given as follows.
Armed with the above assumption, the theorems providing the privacy and the utility guarantee under trajectory are given as follows. 
\begin{theorem}[Privacy under trajectory]
\label{theorem: privacy guarantee (TDP) }
Consider the DP defined in Definition \ref{def: TDP2}.
% Under Assumption \ref{asmp: feasible convex set}, \ref{asmp: iid sampling}, \ref{asmp: upper bound of summation}, given the total number of iterations $T$, for any $\delta>0$ and the privacy budget $\epsilon$, Algorithm \ref{alg: DPTD} satisfies $(\epsilon,\delta) $-DP under trajectory if we set the variance of Gaussian noises equal to
Under Assumption \ref{asmp: feasible convex set}, \ref{asmp: iid sampling}, \ref{asmp: upper bound of summation}, given the total number of iterations $T$, for any $\delta>0$ and the privacy budget $\epsilon$, Algorithm \ref{alg: DPTD} satisfies $(\epsilon,\delta) $-DP under trajectory with the variance
\begin{align*}
\sigma^2_t = \frac{14n^2G^2T\alpha'}{m^2\left(\epsilon - \frac{\log(1/\delta) }{\alpha'-1}\right) },\ \forall t\geq 0\,,
\end{align*}
% $$
% \sigma^2_t = \frac{14n^2G^2T\alpha'}{m^2(\epsilon - \frac{\log(1/\delta) }{\alpha'-1}) }, for\ t\geq 0\,,
% $$
where $\sigma'^2=\frac{\sigma_t^2}{4n^2G^2}\geq 0.7$, $\alpha'=\frac{\log(1/\delta) }{(1-\beta') \epsilon}+1\leq 2\sigma^2\log(\frac{n}{\alpha'(1+\sigma'^2) }) /3+1$ , $\beta'\in (0,1) $, $n$ is the maximum trajectory length and $m$ is the number of trajectories.
\end{theorem}
% Theorem \ref{theorem: privacy guarantee (TDP) } gives the privacy guarantee of Algorithm \ref{alg: DPTD} under trajectory. 
% One can see that the variance of Gaussian noise under trajectory grows as $n$ increases. 
% It is reasonable since in DP under trajectory since it requires to protect the privacy of two trajectories which have $n$ different state-action-state triples
% in the worst-case scenario.
Theorem \ref{theorem: privacy guarantee (TDP) } gives the privacy guarantee of Algorithm \ref{alg: DPTD} under trajectory. 
One can see that the variance of Gaussian noises under trajectory grows as $n$ increases. 
It is reasonable since in DP under trajectory it requires to protect the privacy of two trajectories which have $n$ different state-action-state triples
in the worst-case scenario.

The utility under trajectory of our algorithm is presented in the following theorem.
\begin{theorem}[Utility under trajectory]
\label{theorem: utility in TDP}
% Under Assumptions \ref{asmp: differentiable}, \ref{asmp: existence of solution}, \ref{asmp:wlipschitz},  \ref{asmp:bounded},  setting the parameters $\alpha = \beta = 3$,  $0<\eta \leq \mu / (4L^2_F) $, $0<\kappa\leq \eta \mu^2/ (9L^2_F) $, and $\nu_t = 1/4 (t+b) ^\frac{1}{2}$
% with $b\geq \max\{ (2\kappa L_F^2/\mu) ^2,3\}$, with the updating rules in Algorithm \ref{alg: DPTD} and the Gaussian variance  in Theorem \ref{theorem: privacy guarantee  (TDP) }, we set \zhao{i will revise here}
% \begin{align*}
% T = \frac{Cmn\epsilon}{\sqrt{d\log (1/\delta) }}\,,
% \end{align*}
% where $C$ is a constant. Then the output $\theta$, $\omega$ of Algorithm \ref{alg: DPTD} satisfies the following
Under Assumption \ref{asmp: existence of solution}, \ref{asmp:wlipschitz}, \ref{asmp: feasible convex set}, \ref{asmp: strongly concave}, \ref{asmp: iid sampling}, \ref{asmp: upper bound of summation}, if we set the parameters $\alpha = \beta = 3$,  $0<\eta \leq \mu / (4L^2_F) $, $0<\kappa\leq \eta \mu^2/ (9L^2_F) $, $\nu_t = 1/4 (t+b) ^\frac{1}{2}$ with $b\geq \max\{ (2\kappa L_F^2/\mu) ^2,3\}$ and choose the number of iterations $T = \frac{Cm\epsilon}{n\sqrt{d\log (1/\delta) }}$ where $C$ is a constant, then with the Gaussian variance  in Theorem \ref{theorem: privacy guarantee  (TDP) }, 
% we set 
% \zhao{i will revise here}
% \begin{align*}
% T = \frac{Cmn\epsilon}{\sqrt{d\log (1/\delta) }}\,,
% \end{align*}
the output of Algorithm \ref{alg: DPTD} satisfies the following
\begin{align*}
\mathbb{E}\left\|\mathfrak{M}(\bar{\theta},\bar{\omega})\right\|\leq \widetilde{\mathcal{O}} \left(\frac{n^\frac{1}{4} (d\log (1/\delta) ) ^\frac{1}{8}}{ (m\epsilon) ^\frac{1}{4}}\right) \,.
\end{align*}
% \ze{with the gradient complexity equal to $2(T+1)=\mathcal{O}(\frac{m\epsilon}{n\sqrt{d\log (1/\delta) }})$}
Moreover, the  total gradient complexity of Algorithm \ref{alg: DPTD} is $2(T+1)=\mathcal{O}\left(\frac{m\epsilon}{n\sqrt{d\log (1/\delta) }}\right)$.
\end{theorem}
% The utility of Algorithm \ref{alg: DPTD} under trajectory  maintains  almost the same order as the utility under state-action-state, \textit{i.e.}, $ \widetilde{\mathcal{O}}(\frac{(d\log(1/\delta) ) ^\frac{1}{8}}{(n\epsilon) ^\frac{1}{4}}) $. The difference in $m$ and $n$ is caused by the different definitions of DP and the value of $T$ set in Theorem \ref{theorem: utility} and \ref{theorem: utility in TDP}. 
% One can see that Algorithm \ref{alg: DPTD} still achieves differential privacy with an acceptable utility. 
% \zhao{i will revise here} 
% However, the utility's factor $n^\frac{7}{4}$ should be emphasised, which tells that if $n$ is great large, the utility will be bad. 
Compared to the utility upper bound    under state-action-state $\widetilde{\mathcal{O}}\left(\frac{(d\log(1/\delta) ) ^\frac{1}{8}}{(n\epsilon) ^\frac{1}{4}}\right)$, 
the utility upper bound  under trajectory is  worse by a factor of $\mathcal{O}\left(n^{\frac{1}{2}}\right)$
% \ze{Revise: $\mathcal{O}(n^\frac{1}{2})$} 
since larger dependence on $n$ of variance of Gaussian noises is needed to protect the privacy under trajectory than under state-action-state.
% However, considering $m$ is usually much larger than $n$ in practice, the utility under trajectory is also acceptable.
% However, the utility under trajectory will be still acceptable if $m$ is larger than $n$, which is usually the case in practice due to the sample inefficiency in RL \cite{kakade2003sample}. 
However, the utility under trajectory will be still acceptable if $m$ is larger than $n$, which is possible in practice due to the sample inefficiency of RL algorithms \cite{kakade2003sample}.

% \subsection{Comparisons with DP over Initial Visitation}
\subsection{Comparisons with DP over Initial Visitation Estimate}
\citet{balle2016differentially} consider preserving the privacy in the definition of DP over initial visitation estimate, where they strictly restrict that two different trajectories can only differ in one state transition in two neighboring datasets.
However, DP under trajectory in Definition \ref{def: TDP2} allows that
two different trajectories
can differ in at most $n$ transitions in two neighboring datasets.
Furthermore, they aim to preserve privacy in PE with linear function approximation, while we consider preserving privacy 
in PE with nonlinear function approximation.

% \zhao{cite}\ze{is this ok?}
% the dependence with respect to $n$ of the variance of Gaussian noise under trajectory is $\mathcal{O}()$
%  The difference in $m$ and $n$ is caused by the different definitions of DP and the value of $T$ set in Theorem \ref{theorem: utility} and \ref{theorem: utility in TDP}. 
% However, the utility's factor $n^\frac{7}{4}$ should be emphasised, which tells that if $n$ is great large, the utility will be bad. 
% \zhao{i will revise here, to be fixed}

% \subsection{Comparison between two DP}
% \begin{lemma}[Group Privacy]
% If  f  is  $(\epsilon, \delta)$-DP , then for any  $\left\|x-x^{\prime}\right\|_{1} \leq t$,

% $$\mathbb{P}(f(x) \in S) \leq e^{t \epsilon} \mathbb{P}\left(f\left(x^{\prime}\right) \in S\right)+t e^{t \epsilon} \delta$$
% \end{lemma}

% \subsection{Proof of Theorem \ref{theorem: privacy guarantee}}
% \alphSection{Proof of Theorem \ref{theorem: privacy guarantee}}
\section{Proof of Theorem \ref{theorem: privacy guarantee}}\label{sec:app:pf_thm_sas_privacy}
% In this section, we provide the proof of Theorem  \ref{theorem: privacy guarantee}, which gives the privacy guarantee of Algorithm \ref{alg: DPTD}. 
% To this end, we need the following rules
The formal definition of $\ell_2$-sensitivity is given as follows.
\begin{defn}[$\ell_2$-sensitivity \cite{dwork2014algorithmic}]
\label{def: l2-sensitivity}
The $\ell_2$-sensitivity $\Delta (g) $ of a function $g$ is defined as $\Delta (g) =\sup _{X, X^{\prime}}\left\|g (X) -g\left (X'\right) \right\|$, 
for any two neighbouring datasets $X\subseteq \mathcal{X}$ and $X'\subseteq \mathcal{X}$.
\end{defn}

Before proving Theorem  \ref{theorem: privacy guarantee}, we first present the following auxiliary lemmas. Lemma \ref{lemma: composition of RDP} shows that the mechanism satisfies RDP if this mechanism is a composition of a series of mechanisms which satisfy RDP. 
\begin{lemma}[\cite {mironov2017renyi}] 
\label{lemma: composition of RDP}
If $k$ randomized mechanisms $\mathcal{M}_i:\mathcal{X}\rightarrow \mathcal{Y}$ for $i\in [k]$, satisfy $  (\alpha,\rho_i) $-RDP, then their composition $\left  (\mathcal{M}_1  (X) ,\cdots,\mathcal{M}_k  (X)  \right) $ satisfies $  (\alpha,\sum_{i=1}^k \rho_i) $-RDP for $X\subseteq \mathcal{X}$. Moreover, the input of $i$-th mechanism can base on the outputs of previous $  (i-1) $ mechanisms.
\end{lemma}
Based on Lemma \ref{lemma: RDP_to_DP}, one can establish a DP privacy guarantee of one mechanism by leveraging the privacy guarantee in terms of RDP.
\begin{lemma}[\cite {mironov2017renyi}]
\label{lemma: RDP_to_DP}
If a randomized mechanism $\mathcal{M}:\mathcal{X}\rightarrow \mathcal{Y}$ satisfies $  (\alpha,\rho) $-RDP, then $\mathcal{M}$ satisfies $  (\rho+\log   (1 / \delta)  /  (\alpha-1) , \delta) \text {-DP for all } \delta \in  (0,1) $.
\end{lemma}
 In the online setting, it is unrealistic to access all the samples in a dataset via one query.
%\dong{As the data is generated by a Markov chain and we consider the online reinforcement learning setting, we essentially have an infinite dataset, which makes query of true gradient impossible. 
%}
% \zhao{
For instance, the agent needs to update the approximation of the value function after the agent experiences a new state-action-state pair in TD learning. If a mechanism works under the samples that are subsampled from the whole dataset instead of the whole dataset, this mechanism is considered to use subsampling.
% }
Lemma \ref{lemma: RDP subsampling transformation} can transform the RDP privacy guarantee for a mechanism without subsampling to the RDP privacy guarantee for the mechanism using uniform subsampling.
\begin{lemma}[\cite {wang2019efficient}]
\label{lemma: RDP subsampling transformation}
Given a function $q:\mathcal{S}^n\rightarrow \mathcal{R}$, then Gaussian Mechanism $\mathcal{M}=q  (S) +\mathbf{u}$, where $\mathbf{u} \sim N  (0,\sigma^2\mathbf{I}) $, satisfies $  (\alpha,\alpha\Delta^2  (q) /  (2\sigma^2) ) $-RDP. In addition, if we apply the mechanism $\mathcal{M}$ to a subset of samples using uniform sampling without replacement with sampling rate $\tau$, $\mathcal{M}$ satisfies $  (\alpha,3.5\tau^2\Delta^2  (q) \alpha/\sigma^2) $-RDP given $\sigma'^2=\sigma^2/\Delta^2  (q) \geq 0.7$, $\alpha \leq 2\sigma^2\log  (1/\tau\alpha  (1+\sigma'^2) ) /3+1$.
\end{lemma}

In our main proof of privacy guarantee, we will first prove that our Algorithm \ref{alg: DPTD} satisfies RDP based on Lemma \ref{lemma: RDP subsampling transformation} and Lemma \ref{lemma: composition of RDP}. 
% Then we will transform RDP into DP using Lemma \ref{lemma: RDP_to_DP}. 
Then we show that Algorithm \ref{alg: DPTD} satisfies DP using Lemma \ref{lemma: RDP_to_DP}.

\begin{proof}[Proof of Theorem \ref{theorem: privacy guarantee}]
% For a given trajectory $S$, we use $S'$ to denote its neighbouring trajectory with only one different data point index by $i'$
%\dong{perhaps we want to specify that by data point indexed by i, we meant a   (s, a, s')  pair indexed by i?} 
% in the following discussion. 
% Recall the definition of SASDP in Definition \ref{def: SASDP2}.
% According to Algorithm \ref{alg: DPTD}, we use the following $\mathcal{M}_t$ to denote the mechanism at $t$-th iteration.
Let $\mathcal{M}^p_t$ and $\mathcal{M}^d_t$ be the privacy protection mechanisms on primal side and dual side respectively at the $t$-th iteration constructed by the update rules in in Algorithm \ref{alg: DPTD}, \textit{i.e.},
\begin{equation}
\label{eq: random mechanism of p_t}
    \mathcal{M}_{t}^p = \left\{
    \begin{aligned}
    &  (1-\alpha\nu_{t-1}) p_{t-1} +\alpha\nu_{t-1}\nabla_\theta f  (\theta_{t},\omega_{t};\xi_{t}) +u_{t}^p,&t> 0\\
    &\nabla_\theta f  (\theta_{0},\omega_{0};\xi_{0}) +u_{0}^p,&t=0
    \end{aligned}\right.\,,
\end{equation}
and
\begin{equation}
\label{eq: random mechansim of d_t}
\mathcal{M}_{t}^d =\left\{
    \begin{aligned}
    &  (1-\beta\nu_{t-1}) d_{t-1} +\beta\nu_{t-1}\nabla_\omega f  (\theta_{t},\omega_{t};\xi_{t}) +u_{t}^d,&t>0\\
    &\nabla_\omega f  (\theta_{0},\omega_{0};\xi_{0}) +u_{0}^d,&t=0
    \end{aligned}\right.\,.
\end{equation}
% Our goal is to show the privacy guarantee of both $\mathcal{M}_t^p$ and $\mathcal{M}_t^d$ for $t=0,1,2,\cdots,T-1$. 
We first show the mechanism on primal side $\mathcal{M}_t^p$ satisfies the privacy guarantee for $t=0,1,2,\cdots,T-1$.

\noindent\textbf{Case   (a) .} 
If $t=0$, we have
\begin{align*}
    \mathcal{M}_0^p =  \nabla_\theta f  (\theta_{0},\omega_{0};\xi_{0}) +u_{0}^p\,.
\end{align*}
% $$
% \mathcal{M}_0^d=\nabla_\omega f  (\theta_{0},\omega_{0};\xi_{0}) +u_{0}^d\,,
% $$
% Therefore, We first consider the following Gaussian mechanism $\mathcal{G}_t^p$ which is
% \begin{align*}
%     \mathcal{G}_t^p = \nabla_\theta f  (\theta_{t},\omega_{t};\xi_{t}) +u_{t}^p\,.
% \end{align*}
% it is clear that 
% \begin{align*}
%     \mathcal{M}_0^p = \mathcal{G}_0^p = \nabla_\theta f  (\theta_{0},\omega_{0};\xi_{0}) +u_{0}^p\,.
% \end{align*}
Therefore, we first consider the following Gaussian mechanism
\begin{align*}
    \mathcal{G}_0^p = \nabla_\theta f  (\theta_{0},\omega_{0};\xi_{0}) +u_{0}^p\,,
\end{align*}
% $$
% \mathcal{G}_0^p = \nabla_\theta f  (\theta_{0},\omega_{0};\xi_{0}) +u_{0}^p\,,
% $$
% $$
% \mathcal{G}_0^d=\nabla_\omega f  (\theta_{0},\omega_{0};\xi_{0}) +u_{0}^d\,,
% $$
% where $u_{0}^p\sim N  (0,\sigma_0^2\mathbf{I}_d) $, $u_{0}^d\sim N  (0,\sigma_0^2\mathbf{I}_d) $. Note that the two mechanisms are both based on the subsampling, thus we will first consider the mechanisms without subsampling and get the final RDP by using Lemma \ref{lemma: RDP subsampling transformation}. 
where $u_{0}^p\sim N  (0,\sigma_0^2\mathbf{I}_d) $, $u_{0}^d\sim N  (0,\sigma_0^2\mathbf{I}_d) $. Note that $\mathcal{G}_0^p$ is based on the subsampling. Hence we will first consider the mechanisms without subsampling and get the final RDP by using Lemma \ref{lemma: RDP subsampling transformation}. 
% To this end, we consider the following Gaussian mechanisms without subsampling
Specifically, we consider the following Gaussian mechanism without subsampling
\begin{align*}
    \widetilde{\mathcal{G}}_0^p =\sum_{i=0}^{n-1} \nabla_\theta f  (\theta_{0},\omega_{0};\xi_{i}) +u_{0}^p\,.
\end{align*}
% $$
% \widetilde{\mathcal{G}_0^p} =\sum_{i=0}^{n-1} \nabla_\theta f  (\theta_{0},\omega_{0};\xi_{i}) +u_{0}^p\,,
% $$
% $$
% \widetilde{\mathcal{G}_0^d}=\sum_{i=0}^{n-1}\nabla_\omega f  (\theta_{0},\omega_{0};\xi_{i}) +u_{0}^d\,,
% $$
\textbf{Sensitivity.} Consider the query on the trajectory $\tau$ in $S$ as follows
\begin{align*}
    \widetilde{q}_0^p  (S) =\sum_{i=0}^{n-1}\nabla_\theta f  (\theta_{0},\omega_{0};\xi_{i}) \,.
\end{align*}
Similarly, we can get $\widetilde{q}_0^p  (\hat{S}) $ where $\hat{S}$ is one of $S$'s neighbouring datasets as defined in Definition \ref{def: SASDP2}.
% and $\widetilde {q_0^d}  (S') $. 
Thus, we have
\begin{align*}
\widetilde{q}_0^p  (S)  - \widetilde{q}_0^p  (\hat{S}) =\nabla_\theta f  (\theta_{0},\omega_{0};\xi_{i})  - \nabla_\theta f  (\theta_{0},\omega_{0};\hat{\xi}_{i}) \,.
\end{align*}
% $$
% \widetilde {q_0^p}  (S)  - \widetilde {q_0^p}  (S') =\nabla_\theta f  (\theta_{0},\omega_{0};\xi_{i})  - \nabla_\theta f  (\theta_{0},\omega_{0};\xi_{i'}) \,,
% $$
% $$
% \widetilde {q_0^d}  (S)  - \widetilde {q_0^d}  (S') =\nabla_\omega f  (\theta_{0},\omega_{0};\xi_{i})  - \nabla_\omega f  (\theta_{0},\omega_{0};\xi_{i'}) \,.
% $$

% Since each stochastic gradient 
%\dong{Shall we refer $\nabla f$ as stochastic gradient or something else to avoid confusion? }
% is $G$-Lipschitz, we can obtain the $\ell_{2}$-sensitivity of this query as follows
% The $G$-Lipschitz continuity of each component function implies that
Then Assumption \ref{asmp: G-lispchitz} implies that
\begin{align*}
    \widetilde{\Delta}_{0}^p=\left\|\nabla_\theta f  (\theta_{0},\omega_{0};\xi_{i})  - \nabla_\theta f  (\theta_{0},\omega_{0};\hat{\xi}_{i}) \right\| \leq 2G\,.
\end{align*}
% $$
% \widetilde{\Delta_{0}^p}=\left\|\nabla_\theta f  (\theta_{0},\omega_{0};\xi_{i})  - \nabla_\theta f  (\theta_{0},\omega_{0};\xi_{i'}) \right\| \leq 2G\,,
% $$
% $$
% \widetilde{\Delta_{0}^d}=\left\|\nabla_\omega f  (\theta_{0},\omega_{0};\xi_{i})  - \nabla_\omega f  (\theta_{0},\omega_{0};\xi_{i'}) \right\| \leq 2G\,.
% $$
% \zhao{20210825}
% Actually, in this case the primal side is  the same as the dual side. We can omit one side for the simplicity.

\noindent\textbf{Privacy guarantee of $\mathcal{G}_0^p$.}  
% By Lemma \ref{lemma: RDP subsampling transformation}, if the Gaussian noises $u_0^p$ and $u_0^d$ have the following variance
By Lemma \ref{lemma: RDP subsampling transformation}, if the Gaussian noise $u_0^p$ has the following variance
\begin{align*}
    \sigma^2_0 = \frac{14G^2T\alpha'}{n^2  \left(\epsilon - \frac{\log  (1/\delta) }{\alpha'-1}\right) }\,,
\end{align*}
where $\sigma'^2=\frac{\sigma^2}{4G^2}\geq 0.7$ and $\alpha'\leq 2\sigma^2\log  (\frac{n}{\alpha'  (1+\sigma'^2) }) /3+1$, then $\mathcal{G}_0^p$ will satisfy $  (\alpha', \frac{14\alpha'G^2}{n^2\sigma_0^2}) $-RDP.

\noindent\textbf{Case   (b) .} 
If $t>0$, we have
\begin{align*}
    \mathcal{M}_t^p =   (1-\alpha\nu_{t-1}) p_{t-1} +\alpha\nu_{t-1}\nabla_\theta f  (\theta_{t},\omega_{t};\xi_{t}) +u_{t}^p\,,
\end{align*}
% $$
% \mathcal{M}_t^p =   (1-\alpha\nu_{t-1}) p_{t-1} +\alpha\nu_{t-1}\nabla_\theta f  (\theta_{t},\omega_{t};\xi_{t}) +u_{t}^p\,,
% $$
% $$\mathcal{M}_t^d=
%   (1-\beta\nu_{t-1}) d_{t-1} +\beta\nu_{t-1}\nabla_\omega f  (\theta_{t},\omega_{t};\xi_{t}) +u_{t}^d\,.
% $$
% Then we consider the following Gaussian mechanism
which suggests us considering the following Gaussian mechanism
\begin{align*}
    \mathcal{G}_t^p = \alpha\nu_{t-1}\nabla_\theta f  (\theta_{t},\omega_{t};\xi_{t}) +u_{t}^p\,.
\end{align*}
% To obtain the sensitivity \dong{sensitivity of...?}, we first consider the following query on the whole dataset, without subsampling 
Since the mechanism $\mathcal{G}_t^p$ uses subsampling, we first consider the following mechanism on the whole dataset without subsampling 
\begin{align*}
    \widetilde{\mathcal{G}}_t^p   (S)  =\alpha\nu_{t-1} \sum_{i=0}^{n-1}\nabla_\theta f  (\theta_{t},\omega_{t};\xi_{i}) +u_t^p\,.
\end{align*}
% $$
% \mathcal{G}_t^p = \alpha\nu_{t-1}\nabla_\theta f  (\theta_{t},\omega_{t};\xi_{t}) +u_{t}^p\,,
% $$
% $$
% \mathcal{G}_t^d = \beta\nu_{t-1}\nabla_\omega f  (\theta_{t},\omega_{t};\xi_{t}) +u_{t}^d\,.
% $$
\textbf{Sensitivity.} % \dong{I guess it can be confusing if we mention subsampling? since we've never formally defined what's subsampling}
% \zhao{i added the ``subsampling" in Lemma \ref{lemma: RDP subsampling transformation}.}
Consider the following query without subsampling on the whole dataset
\begin{align*}
    \widetilde{q}_t^p  =\alpha\nu_{t-1} \sum_{i=0}^{n-1}\nabla_\theta f  (\theta_{t},\omega_{t};\xi_{i}) \,.
\end{align*}
% $$
% \widetilde{q_t^p} \dong{\widetilde{q}_t^p}  (S)  =\alpha\nu_{t-1} \sum_{i=0}^{n-1}\nabla_\theta f  (\theta_{t},\omega_{t};\xi_{i}) +u_{t}^p\,,
% $$
% $$
% \widetilde{q_t^d}  (S)  =\beta\nu_{t-1} \sum_{i=0}^{n-1}\nabla_\omega f  (\theta_{t},\omega_{t};\xi_{i}) +u_{t}^d\,,
% $$
Similarly, we can get $\widetilde{q}_t^p  (S^\prime) $. 
% and $\widetilde{q_t^d}  (S') $. 
Thus, we have
\begin{align*}
    \widetilde {q}_t^p  (S)  - \widetilde{q}_t^p  (S^\prime) =\alpha\nu_{t-1}\left  (\nabla_\theta f  (\theta_{t},\omega_{t};\xi_{i})  - \nabla_\theta f  (\theta_{t},\omega_{t};\hat{\xi}_{i}) \right) \,.
\end{align*}
% $$
% \widetilde {q_t^p}  (S)  - \widetilde {q_t^p}  (S') =\alpha\nu_{t-1}\left  (\nabla_\theta f  (\theta_{t},\omega_{t};\xi_{i})  - \nabla_\theta f  (\theta_{t},\omega_{t};\xi_{i'}) \right) \,,
% $$
% $$
% \widetilde {q_t^d}  (S)  - \widetilde {q_t^d}  (S') =\beta\nu_{t-1}\left  (\nabla_\omega f  (\theta_{t},\omega_{t};\xi_{i})  - \nabla_\omega f  (\theta_{t},\omega_{t};\xi_{i'}) \right) \,.
% $$
Then we can obtain the $\ell_2$-sensitivity of the query $\widetilde {q}_t^p$ as follows
\begin{align*}
\widetilde{\Delta}_t^p& = \left\| \alpha\nu_{t-1}\left  (\nabla_\theta f  (\theta_{t},\omega_{t};\xi_{i})  - \nabla_\theta f  (\theta_{t},\omega_{t};\xi_{i'}) \right) \right\|\\
&\leq 2\alpha\nu_{t-1}G\leq 2G\,,
\end{align*}
% $$
% \begin{aligned}
% \widetilde{\Delta_t^p}& = \left\| \alpha\nu_{t-1}\left  (\nabla_\theta f  (\theta_{t},\omega_{t};\xi_{i})  - \nabla_\theta f  (\theta_{t},\omega_{t};\xi_{i'}) \right) \right\|\\
% &\leq 2\alpha\nu_{t-1}G\leq 2G\,,
% \end{aligned}
% $$
% $$
% \begin{aligned}
% \widetilde{\Delta_t^d} &= \left\| \beta\nu_{t-1}\left  (\nabla_\omega f  (\theta_{t},\omega_{t};\xi_{i})  - \nabla_\omega f  (\theta_{t},\omega_{t};\xi_{i'}) \right) \right\|\\
% &\leq 2\beta\nu_{t-1}G\leq 2G\,,
% \end{aligned}
% $$
% The inequalities are because \dong{due to/because of} the stochastic gradient is $G$-Lipschitz and the coefficients $\beta \nu_{t-1}$ and $\alpha \nu_{t-1}$  \dong{which coefficient? $\beta$? } are both less than 1.
where the first inequality comes from Assumption \ref{asmp: G-lispchitz} and $\alpha \nu_{t-1}\leq 1$.
%\dong{due to/because of} 
% the stochastic gradient is $G$-Lipschitz and the coefficients $\beta \nu_{t-1}$ and $\alpha \nu_{t-1}$  
%\dong{which coefficient? $\beta$? } 
% are both less than 1.

\noindent \textbf{Privacy guarantee of $\mathcal{G}_t^p$.} 
% By Lemma \ref{lemma: RDP subsampling transformation}, if the Gaussian noises $u_t^p$ and $u_t^d$ have the following variance
By Lemma \ref{lemma: RDP subsampling transformation}, if the Gaussian noise $u_t^p$ has the following variance
\begin{align}\label{eq:sigma}
    \sigma^2_t = \frac{14G^2T\alpha'}{n^2  \left(\epsilon - \frac{\log  (1/\delta) }{\alpha'-1}\right) }\,,
\end{align}
where $\sigma'^2=\frac{\sigma_t^2}{4G^2}\geq 0.7$, $\alpha'=\frac{\log  (1/\delta) }{  (1-\beta') \epsilon}+1\leq 2\sigma^2\log  (\frac{n}{\alpha'  (1+\sigma'^2) }) /3+1$ and $\beta'\in   (0,1) $, 
then the mechanism $\mathcal{G}_t^p$ will satisfy $  \left(\alpha', \frac{14\alpha'G^2}{n^2\sigma_t^2}\right) $-RDP.
% $$
%   (\alpha', \frac{14\alpha'G^2}{n^2\sigma_t^2}) \text{-RDP}
% $$
% Obviously, the mechanisms \zhao{which one} in Case 1 and Case 2 are able to satisfy the same RDP under the same Gaussian noise. \zhao{to be fixed}

\noindent\textbf{Privacy guarantee of $\mathcal{M}^p_t$.} 
By the definition of $\mathcal{M}^p_t$ in Eq.  (\ref{eq: random mechanism of p_t}), $\mathcal{M}^p_t$ is composed of several Gaussian mechanisms, \textit{i.e.}, $\mathcal{M}_t^p=  (\mathcal{G}^p_0,...,\mathcal{G}^p_t) $. 
Then Lemma \ref{lemma: composition of RDP} implies that $\mathcal{M}_t^p$ satisfies $  (\alpha', \sum_{i=0}^t\frac{14\alpha'G^2}{n^2\sigma_t^2}) $-RDP. 
Thus the output on the primal side  satisfies $  \left(\alpha', \sum_{i=0}^T\frac{14\alpha'G^2}{n^2\sigma_t^2}\right) $-RDP. Finally, by using Lemma \ref{lemma: RDP_to_DP}, we transform RDP to DP and thus  the output satisfies 
\begin{align*}
    \left  (\sum_{i=0}^T\left  (\frac{14\alpha'G^2}{n^2\sigma_i^2}\right) +\frac{log  (1/\delta) }{\alpha'-1}, \delta\right) \text{-DP}\,.
\end{align*}
Substituting the value of $\sigma_i$ in Eq.  (\ref{eq:sigma})  simplifies the above result to $  (\epsilon, \delta) $-DP.
The proof of privacy guarantee of the dual side is similar to that of the primal side and is omitted here.
\end{proof}

% ---------------------------------------------- Proof of thm 5.2 ----------------------------------------
% \subsection{Proof of Theorem \ref{theorem: privacy guarantee   (TDP) }}

\section{Proof of Theorem \ref{theorem: utility}}
\label{app:sec:pf_theorem_utility}
In this section, we provide the proof of Theorem \ref{theorem: utility} , which gives the utility of Algorithm \ref{alg: DPTD} with $  (\epsilon,\delta) $-DP under state-action-state. 
% To this end, we will first introduce several lemmas to support our main proof.
To this end, we first introduce the following lemmas.

\begin{lemma}[\cite {lin2020gradient}]
\label{lemma: lispchitz constant of omega}
Under Assumptions \ref{asmp:wlipschitz}, \ref{asmp: strongly concave}, the mapping $\omega^\ast  (\theta) =\argmax_{\omega\in\Omega}F  (\theta,\omega) $ is Lipschitz continuous, which is
\begin{align*}
    \left\|\omega^{\ast}  (\theta) -\omega^{\ast}\left  (\theta^{\prime}\right) \right\| \leq L_{\omega}\left\|\theta-\theta^{\prime}\right\|, \quad \forall \theta, \theta^{\prime} \in \Theta
\end{align*}
% $$
% \left\|\omega^{*}  (\theta) -\omega^{*}\left  (\theta^{\prime}\right) \right\| \leq L_{\omega}\left\|\theta-\theta^{\prime}\right\|, \quad \forall \theta, \theta^{\prime} \in \Theta
% $$
where the Lipschitz constant is $L_\omega = \frac{L_F}{\mu}$.
\end{lemma}
Lemma \ref{lemma: lispchitz constant of omega} shows that $\omega^\ast  (\theta) $ also satisfies $L_\omega$-Lipschitz if we view $\omega^*  (\theta) $ as a mapping from the set $\Theta$ to the set $\Omega$.

\begin{lemma}[\cite {qiu2020single}]
\label{lemma: first in main proof}
Under Assumptions \ref{asmp:wlipschitz}, \ref{asmp: feasible convex set}, \ref{asmp: strongly concave}, letting $0<\kappa \nu_t\leq \mu/  (16L^2_F) $ and $\nu_t\leq1$, with the updating rules shown in Algorithm \ref{alg: DPTD}, we have
\begin{align*}
    J\left  (\theta_{t+1}\right) -J\left  (\theta_{t}\right)  
\leq -\frac{3 \nu_{t}}{4\kappa}\left\|\widetilde{\theta}_{t+1}-\theta_{t}\right\|^{2} +2 L_{F}^{2} \kappa \nu_{t}\left\|\omega_{t}-\omega\left  (\theta_{t}\right) ^{*}\right\|^{2}+4 \kappa \nu_{t}\left\|\nabla_{\theta} F\left  (\theta_{t}, \omega_{t}\right) -p_{t}\right\|^{2}\,, 
\end{align*}
% $$
% \begin{aligned}
% &J\left  (\theta_{t+1}\right) -J\left  (\theta_{t}\right)  \\
% \leq& -\frac{3 \nu_{t}}{4\kappa}\left\|\widetilde{\theta}_{t+1}-\theta_{t}\right\|^{2} +2 L_{F}^{2} \kappa \nu_{t}\left\|\omega_{t}-\omega\left  (\theta_{t}\right) ^{*}\right\|^{2}+4 \kappa \nu_{t}\left\|\nabla_{\theta} F\left  (\theta_{t}, \omega_{t}\right) -p_{t}\right\|^{2}\,, 
% \end{aligned}
% $$
where $J  (\theta) =\max _{\omega \in \Omega} F  (\theta, \omega) $ and $\omega^{*}  (\theta) :=\operatorname{argmax}_{\omega \in \Omega} F  (\theta, \omega) $.
\end{lemma}

\begin{lemma}[\cite {qiu2020single}]
\label{lemma: second in main proof}
Under Assumptions \ref{asmp:wlipschitz}, \ref{asmp: feasible convex set}, \ref{asmp: strongly concave}, letting $0<\nu_t\leq 1/8$ and $0<\eta\leq  (4L_F) ^{-1}$, with the updating rules shown in Algorithm \ref{alg: DPTD}, we have
\begin{align*}
    \left\|\omega_{t+1}-\omega^{*}\left  (\theta_{t}\right) \right\|^{2} \leq \left  (1-\frac{\nu_{t} \eta \mu}{2}\right) \left\|\omega_{t}-\omega^{*}\left  (\theta_{t}\right) \right\|^{2}-\frac{3 \nu_{t}}{4}\left\|\widetilde{\omega}_{t+1}-\omega_{t}\right\|^{2}+\frac{4 \eta \nu_{t}}{\mu}\left\|\nabla_{\omega} F\left  (\theta_{t}, \omega_{t}\right) -d_{t}\right\|^{2}\,,
\end{align*}
where $\omega^{*}\left  (\theta_{t}\right) =\operatorname{argmax}_{\omega \in \Omega} F\left  (\theta_{t}, \omega\right) $.
\end{lemma}

\begin{lemma}[\cite {qiu2020single}]
\label{lemma: third in main proof}
Under Assumptions \ref{asmp:wlipschitz}, \ref{asmp: feasible convex set}, \ref{asmp: strongly concave}, letting $0<\nu_t\leq 1/8$ and $0<\eta\leq  (4L_F) ^{-1}$, with the updating rules shown in Algorithm \ref{alg: DPTD}, we have
\begin{align*}
    \left\|\omega_{t+1}-\omega^{*}\left  (\theta_{t+1}\right) \right\|^{2} 
\leq& \left  (1-\frac{\mu \eta \nu_{t}}{4}\right) \left\|\omega_{t}-\omega^{*}\left  (\theta_{t}\right) \right\|^{2}-\frac{3 \nu_{t}}{4}\left\|\widetilde{\omega}_{t+1}-\omega_{t}\right\|^{2} +\frac{75 \eta \nu_{t}}{16 \mu}\left\|d_{t}-\nabla_{\omega} F\left  (\theta_{t}, \omega_{t}\right) \right\|^{2}\\
&+\frac{75 L_{\omega}^{2} \nu_{t}}{16 \mu \eta}\left\|\widetilde{\theta}_{t+1}-\theta_{t}\right\|^{2}\,,
\end{align*}
% $$
% \begin{aligned}
% &\left\|\omega_{t+1}-\omega^{*}\left  (\theta_{t+1}\right) \right\|^{2} \\
% \leq &\left  (1-\frac{\mu \eta \nu_{t}}{4}\right) \left\|\omega_{t}-\omega^{*}\left  (\theta_{t}\right) \right\|^{2}-\frac{3 \nu_{t}}{4}\left\|\widetilde{\omega}_{t+1}-\omega_{t}\right\|^{2} \\
% &+\frac{75 \eta \nu_{t}}{16 \mu}\left\|d_{t}-\nabla_{\omega} F\left  (\theta_{t}, \omega_{t}\right) \right\|^{2}+\frac{75 L_{\omega}^{2} \nu_{t}}{16 \mu \eta}\left\|\widetilde{\theta}_{t+1}-\theta_{t}\right\|^{2}\,,
% \end{aligned}
% $$
where $\omega^{*}\left  (\theta_{t}\right) =\operatorname{argmax}_{\omega \in \Omega} F\left  (\theta_{t}, \omega\right) $ and $\omega^{*}\left  (\theta_{t+1}\right) =\operatorname{argmax}_{\omega \in \Omega} F\left  (\theta_{t+1}, \omega\right) $.
\end{lemma}

\begin{lemma}[Bounded variance]
\label{lemma: bounded variance}
Under Assumption \ref{asmp: G-lispchitz}, the variance of the stochastic gradient $\nabla f (\theta, \omega; \xi) =\left (\nabla_{\theta} f (\theta, \omega;\xi) , \nabla_{\omega} f (\theta, \omega;\xi) \right) $ is bounded as $\mathbb{E}_{\xi \sim \Xi }\|\nabla f (\theta, \omega;\xi) -\nabla F (\theta, \omega) \|^{2} \leq \sigma^{2}$, where $\sigma^2=2G^2$.
\end{lemma}
Lemma \ref{lemma: bounded variance} shows that the stochastic gradient $\nabla f (\theta,\omega;\xi) $ is bounded by a constant $\sigma^2$, related to the property of function $f$ and $F$. 
% The detailed proof of Lemma \ref{lemma: bounded variance} is behind the main proof. 
% Armed with Lemma \ref{lemma: bounded variance},  the following lemma gives two inequalities and the detailed proof is behind the main proof.
Armed with Lemma \ref{lemma: bounded variance},  the following lemma further upper bounds the variances of the gradient estimators on the primal side and the dual side and  the detailed proof is deferred to Appendix \ref{sec:app:pf_tech_lem}.
\begin{lemma}[With bounded  variance]
\label{lemma: forth in main proof}
Under Assumptions \ref{asmp:wlipschitz}, \ref{asmp: feasible convex set}, \ref{asmp: strongly concave}, letting $0<\nu_t\leq   (8\alpha) ^{-1}$ and $0<\eta\leq  (4L_F) ^{-1}$, with the updating rules shown in Algorithm \ref{alg: DPTD}, we have
\begin{align*}
     \mathbb{E}\left\|\nabla_{\theta} F\left  (\theta_{t+1}, \omega_{t+1}\right) -p_{t+1}\right\|^{2}  \leq& \left  (1-\alpha \nu_{t}\right)  \mathbb{E}\left\|\nabla_{\theta} F\left  (\theta_{t}, \omega_{t}\right) -p_{t}\right\|^{2} +\frac{9 \nu_{t} L_{F}^{2}}{8 \alpha} \mathbb{E}\left  (\left\|\widetilde{\theta}_{t+1}-\theta_{t}\right\|^{2}+\left\|\widetilde{\omega}_{t+1}-\omega_{t}\right\|^{2}\right) \\
&+\alpha^{2} \nu_{t}^{2} \sigma^{2} + d\sigma_{t+1}^2\,,
\end{align*}
and
\begin{align*}
    \mathbb{E}\left\|\nabla_{\omega} F\left  (\theta_{t+1}, \omega_{t+1}\right) -d_{t+1}\right\|^{2} \leq &\left  (1-\alpha \nu_{t}\right)  \mathbb{E}\left\|\nabla_{\omega} F\left  (\theta_{t}, \omega_{t}\right) -d_{t}\right\|^{2} +\frac{9 \nu_{t} L_{F}^{2}}{8 \alpha} \mathbb{E}\left  (\left\|\widetilde{\theta}_{t+1}-\theta_{t}\right\|^{2}+\left\|\widetilde{\omega}_{t+1}-\omega_{t}\right\|^{2}\right) \\
&+\alpha^{2} \nu_{t}^{2} \sigma^{2} + d\sigma_{t+1}^2\,,
\end{align*}
where $\sigma^2=2G^2$.
% $$
% \begin{aligned}
% & \mathbb{E}\left\|\nabla_{\theta} F\left  (\theta_{t+1}, \omega_{t+1}\right) -p_{t+1}\right\|^{2} \\ \leq& \left  (1-\alpha \nu_{t}\right)  \mathbb{E}\left\|\nabla_{\theta} F\left  (\theta_{t}, \omega_{t}\right) -p_{t}\right\|^{2} \\
% &+\frac{9 \nu_{t} L_{F}^{2}}{8 \alpha} \mathbb{E}\left  (\left\|\widetilde{\theta}_{t+1}-\theta_{t}\right\|^{2}+\left\|\widetilde{\omega}_{t+1}-\omega_{t}\right\|^{2}\right) \\
% &+\alpha^{2} \nu_{t}^{2} \sigma^{2} + d\sigma_{t+1}^2\,,
% \end{aligned}
% $$
% $$
% \begin{aligned}
% &\mathbb{E}\left\|\nabla_{\omega} F\left  (\theta_{t+1}, \omega_{t+1}\right) -d_{t+1}\right\|^{2} \\ \leq &\left  (1-\alpha \nu_{t}\right)  \mathbb{E}\left\|\nabla_{\omega} F\left  (\theta_{t}, \omega_{t}\right) -d_{t}\right\|^{2} \\
% &+\frac{9 \nu_{t} L_{F}^{2}}{8 \alpha} \mathbb{E}\left  (\left\|\widetilde{\theta}_{t+1}-\theta_{t}\right\|^{2}+\left\|\widetilde{\omega}_{t+1}-\omega_{t}\right\|^{2}\right) \\
% &+\alpha^{2} \nu_{t}^{2} \sigma^{2} + d\sigma_{t+1}^2\,,
% \end{aligned}
% $$
\end{lemma}

\begin{proof}[Proof of Theorem \ref{theorem: utility}]
% \noindent \textbf{Main proof.} 
% We assume \zhao{it is not assumption, it is part of our algo, we need to justify our choice on step size} that the step size is of the form
% $
% \nu_t = \frac{a}{  (t+b) ^\frac{1}{2}}
% $, where $a=\frac{1}{16}$. 
Recall the step size is chosen as $
\nu_t = \frac{a}{  (t+b) ^\frac{1}{2}}
$ with $a=\frac{1}{16}$ in Theorem \ref{theorem: utility}. 
% By our assumption on $L_F$-smoothness and $\mu$-strongly concavity, we have a basic fact \dong{a well-known fact?} \zhao{cite} that $L_F\geq \mu$. Then we interpret \dong{I don't think we are interpreting, `` by the parameter setting in Theorem...'' }the parameter setting in Theorem \ref{theorem: utility} as
% By our assumption on $L_F$-smoothness and $\mu$-strongly concavity, we have a basic fact \dong{a well-known fact?} \zhao{cite} that $L_F\geq \mu$. 
By Assumpion \ref{asmp:wlipschitz} and Assumption \ref{asmp: strongly concave}, it is clear that $L_F\geq \mu$. 
% \dong{I don't think we are interpreting, `` by the parameter setting in Theorem...'' }
The parameter $\eta$ and $\nu_{t}$ in Theorem \ref{theorem: utility} could be further bounded as
\begin{align*}
    \eta \leq \frac{\mu}{4L_F^2}\leq \frac{1}{4L_F}
\end{align*}
and
\begin{align*}
    \nu_{t} \leq \frac{a}{b^{1 / 2}} \leq \min \left\{\frac{1}{27}, \quad \frac{\mu}{16 \kappa L_{F}^{2}}\right\}\,.
\end{align*}
% $$
% \eta \leq \frac{\mu}{4L_F^2}\leq \frac{1}{4L_F}
% $$
% $$
% \nu_{t} \leq \frac{a}{b^{1 / 2}} \leq \min \left\{\frac{1}{27}, \quad \frac{\mu}{16 \kappa L_{F}^{2}}\right\}
% $$
% Thus, with such parameter settings, we can directly apply Lemmas \ref{lemma: first in main proof}, \ref{lemma: second in main proof} and \ref{lemma: third in main proof} in the following proof. 
Thus, with such parameter settings, we are able to apply Lemmas \ref{lemma: first in main proof}, \ref{lemma: second in main proof} and \ref{lemma: third in main proof} in the following proof. 
By Lemma \ref{lemma: first in main proof}, we have
\begin{align*}
    J\left  (\theta_{t+1}\right) -J\left  (\theta_{t}\right) 
\leq -\frac{3 \nu_{t}}{4 \kappa}\left\|\widetilde{\theta}_{t+1}-\theta_{t}\right\|^{2}
+2 L_{F}^{2} \kappa \nu_{t}\left\|\omega_{t}-\omega\left  (\theta_{t}\right) ^{*}\right\|^{2}+4 \kappa \nu_{t}\left\|\nabla_{\theta} F\left  (\theta_{t}, \omega_{t}\right) -p_{t}\right\|^{2}\,.
\end{align*}
% $$
% \begin{aligned}
% &J\left  (\theta_{t+1}\right) -J\left  (\theta_{t}\right) \\
% \leq& -\frac{3 \nu_{t}}{4 \kappa}\left\|\widetilde{\theta}_{t+1}-\theta_{t}\right\|^{2}
% +2 L_{F}^{2} \kappa \nu_{t}\left\|\omega_{t}-\omega\left  (\theta_{t}\right) ^{*}\right\|^{2}+4 \kappa \nu_{t}\left\|\nabla_{\theta} F\left  (\theta_{t}, \omega_{t}\right) -p_{t}\right\|^{2}\,.
% \end{aligned}
% $$
Taking expectation on both sides shows that
\begin{align}\label{ineq: main proof 1}
   \mathbb{E}\left[J\left  (\theta_{t+1}\right) -J\left  (\theta_{t}\right) \right] 
\leq -\frac{3 \nu_{t}}{4 \kappa} \mathbb{E}\left\|\tilde{\theta}_{t+1}-\theta_{t}\right\|^{2}
+2 L_{F}^{2} \kappa \nu_{t} \mathbb{E}\left\|\omega_{t}-\omega^{*}\left  (\theta_{t}\right) \right\|^{2} +4 \nu_{t} \kappa \mathbb{E}\left\|\nabla_{\theta} F\left  (\theta_{t}, \omega_{t}\right) -p_{t}\right\|^{2}\,.
\end{align}
In the above inequality, the LHS will be a telescoping sum if we sum over $t$ from $t=0$ to $T-1$. And then we can move the first term on the RHS to the LHS, which will give us an upper bound for the summation of $\mathbb{E}\left\|\tilde{\theta}_{t+1}-\theta_{t}\right\|^{2}$. Thus, to get the final bound, we need to get the upper bound of another two terms on the RHS of Eq. (\ref{ineq: main proof 1}), \textit{i.e.}, $\mathbb{E}\left\|\omega_{t}-\omega^{*}\left  (\theta_{t}\right) \right\|^{2}$ and $ \mathbb{E}\left\|\nabla_{\theta} F\left  (\theta_{t}, \omega_{t}-p_{t}\right) \right\|^{2}$.
% Then, by Lemma \ref{lemma: third in main proof} and taking expectation on both sides, we can establish an inequality whose right-hand side gives a contraction of $\mathbb{E}\left\|\omega_{t}-\omega^{*}\left  (\theta_{t}\right) \right\|^{2}$. The inequality is
% $\mathbb{E}\left\|\omega_{t}-\omega^{*}\left  (\theta_{t}\right) \right\|^{2}$ could be upper bounded by Lemma \ref{lemma: third in main proof} as
Furthermore, Lemma \ref{lemma: third in main proof} shows that
\begin{align*}
  \mathbb{E}\left\|\omega_{t+1}-\omega^{*}\left  (\theta_{t+1}\right) \right\|^{2}
\leq& \left  (1-\frac{\mu \eta \nu_{t}}{4}\right)  \mathbb{E}\left\|\omega_{t}-\omega^{*}\left  (\theta_{t}\right) \right\|^{2}
-\frac{3 \nu_{t}}{4} \mathbb{E}\left\|\widetilde{\omega}_{t+1}-\omega_{t}\right\|^{2}\\
&+\frac{75 \eta \nu_{t}}{16 \mu} \mathbb{E}\left\|d_{t}-\nabla F_\omega\left  (\theta_{t}, \omega_{t}\right) \right\|^{2}+\frac{75 L_{\omega}^{2} \nu_{t}}{16 \mu \eta} \mathbb{E}\left\|\widetilde{\theta}_{t+1}-\theta_{t}\right\|^{2}\,.  
\end{align*}
% $$
% \begin{aligned}
% &\mathbb{E}\left\|\omega_{t+1}-\omega^{*}\left  (\theta_{t+1}\right) \right\|^{2}\\
% \leq &\left  (1-\frac{\mu \eta \nu_{t}}{4}\right)  \mathbb{E}\left\|\omega_{t}-\omega^{*}\left  (\theta_{t}\right) \right\|^{2}
% -\frac{3 \nu_{t}}{4} \mathbb{E}\left\|\widetilde{\omega}_{t+1}-\omega_{t}\right\|^{2}\\
% &+\frac{75 \eta \nu_{t}}{16 \mu} \mathbb{E}\left\|d_{t}-\nabla F_\omega\left  (\theta_{t}, \omega_{t}\right) \right\|^{2}+\frac{75 L_{\omega}^{2} \nu_{t}}{16 \mu \eta} \mathbb{E}\left\|\widetilde{\theta}_{t+1}-\theta_{t}\right\|^{2}\,.
% \end{aligned}
% $$
Multiplying both sides of the above inequality by $10L_F^2\kappa/  (\mu\eta) $ leads to
\begin{align*}
    \frac{10 L_{F}^{2} \kappa}{\mu \eta} \mathbb{E}\left\|\omega_{t+1}-\omega^{*}\left  (\theta_{t+1}\right) \right\|^{2} 
\leq & \frac{10 L_{F}^{2} \kappa}{\mu \eta}\left  (1-\frac{\mu \eta \nu_{t}}{4}\right)  \mathbb{E}\left\|\omega_{t}-\omega^{*}\left  (\theta_{t}\right) \right\|^{2}-\frac{15 L_{F}^{2} \kappa \nu_{t}}{2 \mu \eta} \mathbb{E}\left\|\widetilde{\omega}_{t+1}-\omega_{t}\right\|^{2} \\
&+\frac{375 L_{F}^{2} \kappa \nu_{t}}{8 \mu^{2}} \mathbb{E}\left\|d_{t}-\nabla F_\omega\left  (\theta_{t}, \omega_{t}\right) \right\|^{2}+\frac{375 L_{F}^{2} L_{\omega}^{2} \kappa \nu_{t}}{8 \mu^{2} \eta^{2}} \mathbb{E}\left\|\tilde{\theta}_{t+1}-\theta_{t}\right\|^{2}\,.
\end{align*}
% $$
% \begin{aligned}
% \frac{10 L_{F}^{2} \kappa}{\mu \eta} &\mathbb{E}\left\|\omega_{t+1}-\omega^{*}\left  (\theta_{t+1}\right) \right\|^{2} \\
% \leq & \frac{10 L_{F}^{2} \kappa}{\mu \eta}\left  (1-\frac{\mu \eta \nu_{t}}{4}\right)  \mathbb{E}\left\|\omega_{t}-\omega^{*}\left  (\theta_{t}\right) \right\|^{2}-\frac{15 L_{F}^{2} \kappa \nu_{t}}{2 \mu \eta} \mathbb{E}\left\|\widetilde{\omega}_{t+1}-\omega_{t}\right\|^{2} \\
% &+\frac{375 L_{F}^{2} \kappa \nu_{t}}{8 \mu^{2}} \mathbb{E}\left\|d_{t}-\nabla F_\omega\left  (\theta_{t}, \omega_{t}\right) \right\|^{2}+\frac{375 L_{F}^{2} L_{\omega}^{2} \kappa \nu_{t}}{8 \mu^{2} \eta^{2}} \mathbb{E}\left\|\tilde{\theta}_{t+1}-\theta_{t}\right\|^{2}\,.
% \end{aligned}
% $$
% To get a subtraction of $\left\| \omega_t-\omega^*  (\theta_t) \right\|$ from $t+1$ to $t$, we rearrange the terms and  get 
Rearranging the terms shows that 
\begin{align}\label{ineq: main proof 2}
    \frac{10 L_{F}^{2} \kappa}{\mu \eta} &\left  (\mathbb{E}\left\|\omega_{t+1}-\omega^{*}\left  (\theta_{t+1}\right) \right\|^{2}-\mathbb{E}\left\|\omega_{t}-\omega^{*}\left  (\theta_{t}\right) \right\|^{2}\right) \notag \\
\leq &-\frac{5 L_{F}^{2} \kappa \nu_{t}}{2} \mathbb{E}\left\|\omega_{t}-\omega^{*}\left  (\theta_{t}\right) \right\|^{2}-\frac{15 L_{F}^{2} \kappa \nu_{t}}{2 \mu \eta} \mathbb{E}\left\|\widetilde{\omega}_{t+1}-\omega_{t}\right\|^{2}\notag \\
&+\frac{375 L_{F}^{2} \kappa \nu_{t}}{8 \mu^{2}} \mathbb{E}\left\|d_{t}-\nabla F_{\omega}\left  (\theta_{t}, \omega_{t}\right) \right\|^{2}+\frac{375 L_{F}^{2} L_{\omega}^{2} \kappa \nu_{t}}{8 \mu^{2} \eta^{2}} \mathbb{E}\left\|\widetilde{\theta}_{t+1}-\theta_{t}\right\|^{2}\,.
\end{align}
% \begin{equation}
% \label{ineq: main proof 2}
% \begin{aligned}
% \frac{10 L_{F}^{2} \kappa}{\mu \eta} &\left  (\mathbb{E}\left\|\omega_{t+1}-\omega^{*}\left  (\theta_{t+1}\right) \right\|^{2}-\mathbb{E}\left\|\omega_{t}-\omega^{*}\left  (\theta_{t}\right) \right\|^{2}\right)  \\
% \leq &-\frac{5 L_{F}^{2} \kappa \nu_{t}}{2} \mathbb{E}\left\|\omega_{t}-\omega^{*}\left  (\theta_{t}\right) \right\|^{2}-\frac{15 L_{F}^{2} \kappa \nu_{t}}{2 \mu \eta} \mathbb{E}\left\|\widetilde{\omega}_{t+1}-\omega_{t}\right\|^{2} \\
% &+\frac{375 L_{F}^{2} \kappa \nu_{t}}{8 \mu^{2}} \mathbb{E}\left\|d_{t}-\nabla F_{\omega}\left  (\theta_{t}, \omega_{t}\right) \right\|^{2}+\frac{375 L_{F}^{2} L_{\omega}^{2} \kappa \nu_{t}}{8 \mu^{2} \eta^{2}} \mathbb{E}\left\|\widetilde{\theta}_{t+1}-\theta_{t}\right\|^{2}
% \end{aligned}
% \end{equation}
Then we define
\begin{align*}
    P_{t}:=J\left  (\theta_{t}\right) -J^{*}+\frac{10 L_{F}^{2} \kappa}{\mu \eta}\left\|\omega_{t}-\omega^{*}\left  (\theta_{t}\right) \right\|^{2}, \quad \forall t \geq 0
\end{align*}
% $$
% P_{t}:=J\left  (\theta_{t}\right) -J^{*}+\frac{10 L_{F}^{2} \kappa}{\mu \eta}\left\|\omega_{t}-\omega^{*}\left  (\theta_{t}\right) \right\|^{2}, \quad \forall t \geq 0
% $$
where $J^*>-\infty$ is the minimal value of $J$. Thus we have $J  (\theta) -J^*>0, \forall \theta \in \Theta$.
% Summing up Eq.  (\ref{ineq: main proof 1})  and Eq.  (\ref{ineq: main proof 2}), we have
Taking both Eq.  (\ref{ineq: main proof 1})  and Eq.  (\ref{ineq: main proof 2}) into consideration, we have
\begin{align*}
    \mathbb{E}\left[P_{t+1}-P_{t}\right] \leq &-\left  (\frac{3 \nu_{t}}{4 \kappa}-\frac{375 L_{F}^{2} L_{\omega}^{2} \kappa \nu_{t}}{8 \mu^{2} \eta^{2}}\right)  \mathbb{E}\left\|\tilde{\theta}_{t+1}-\theta_{t}\right\|^{2}-\frac{15 L_{F}^{2} \kappa \nu_{t}}{2 \mu \eta} \mathbb{E}\left\|\widetilde{\omega}_{t+1}-\omega_{t}\right\|^{2} \\
&+\frac{375 L_{F}^{2} \kappa \nu_{t}}{8 \mu^{2}} \mathbb{E}\left\|d_{t}-\nabla_{\omega} F\left  (\theta_{t}, \omega_{t}\right) \right\|^{2}+4 \nu_{t} \kappa \mathbb{E}\left\|p_{t}-\nabla_{\theta} F\left  (\theta_{t}, \omega_{t}\right) \right\|^{2}-\frac{L_{F}^{2} \kappa \nu_{t}}{2} \mathbb{E}\left\|\omega_{t}-\omega^{*}\left  (\theta_{t}\right) \right\|^{2}\,.
\end{align*}
% $$
% \begin{aligned}
% \mathbb{E}\left[P_{t+1}-P_{t}\right] \leq &-\left  (\frac{3 \nu_{t}}{4 \kappa}-\frac{375 L_{F}^{2} L_{\omega}^{2} \kappa \nu_{t}}{8 \mu^{2} \eta^{2}}\right)  \mathbb{E}\left\|\tilde{\theta}_{t+1}-\theta_{t}\right\|^{2}\\
% &-\frac{15 L_{F}^{2} \kappa \nu_{t}}{2 \mu \eta} \mathbb{E}\left\|\widetilde{\omega}_{t+1}-\omega_{t}\right\|^{2} \\
% &+\frac{375 L_{F}^{2} \kappa \nu_{t}}{8 \mu^{2}} \mathbb{E}\left\|d_{t}-\nabla_{\omega} F\left  (\theta_{t}, \omega_{t}\right) \right\|^{2}\\
% &+4 \nu_{t} \kappa \mathbb{E}\left\|p_{t}-\nabla_{\theta} F\left  (\theta_{t}, \omega_{t}\right) \right\|^{2}\\
% &-\frac{L_{F}^{2} \kappa \nu_{t}}{2} \mathbb{E}\left\|\omega_{t}-\omega^{*}\left  (\theta_{t}\right) \right\|^{2}
% \end{aligned}\,.
% $$
We can simplify the coefficient $-\left  (\frac{3 \nu_{t}}{4 \kappa}-\frac{375 L_{F}^{2} L_{\omega}^{2} \kappa \nu_{t}}{8 \mu^{2} \eta^{2}}\right) $ in the above inequality. First,  by the parameter setting in Theorem \ref{theorem: utility}, we have $0<\kappa \leq \eta \mu^{2} /\left  (9 L_{F}^{2}\right) $, which gives us $\eta \geq 9 L_{F}^{2} \kappa / \mu^{2}$ and further $\eta^{2} \geq 81 L_{F}^{4} \kappa^{2} / \mu^{4}$. Second, by Lemma \ref{lemma: lispchitz constant of omega}, we have $L_\omega = L_F/\mu$. Thus 
%   \zhao{detail \& where does the LHS come from?}
\begin{align*}
      \frac{375 L_{F}^{2} L_{\omega}^{2} \kappa \nu_{t}}{8 \mu^{2} \eta^{2}} = \frac{375 L_{F}^{4} \kappa \nu_{t}}{8 \mu^{4} \eta^{2}} \leq \frac{125\nu_t}{216\kappa}\,,
\end{align*}
%   $$
%   \frac{375 L_{F}^{2} L_{\omega}^{2} \kappa \nu_{t}}{8 \mu^{2} \eta^{2}} = \frac{375 L_{F}^{4} \kappa \nu_{t}}{8 \mu^{4} \eta^{2}} \leq \frac{125\nu_t}{216\kappa}\,,
%   $$
  and the coefficient is bounded by
  \begin{align*}
   -\left  (\frac{3 \nu_{t}}{4 \kappa}-\frac{375 L_{F}^{2} L_{\omega}^{2} \kappa \nu_{t}}{8 \mu^{2} \eta^{2}}\right) \leq -\frac{37\nu_t}{216\kappa}\leq -\frac{\nu_t}{8\kappa}\,,   
  \end{align*}
%  $$
%  -\left  (\frac{3 \nu_{t}}{4 \kappa}-\frac{375 L_{F}^{2} L_{\omega}^{2} \kappa \nu_{t}}{8 \mu^{2} \eta^{2}}\right) \leq -\frac{37\nu_t}{216\kappa}\leq -\frac{\nu_t}{8\kappa}\,,
%  $$
which implies that
\begin{align}\label{ineq: main proof 3}
    \mathbb{E}\left[P_{t+1}-P_{t}\right] \leq &-\frac{\nu_{t}}{8 \kappa} \mathbb{E}\left\|\widetilde{\theta}_{t+1}-\theta_{t}\right\|^{2}-\frac{15 L_{F}^{2} \kappa \nu_{t}}{2 \mu \eta} \mathbb{E}\left\|\widetilde{\omega}_{t+1}-\omega_{t}\right\|^{2}-\frac{L_{F}^{2} \kappa \nu_{t}}{2} \mathbb{E}\left\|\omega_{t}-\omega^{*}\left  (\theta_{t}\right) \right\|^{2}\notag \\
&+\frac{375 L_{F}^{2} \kappa \nu_{t}}{8 \mu^{2}} \mathbb{E}\left\|d_{t}-\nabla_{\omega} F\left  (\theta_{t}, \omega_{t}\right) \right\|^{2}+4 \nu_{t} \kappa \mathbb{E}\left\|p_{t}-\nabla_{\theta} F\left  (\theta_{t}, \omega_{t}\right) \right\|^{2}\,.
\end{align}
% \begin{equation}
% \label{ineq: main proof 3}
% \begin{aligned}
% \mathbb{E}\left[P_{t+1}-P_{t}\right] \leq &-\frac{\nu_{t}}{8 \kappa} \mathbb{E}\left\|\widetilde{\theta}_{t+1}-\theta_{t}\right\|^{2}\\
% &-\frac{15 L_{F}^{2} \kappa \nu_{t}}{2 \mu \eta} \mathbb{E}\left\|\widetilde{\omega}_{t+1}-\omega_{t}\right\|^{2}\\
% &-\frac{L_{F}^{2} \kappa \nu_{t}}{2} \mathbb{E}\left\|\omega_{t}-\omega^{*}\left  (\theta_{t}\right) \right\|^{2} \\
% &+\frac{375 L_{F}^{2} \kappa \nu_{t}}{8 \mu^{2}} \mathbb{E}\left\|d_{t}-\nabla_{\omega} F\left  (\theta_{t}, \omega_{t}\right) \right\|^{2}\\
% &+4 \nu_{t} \kappa \mathbb{E}\left\|p_{t}-\nabla_{\theta} F\left  (\theta_{t}, \omega_{t}\right) \right\|^{2}\,.
% \end{aligned}
% \end{equation}
The LHS of Eq. (\ref{ineq: main proof 3}) will be a telescoping sum if we task summation from $t=0$ to $T-1$. And then we can move $\mathbb{E}\left\|\widetilde{\theta}_{t+1}-\theta_{t}\right\|^{2}$ and $\mathbb{E}\left\|\omega_{t}-\omega^{*}\left  (\theta_{t}\right) \right\|^{2}$ from the RHS to the LHS, which will  help us bound the two terms. 
Thus, we expect to upper bound
$\mathbb{E}\left\|d_{t}-\nabla_{\omega} F\left  (\theta_{t}, \omega_{t}\right) \right\|^{2}$ and $\mathbb{E}\left\|p_{t}-\nabla_{\theta} F\left  (\theta_{t}, \omega_{t}\right) \right\|^{2}$.
By Lemma \ref{lemma: forth in main proof}, we have
\begin{align}\label{ineq: main proof 4}
\mathbb{E}\left\|\nabla_{\theta} F\left  (\theta_{t+1}, \omega_{t+1}\right) -p_{t+1}\right\|^{2}  \leq& \left  (1-\alpha \nu_{t}\right)  \mathbb{E}\left\|\nabla_{\theta} F\left  (\theta_{t}, \omega_{t}\right) -p_{t}\right\|^{2} +\frac{9 \nu_{t} L_{F}^{2}}{8 \alpha} \mathbb{E}\left  (\left\|\widetilde{\theta}_{t+1}-\theta_{t}\right\|^{2}+\left\|\widetilde{\omega}_{t+1}-\omega_{t}\right\|^{2}\right) \notag\\
&+\alpha^{2} \nu_{t}^{2} \sigma^{2} + d\sigma_{t+1}^2\,,
\end{align}
% \begin{equation}
% \label{ineq: main proof 4}
% \begin{aligned}
% & \mathbb{E}\left\|\nabla_{\theta} F\left  (\theta_{t+1}, \omega_{t+1}\right) -p_{t+1}\right\|^{2} \\ \leq& \left  (1-\alpha \nu_{t}\right)  \mathbb{E}\left\|\nabla_{\theta} F\left  (\theta_{t}, \omega_{t}\right) -p_{t}\right\|^{2} \\
% &+\frac{9 \nu_{t} L_{F}^{2}}{8 \alpha} \mathbb{E}\left  (\left\|\widetilde{\theta}_{t+1}-\theta_{t}\right\|^{2}+\left\|\widetilde{\omega}_{t+1}-\omega_{t}\right\|^{2}\right) \\
% &+\alpha^{2} \nu_{t}^{2} \sigma^{2} + d\sigma_{t+1}^2\,,
% \end{aligned}
% \end{equation}
and
\begin{align}\label{ineq: main proof 5}
\mathbb{E}\left\|\nabla_{\omega} F\left  (\theta_{t+1}, \omega_{t+1}\right) -d_{t+1}\right\|^{2}  \leq &\left  (1-\alpha \nu_{t}\right)  \mathbb{E}\left\|\nabla_{\omega} F\left  (\theta_{t}, \omega_{t}\right) -d_{t}\right\|^{2} +\frac{9 \nu_{t} L_{F}^{2}}{8 \alpha} \mathbb{E}\left  (\left\|\widetilde{\theta}_{t+1}-\theta_{t}\right\|^{2}+\left\|\widetilde{\omega}_{t+1}-\omega_{t}\right\|^{2}\right) \notag\\
&+\alpha^{2} \nu_{t}^{2} \sigma^{2} + d\sigma_{t+1}^2\,. 
\end{align}
% \begin{equation}
% \label{ineq: main proof 5}
% \begin{aligned}
% &\mathbb{E}\left\|\nabla_{\omega} F\left  (\theta_{t+1}, \omega_{t+1}\right) -d_{t+1}\right\|^{2} \\ \leq &\left  (1-\alpha \nu_{t}\right)  \mathbb{E}\left\|\nabla_{\omega} F\left  (\theta_{t}, \omega_{t}\right) -d_{t}\right\|^{2} \\
% &+\frac{9 \nu_{t} L_{F}^{2}}{8 \alpha} \mathbb{E}\left  (\left\|\widetilde{\theta}_{t+1}-\theta_{t}\right\|^{2}+\left\|\widetilde{\omega}_{t+1}-\omega_{t}\right\|^{2}\right) \\
% &+\alpha^{2} \nu_{t}^{2} \sigma^{2} + d\sigma_{t+1}^2\,.
% \end{aligned}
% \end{equation}
% By the parameter setting in Theorem \ref{theorem: utility}, we set $\alpha=3$, which will further help us simplify the coefficient and get a contraction.
% If we set $\alpha=3$ as suggested in Theorem \ref{theorem: utility}, we could further simplify the coefficient and get a contraction.

\noindent We define a Lyapunov function to enable a telescoping summation, which is for $\forall t \geq 0$, 
\begin{align*}
Q_{t}:=P_{t}+\frac{2 \kappa}{\mu \eta}\left\|\nabla_{\theta} F\left  (\theta_{t}, \omega_{t}\right) -p_{t}\right\|^{2}+\frac{2 \kappa}{\mu \eta}\left\|\nabla_{\omega} F\left  (\theta_{t}, \omega_{t}\right) -d_{t}\right\|^{2}\,.
\end{align*}
% $$
% Q_{t}:=P_{t}+\frac{2 \kappa}{\mu \eta}\left\|\nabla_{\theta} F\left  (\theta_{t}, \omega_{t}\right) -p_{t}\right\|^{2}+\frac{2 \kappa}{\mu \eta}\left\|\nabla_{\omega} F\left  (\theta_{t}, \omega_{t}\right) -d_{t}\right\|^{2}\,.
% $$
Multiplying both side of Eq. (\ref{ineq: main proof 4}) and Eq. (\ref{ineq: main proof 5}) by $2\kappa/  (\mu\eta) $ and combining with Eq. (\ref{ineq: main proof 3}), we have
\begin{align*}
\mathbb{E}[Q_{t+1}-Q_t]
\leq &-\left  (\frac{\nu_{t}}{8 \kappa}-\frac{3 \kappa \nu_{t} L_{F}^{2}}{2 \mu \eta}\right)  \mathbb{E}\left\|\widetilde{\theta}_{t+1}-\theta_{t}\right\|^{2}-\frac{L_{F}^{2} \kappa \nu_{t}}{2} \mathbb{E}\left\|\omega_{t}-\omega^{*}\left  (\theta_{t}\right) \right\|^{2}+\frac{4d\kappa\sigma^2_{t+1}}{\mu\eta}\\
&-\frac{6 L_{F}^{2} \kappa \nu_{t}}{\mu \eta} \mathbb{E}\left\|\widetilde{\omega}_{t+1}-\omega_{t}\right\|^{2}+\frac{36 \sigma^{2} \nu_{t}^{2} \kappa}{\mu \eta} -\left  (\frac{12 \nu_{t} \kappa}{\mu \eta}-4 \nu_{t} \kappa\right)  \mathbb{E}\left\|\nabla_{\theta} F\left  (\theta_{t}, \omega_{t}\right) -p_{t}\right\|^{2}\\
&-\left  (\frac{12 \nu_{t} \kappa}{\mu \eta}-\frac{375 L_{F}^{2} \kappa \nu_{t}}{8 \mu^{2}}\right)  \mathbb{E}\left\|\nabla_{\omega} F\left  (\theta_{t}, \omega_{t}\right) -d_{t}\right\|^{2}\,.
\end{align*}
% $$
% \begin{aligned}
% \mathbb{E}&[Q_{t+1}-Q_t]\\
% \leq &-\left  (\frac{\nu_{t}}{8 \kappa}-\frac{3 \kappa \nu_{t} L_{F}^{2}}{2 \mu \eta}\right)  \mathbb{E}\left\|\widetilde{\theta}_{t+1}-\theta_{t}\right\|^{2}\\
% &-\frac{L_{F}^{2} \kappa \nu_{t}}{2} \mathbb{E}\left\|\omega_{t}-\omega^{*}\left  (\theta_{t}\right) \right\|^{2}+\frac{4d\kappa\sigma^2_{t+1}}{\mu\eta}\\
% &-\frac{6 L_{F}^{2} \kappa \nu_{t}}{\mu \eta} \mathbb{E}\left\|\widetilde{\omega}_{t+1}-\omega_{t}\right\|^{2}+\frac{36 \sigma^{2} \nu_{t}^{2} \kappa}{\mu \eta} \\
% &-\left  (\frac{12 \nu_{t} \kappa}{\mu \eta}-4 \nu_{t} \kappa\right)  \mathbb{E}\left\|\nabla_{\theta} F\left  (\theta_{t}, \omega_{t}\right) -p_{t}\right\|^{2}\\
% &-\left  (\frac{12 \nu_{t} \kappa}{\mu \eta}-\frac{375 L_{F}^{2} \kappa \nu_{t}}{8 \mu^{2}}\right)  \mathbb{E}\left\|\nabla_{\omega} F\left  (\theta_{t}, \omega_{t}\right) -d_{t}\right\|^{2}
% \end{aligned}
% $$
% The coefficients in the above inequality are complex and we simplify them by utilizing the parameter setting, similar to what we have done before.

By the parameter setting in Theorem \ref{theorem: utility}, we have $0<\eta \leq \mu /\left  (4 L_{F}^{2}\right) $ and $0<\kappa \leq \eta \mu^{2} /\left  (9 L_{F}^{2}\right) $, or $\eta \geq 9 L_{F}^{2} \kappa / \mu^{2}$. 
Moreover, by our assumption on $L_F$-smoothness and $\mu$-strongly concavity, we have $L_F\geq \mu>0$. 
% \zhao{where does $L_F\geq\mu>0$ come from?   (it seems not in our assumptions?) }
We now simply the coefficients in the above display.
First, for $-\left  (\frac{\nu_{t}}{8 \kappa}-\frac{3 \kappa \nu_{t} L_{F}^{2}}{2 \mu \eta}\right) $, we have
%\zhao{how to cancel $\kappa$ in the numerator of the LHS in the first ineq?}
\begin{align*}
-\left  (\frac{\nu_{t}}{8 \kappa}-\frac{3 \kappa \nu_{t} L_{F}^{2}}{2 \mu \eta}\right) \leq -\frac{\nu_t}{8\kappa}+\frac{\nu_t\mu}{6}\,,
\end{align*}
% $$
% -\left  (\frac{\nu_{t}}{8 \kappa}-\frac{3 \kappa \nu_{t} L_{F}^{2}}{2 \mu \eta}\right) \leq -\frac{\nu_t}{8\kappa}+\frac{\nu_t\mu}{6}\,,
% $$
and we remove $\mu$ by the following inequality
\begin{align*}
    \mu\kappa\leq \frac{\eta\mu^3}{9L^2_F}\leq \frac{\eta\mu}{9} \leq \frac{\mu^2}{36L^2_F} \leq \frac{1}{36}\,.
\end{align*}
% $$
% \mu\kappa\leq \frac{\eta\mu^3}{9L^2_F}\leq \frac{\eta\mu}{9} \leq \frac{\mu^2}{36L^2_F} \leq \frac{1}{36}\,.
% $$
Thus we obtain
\begin{align*}
-\left  (\frac{\nu_{t}}{8 \kappa}-\frac{3 \kappa \nu_{t} L_{F}^{2}}{2 \mu \eta}\right) \leq -\frac{\nu_t}{8\kappa}+\frac{\nu_t}{216\kappa} \leq-\frac{\nu_{t}}{16 \kappa}\,.
\end{align*}
% $$
% -\left  (\frac{\nu_{t}}{8 \kappa}-\frac{3 \kappa \nu_{t} L_{F}^{2}}{2 \mu \eta}\right) \leq -\frac{\nu_t}{8\kappa}+\frac{\nu_t}{216\kappa} \leq-\frac{\nu_{t}}{16 \kappa}\,.
% $$
Second, for $-\left  (\frac{12 \nu_{t} \kappa}{\mu \eta}-4 \kappa \nu_{t}\right) $, 
% we need the following equality to help us
notice that
% \zhao{why $\mu \eta\leq 2?$}
\begin{align*}
\mu\eta\leq \frac{\mu^2}{4L_F^2}\leq \frac{1}{4}\,,
\end{align*}
% $$
% \mu\eta\leq \frac{\mu^2}{4L_F^2}\leq \frac{1}{4}\,,
% $$
and hence
\begin{align*}
-\left  (\frac{12 \nu_{t} \kappa}{\mu \eta}-4 \kappa \nu_{t}\right)  \leq \frac{-12 \nu_{t} \kappa}{\mu \eta}+\frac{\kappa\nu_t}{\mu\eta} \leq-\frac{4 \nu_{t} \kappa}{\mu \eta}\,.
\end{align*}
% $$
% -\left  (\frac{12 \nu_{t} \kappa}{\mu \eta}-4 \kappa \nu_{t}\right)  \leq \frac{-12 \nu_{t} \kappa}{\mu \eta}+\frac{\kappa\nu_t}{\mu\eta} \leq-\frac{4 \nu_{t} \kappa}{\mu \eta}\,.
% $$
Third, for $-\left  (\frac{12 \nu_{t} \kappa}{\mu \eta}-\frac{375 L_{F}^{2} \kappa \nu_{t}}{8 \mu^{2}}\right) $, we obtain the following simplified result 
\begin{align*}
-\left  (\frac{12 \nu_{t} \kappa}{\mu \eta}-\frac{375 L_{F}^{2} \kappa \nu_{t}}{8 \mu^{2}}\right) \leq -\frac{\nu_t\kappa}{\mu}\cdot \frac{9L^2_F}{8\mu}<0\,.
\end{align*}
% $$
% -\left  (\frac{12 \nu_{t} \kappa}{\mu \eta}-\frac{375 L_{F}^{2} \kappa \nu_{t}}{8 \mu^{2}}\right) \leq -\frac{\nu_t\kappa}{\mu}\cdot \frac{9L^2_F}{8\mu}<0\,.
% $$
Then plugging the three simplified coefficients into the inequality and omitting the terms with negative coefficients leads to
\begin{align}\label{ineq: main proof 6}
\mathbb{E}\left[Q_{t+1}-Q_{t}\right] \leq &-\frac{\nu_{t}}{16 \kappa} \mathbb{E}\left\|\widetilde{\theta}_{t+1}-\theta_{t}\right\|^{2}-\frac{4 \nu_{t} \kappa}{\mu \eta} \mathbb{E}\left\|p_{t}-\nabla_{\theta} F\left  (\theta_{t}, \omega_{t}\right) \right\|^{2}-\frac{\kappa \nu_{t} L_{F}^{2}}{2} \mathbb{E}\left\|\omega_{t}-\omega^{*}\left  (\theta_{t}\right) \right\|^{2}\notag \\
&+\frac{36 \sigma^{2} \nu_{t}^{2} \kappa}{\mu \eta}+\frac{4d\kappa\sigma^2_{t+1}}{\mu\eta}\,.
\end{align}
% \begin{equation}
% \label{ineq: main proof 6}
% \begin{aligned}
% \mathbb{E}\left[Q_{t+1}-Q_{t}\right] \leq &-\frac{\nu_{t}}{16 \kappa} \mathbb{E}\left\|\widetilde{\theta}_{t+1}-\theta_{t}\right\|^{2}\\
% &-\frac{4 \nu_{t} \kappa}{\mu \eta} \mathbb{E}\left\|p_{t}-\nabla_{\theta} F\left  (\theta_{t}, \omega_{t}\right) \right\|^{2}\\
% &-\frac{\kappa \nu_{t} L_{F}^{2}}{2} \mathbb{E}\left\|\omega_{t}-\omega^{*}\left  (\theta_{t}\right) \right\|^{2} \\
% &+\frac{36 \sigma^{2} \nu_{t}^{2} \kappa}{\mu \eta}+\frac{4d\kappa\sigma^2_{t+1}}{\mu\eta}\,.
% \end{aligned}
% \end{equation}
% Taking summation on both sides of Eq. (\ref{ineq: main proof 6}) from $t=0$ to $T-1$, and moving terms which are related to our metric from the right to the left, we will get
Taking summation on both sides of Eq. (\ref{ineq: main proof 6}) from $t=0$ to $T-1$ and
rearranging shows that
% \begin{align*}
% \sum_{t=0}^{T-1} \frac{\nu_{t} \kappa}{16}  \left( \frac{1}{\kappa^{2}} \mathbb{E}\left\|\widetilde{\theta}_{t+1}-\theta_{t}\right\|^{2}
% &+\frac{64}{\mu \eta} \mathbb{E}\left\|p_{t}-\nabla_{\theta} F\left  (\theta_{t}, \omega_{t}\right) \right\|^{2}+8 L_{F}^{2} \mathbb{E}\left\|\omega_{t}-\omega^{*}\left  (\theta_{t}\right) \right\|^{2} \right)\\
% &\leq \frac{36 \sigma^{2} \kappa \sum_{t=0}^{T-1} \nu_{t}^{2}+4\kappa d\sum_{t=1}^T\sigma_t^2}{\mu \eta}+\mathbb{E}\left[Q_{0}-Q_{T}\right]\\
% &\leq \frac{36 \sigma^{2} \kappa \sum_{t=0}^{T-1} \nu_{t}^{2}+4\kappa d\sum_{t=1}^T\sigma_t^2}{\mu \eta}+Q_{0}\,,
% \end{align*}
\begin{align*}
\sum_{t=0}^{T-1} &\frac{\nu_{t} \kappa}{16}  \left( \frac{1}{\kappa^{2}} \mathbb{E}\left\|\widetilde{\theta}_{t+1}-\theta_{t}\right\|^{2}+\frac{64}{\mu \eta} \mathbb{E}\left\|p_{t}-\nabla_{\theta} F\left  (\theta_{t}, \omega_{t}\right) \right\|^{2}+8 L_{F}^{2} \mathbb{E}\left\|\omega_{t}-\omega^{*}\left  (\theta_{t}\right) \right\|^{2} \right)\\
&\leq \frac{36 \sigma^{2} \kappa \sum_{t=0}^{T-1} \nu_{t}^{2}+4\kappa d\sum_{t=1}^T\sigma_t^2}{\mu \eta}+\mathbb{E}\left[Q_{0}-Q_{T}\right]\\
&\leq \frac{36 \sigma^{2} \kappa \sum_{t=0}^{T-1} \nu_{t}^{2}+4\kappa d\sum_{t=1}^T\sigma_t^2}{\mu \eta}+Q_{0}\,,
\end{align*}
% $$
% \begin{aligned}
% &\sum_{t=0}^{T-1} \frac{\nu_{t} \kappa}{16}  (\frac{1}{\kappa^{2}} \mathbb{E}\left\|\widetilde{\theta}_{t+1}-\theta_{t}\right\|^{2}+\frac{64}{\mu \eta} \mathbb{E}\left\|p_{t}-\nabla_{\theta} F\left  (\theta_{t}, \omega_{t}\right) \right\|^{2}\\
% &+8 L_{F}^{2} \mathbb{E}\left\|\omega_{t}-\omega^{*}\left  (\theta_{t}\right) \right\|^{2}) \\
% &\leq \frac{36 \sigma^{2} \kappa \sum_{t=0}^{T-1} \nu_{t}^{2}+4\kappa d\sum_{t=1}^T\sigma_t^2}{\mu \eta}+\mathbb{E}\left[Q_{0}-Q_{T}\right]\\
% &\leq \frac{36 \sigma^{2} \kappa \sum_{t=0}^{T-1} \nu_{t}^{2}+4\kappa d\sum_{t=1}^T\sigma_t^2}{\mu \eta}+Q_{0}\,.
% \end{aligned}
% $$
where the last inequality comes from $Q_t\geq 0, \forall t\geq0$ and $Q_0$ is determined in the initialization.

% Since  $\nu_t$ is set to $\frac{a}{  (t+b) ^\frac{1}{2}}$, we can use the fact $\nu_t\geq\nu_T$ for any $0\leq t\leq T$ to contract the left hand side of the above inequality. 
Since  $\nu_t$ is set to $\frac{a}{  (t+b) ^\frac{1}{2}}$, we can use the fact $\nu_t\geq\nu_T$ for any $0\leq t\leq T$ to upper bound the LHS of the above inequality. 
% And we have derived $\mu\eta\leq\frac{1}{4}$ before. Combining $\nu_t\geq \nu_T$ and $\frac{1}{\mu\eta}\geq 1$, we obtain 
Since $\mu\eta\leq\frac{1}{4}$, combining $\nu_t\geq \nu_T$ and $\frac{1}{\mu\eta}\geq 1$ shows that 
\begin{align*}
\frac{\nu_{T} \kappa}{16}  \sum_{t=0}^{T-1}&  \left(\frac{1}{\kappa^{2}} \mathbb{E}\left\|\widetilde{\theta}_{t+1}-\theta_{t}\right\|^{2}+\mathbb{E}\left\|p_{t}-\nabla_{\theta} F\left  (\theta_{t}, \omega_{t}\right) \right\|^{2}+L_{F}^{2} \mathbb{E}\left\|\omega_{t}-\omega^{*}\left  (\theta_{t}\right) \right\|^{2}\right)  \\
 \leq& \sum_{t=0}^{T-1} \frac{\nu_{t} \kappa}{16}  \left(\frac{1}{\kappa^{2}} \mathbb{E}\left\|\tilde{\theta}_{t+1}-\theta_{t}\right\|^{2}+\frac{64}{\mu \eta} \mathbb{E}\left\|p_{t}-\nabla_{\theta} F\left  (\theta_{t}, \omega_{t}\right) \right\|^{2}+8 L_{F}^{2} \mathbb{E}\left\|\omega_{t}-\omega^{*}\left  (\theta_{t}\right) \right\|^{2}\right)  \\
 \leq& \frac{36 \sigma^{2} \kappa \sum_{t=0}^{T-1} \nu_{t}^{2}+4\kappa d\sum_{t=1}^T\sigma_t^2}{\mu \eta}+Q_{0}\,.
\end{align*}
% $$
% \begin{aligned}
% &\frac{\nu_{T} \kappa}{16}  \sum_{t=0}^{T-1}  (\frac{1}{\kappa^{2}} \mathbb{E}\left\|\widetilde{\theta}_{t+1}-\theta_{t}\right\|^{2}+\mathbb{E}\left\|p_{t}-\nabla_{\theta} F\left  (\theta_{t}, \omega_{t}\right) \right\|^{2}\\
% &+L_{F}^{2} \mathbb{E}\left\|\omega_{t}-\omega^{*}\left  (\theta_{t}\right) \right\|^{2})  \\
%  \leq& \sum_{t=0}^{T-1} \frac{\nu_{t} \kappa}{16}  (\frac{1}{\kappa^{2}} \mathbb{E}\left\|\tilde{\theta}_{t+1}-\theta_{t}\right\|^{2}+\frac{64}{\mu \eta} \mathbb{E}\left\|p_{t}-\nabla_{\theta} F\left  (\theta_{t}, \omega_{t}\right) \right\|^{2}\\
% &+8 L_{F}^{2} \mathbb{E}\left\|\omega_{t}-\omega^{*}\left  (\theta_{t}\right) \right\|^{2})  \\
%  \leq& \frac{36 \sigma^{2} \kappa \sum_{t=0}^{T-1} \nu_{t}^{2}+4\kappa d\sum_{t=1}^T\sigma_t^2}{\mu \eta}+Q_{0}\,.
% \end{aligned}
% $$
% Multiplying both sides by $16/  (T\nu_T\kappa) $, we get
% \zhao{$+\frac{16Q_{0}}{\kappa \nu_T T}$?}
Rearranging the above display leads to
\begin{align}\label{ineq: main proof 7}
\frac{1}{T}  \sum_{t=0}^{T-1}  \left(\frac{1}{\kappa^{2}} \mathbb{E}\left\|\widetilde{\theta}_{t+1}-\theta_{t}\right\|^{2}+\mathbb{E}\left\|p_{t}-\nabla_{\theta} F\left  (\theta_{t}, \omega_{t}\right) \right\|^{2}+L_{F}^{2} \mathbb{E}\left\|\omega_{t}-\omega^{*}\left  (\theta_{t}\right) \right\|^{2}\right) 
 \leq& \frac{576 \sigma^{2}  \sum_{t=0}^{T-1} \nu_{t}^{2}+64d\sum_{t=1}^T\sigma_t^2}{\mu \eta T\nu_T}\notag\\
 &+\frac{16Q_{0}}{\kappa \nu_T T}\,.
\end{align}
% \begin{equation}
% \label{ineq: main proof 7}
% \begin{aligned}
% &\frac{1}{T}  \sum_{t=0}^{T-1}  (\frac{1}{\kappa^{2}} \mathbb{E}\left\|\widetilde{\theta}_{t+1}-\theta_{t}\right\|^{2}+\mathbb{E}\left\|p_{t}-\nabla_{\theta} F\left  (\theta_{t}, \omega_{t}\right) \right\|^{2}\\
% &+L_{F}^{2} \mathbb{E}\left\|\omega_{t}-\omega^{*}\left  (\theta_{t}\right) \right\|^{2})  \\
%  \leq& \frac{576 \sigma^{2}  \sum_{t=0}^{T-1} \nu_{t}^{2}+64d\sum_{t=1}^T\sigma_t^2}{\mu \eta T\nu_T}+\frac{16Q_{0}}{\kappa \nu_T T}\zhao{+\frac{16Q_{0}}{\kappa \nu_T T}?}\,.
% \end{aligned}
% \end{equation}
Applying Jensen's inequality to the LHS of Eq.  (\ref{ineq: main proof 7})  shows that
\begin{align}\label{ineq: main proof 8}
\frac{1}{T} \sum_{t=0}^{T-1}& \mathbb{E}  \left(\frac{1}{\kappa}\left\|\widetilde{\theta}_{t+1}-\theta_{t}\right\|+\left\|p_{t}-\nabla_{\theta} F\left  (\theta_{t}, \omega_{t}\right) \right\|+L_{F}\left\|\omega_{t}-\omega^{*}\left  (\theta_{t}\right) \right\|\right) \notag\\
\leq&\left[\frac{3}{T}\sum_{t=0}^{T-1} \mathbb{E}  \left(\frac{1}{\kappa^{2}}\left\|\widetilde{\theta}_{t+1}-\theta_{t}\right\|^{2}+\left\|p_{t}-\nabla_{\theta} F\left  (\theta_{t}, \omega_{t}\right) \right\|^{2}+L_{F}^{2}\left\|\omega_{t}-\omega^{*}\left  (\theta_{t}\right) \right\|^{2}\right) \right]^{1 / 2}\,.
\end{align}
% \begin{equation}
% \label{ineq: main proof 8}
% \begin{aligned}
% &\frac{1}{T} \sum_{t=0}^{T-1} \mathbb{E}  (\frac{1}{\kappa}\left\|\widetilde{\theta}_{t+1}-\theta_{t}\right\|+\left\|p_{t}-\nabla_{\theta} F\left  (\theta_{t}, \omega_{t}\right) \right\| \\
% &+L_{F}\left\|\omega_{t}-\omega^{*}\left  (\theta_{t}\right) \right\|) \\
% \leq&[\frac{3}{T}\sum_{t=0}^{T-1} \mathbb{E}  (\frac{1}{\kappa^{2}}\left\|\widetilde{\theta}_{t+1}-\theta_{t}\right\|^{2}+\left\|p_{t}-\nabla_{\theta} F\left  (\theta_{t}, \omega_{t}\right) \right\|^{2}\\
% &+L_{F}^{2}\left\|\omega_{t}-\omega^{*}\left  (\theta_{t}\right) \right\|^{2}) ]^{1 / 2}
% \end{aligned}
% \end{equation}
% \zhao{it is not wrong but why do we need to relax $\frac{1}{T}$ to $\frac{3}{T}$ here?}\ze{3 comes from Jensen's inequality directly.}\zhao{to be fixed}
% For the RHS of Eq. (\ref{ineq: main proof 7}), we can bound it by the value of $\nu_t$, $\sigma_t$. 
The RHS of Eq. (\ref{ineq: main proof 7}) could be bounded by of $\nu_t$ and $\sigma_t$. 
By the Gaussian noise set in Theorem \ref{theorem: privacy guarantee}, we have
$$
\sigma_t^2=\frac{14G^2T\alpha'}{n^2\beta'\epsilon},\ \forall t\geq 0 \,,
$$
where $\alpha'=\frac{\log  (1/\delta) }{  (1-\beta') \epsilon}+1$ and $\beta'\in  (0,1) $. Recall that $\nu_t=\frac{1}{4  (t+b) ^\frac{1}{2}}$. Then we obtain
\begin{align}\label{ineq: main proof 9}
\frac{576 \sigma^{2}  \sum_{t=0}^{T-1} \nu_{t}^{2}+64d\sum_{t=1}^T\sigma_t^2}{\mu \eta T\nu_T}+\frac{16Q_{0}}{\kappa \nu_T T}\leq& \frac{144 \sigma^{2}}{\mu \eta T}   (T+b) ^\frac{1}{2}\log  (T+b)  +\frac{3584 G^2d T   (T+b) ^{\frac 12} \left  (\frac{\log \left  (\frac{1}{\delta }\right) }{  (1-\beta' )  \epsilon }+1\right) }{\beta'  \eta  \mu  n^2 \epsilon }\notag\\   
&+\frac{64 Q_{0}  (T+b) ^{1 / 2}}{\kappa T}\notag\\
\leq &\left  (\frac{64 Q_{0}}{\kappa}+\frac{144 \sigma^{2}}{\mu \eta}\right) \left  (\frac{1}{\sqrt{T}}+\frac{\sqrt{b}}{T}\right)  \log   (T+b) \notag \\
&+\frac{3584 G^2d T   (T^\frac{1}{2}+b^{\frac 12})  \left  (\frac{\log \left  (\frac{1}{\delta }\right) }{  (1-\beta' )  \epsilon }+1\right) }{\beta'  \eta  \mu  n^2 \epsilon }\,,
\end{align}
where the last inequality is due to $\sqrt{x+y}\leq \sqrt{x}+\sqrt{y}$ for $x\geq 0$, $y\geq 0$.

\noindent Now combining Eq. (\ref{ineq: main proof 7}), Eq. (\ref{ineq: main proof 8}) and Eq. (\ref{ineq: main proof 9}), we have
\begin{align}\label{ineq: main proof 10}
\frac{1}{T} \sum_{t=0}^{T-1}& \mathbb{E}  \left(\frac{1}{\kappa}\left\|\widetilde{\theta}_{t+1}-\theta_{t}\right\|+\left\|p_{t}-\nabla_{\theta} F\left  (\theta_{t}, \omega_{t}\right) \right\| 
+L_{F}\left\|\omega_{t}-\omega^{*}\left  (\theta_{t}\right) \right\|\right) \notag\\
\leq &\sqrt{3}\left  (\frac{576 \sigma^{2}  \sum_{t=0}^{T-1} \nu_{t}^{2}+64d\sum_{t=1}^T\sigma_t^2}{\mu \eta T\nu_T}+\frac{16Q_{0}}{\mu \nu_T T}\right) ^\frac{1}{2}\notag\\
\leq &\sqrt{3}\left  (\frac{64 Q_{0}}{\kappa}+\frac{144 \sigma^{2}}{\mu \eta}\right) ^\frac{1}{2}\left  (\frac{1}{{T}^\frac{1}{4}}+\frac{{b}^\frac{1}{4}}{T^\frac{1}{2}}\right)  \log^\frac{1}{2}   (T+b) +\frac{60\sqrt{3} Gd^\frac{1}{2}   (T^\frac{3}{4}+b^{\frac 14}T^\frac{1}{2}) \left  (\frac{\log^\frac{1}{2} \left  (\frac{1}{\delta }\right) }{  (1-\beta' ) ^\frac{1}{2} \epsilon^\frac{1}{2} }+1\right) }{{\beta'}^\frac{1}{2}  \eta^\frac{1}{2}  \mu^\frac{1}{2}  n \epsilon^\frac{1}{2}}\notag\\
\leq &\sqrt{3}\left  (\frac{64 Q_{0}}{\kappa}+\frac{144 \sigma^{2}}{\mu \eta}\right) ^\frac{1}{2}\left  (\frac{1}{{T}^\frac{1}{4}}+\frac{{b}^\frac{1}{4}}{T^\frac{1}{2}}\right)  \log^\frac{1}{2}   (T+b) +\frac{120\sqrt{3} G d^\frac{1}{2}  (T^\frac{3}{4}+b^{\frac 14}T^\frac{1}{2}) \log^\frac{1}{2} \left  (\frac{1}{\delta }\right) }{ \eta^\frac{1}{2}  \mu^\frac{1}{2}   {\beta'}^\frac{1}{2}   (1-\beta' ) ^\frac{1}{2} n\epsilon}
\,,
\end{align}
% \begin{equation}
% \label{ineq: main proof 10}
% \begin{aligned}
% &\frac{1}{T} \sum_{t=0}^{T-1} \mathbb{E}  (\frac{1}{\kappa}\left\|\widetilde{\theta}_{t+1}-\theta_{t}\right\|+\left\|p_{t}-\nabla_{\theta} F\left  (\theta_{t}, \omega_{t}\right) \right\| \\
% &+L_{F}\left\|\omega_{t}-\omega^{*}\left  (\theta_{t}\right) \right\|) \\
% \leq &\sqrt{3}\left  (\frac{576 \sigma^{2}  \sum_{t=0}^{T-1} \nu_{t}^{2}+64d\sum_{t=1}^T\sigma_t^2}{\mu \eta T\nu_T}+\frac{16Q_{0}}{\mu \nu_T T}\right) ^\frac{1}{2}\\
% \leq &\sqrt{3}\left  (\frac{64 Q_{0}}{\kappa}+\frac{144 \sigma^{2}}{\mu \eta}\right) ^\frac{1}{2}\left  (\frac{1}{{T}^\frac{1}{4}}+\frac{{b}^\frac{1}{4}}{T^\frac{1}{2}}\right)  \log^\frac{1}{2}   (T+b) \\
% &+\frac{60\sqrt{3} Gd^\frac{1}{2}   (T^\frac{3}{4}+b^{\frac 14}T^\frac{1}{2}) \left  (\frac{\log^\frac{1}{2} \left  (\frac{1}{\delta }\right) }{  (1-\beta' ) ^\frac{1}{2} \epsilon^\frac{1}{2} }+1\right) }{{\beta'}^\frac{1}{2}  \eta^\frac{1}{2}  \mu^\frac{1}{2}  n \epsilon^\frac{1}{2}}\\
% \leq &\sqrt{3}\left  (\frac{64 Q_{0}}{\kappa}+\frac{144 \sigma^{2}}{\mu \eta}\right) ^\frac{1}{2}\left  (\frac{1}{{T}^\frac{1}{4}}+\frac{{b}^\frac{1}{4}}{T^\frac{1}{2}}\right)  \log^\frac{1}{2}   (T+b) \\
% &+\frac{120\sqrt{3} G d^\frac{1}{2}  (T^\frac{3}{4}+b^{\frac 14}T^\frac{1}{2}) \log^\frac{1}{2} \left  (\frac{1}{\delta }\right) }{ \eta^\frac{1}{2}  \mu^\frac{1}{2}   {\beta'}^\frac{1}{2}   (1-\beta' ) ^\frac{1}{2} n\epsilon}
% \,,
% \end{aligned}
% \end{equation}
where the second inequality is due to $\sqrt{x+y}\leq \sqrt{x}+\sqrt{y}$ for $x\geq 0,y\geq 0$. Simplifying  Eq.  (\ref{ineq: main proof 10})  by replacing the constant coefficients with $C_i$ leads to 
\begin{align}\label{ineq: main proof 11}
\frac{1}{T} \sum_{t=0}^{T-1}& \mathbb{E}  \left(\frac{1}{\kappa}\left\|\widetilde{\theta}_{t+1}-\theta_{t}\right\|+\left\|p_{t}-\nabla_{\theta} F\left  (\theta_{t}, \omega_{t}\right) \right\| +L_{F}\left\|\omega_{t}-\omega^{*}\left  (\theta_{t}\right) \right\|\right) \notag\\
\leq &\sqrt{3}\left  (\frac{64 Q_{0}}{\kappa}+\frac{144 \sigma^{2}}{\mu \eta}\right) ^\frac{1}{2}\left  (\frac{1}{{T}^\frac{1}{4}}+\frac{{b}^\frac{1}{4}}{T^\frac{1}{2}}\right)  \log^\frac{1}{2}   (T+b) +\frac{120\sqrt{3} G d^\frac{1}{2}  (T^\frac{3}{4}+b^{\frac 14}T^\frac{1}{2}) \log^\frac{1}{2} \left  (\frac{1}{\delta }\right) }{ \eta^\frac{1}{2}  \mu^\frac{1}{2}   {\beta'}^\frac{1}{2}   (1-\beta' ) ^\frac{1}{2} n\epsilon}\notag\\
= &\frac{C_1\log^\frac{1}{2}   (T+b) }{{T}^\frac{1}{4}}+\frac{C_2\log^\frac{1}{2}   (T+b) }{T^\frac{1}{2}}+\frac{   (C_3T^\frac{3}{4}+C_4T^\frac{1}{2}) d^\frac{1}{2}\log^\frac{1}{2} \left  (\frac{1}{\delta }\right) }{ n\epsilon}\,,
\end{align}
% \begin{equation}
% \label{ineq: main proof 11}
% \begin{aligned}
% &\frac{1}{T} \sum_{t=0}^{T-1} \mathbb{E}  (\frac{1}{\kappa}\left\|\widetilde{\theta}_{t+1}-\theta_{t}\right\|+\left\|p_{t}-\nabla_{\theta} F\left  (\theta_{t}, \omega_{t}\right) \right\| \\
% &+L_{F}\left\|\omega_{t}-\omega^{*}\left  (\theta_{t}\right) \right\|) \\
% \leq &\sqrt{3}\left  (\frac{64 Q_{0}}{\kappa}+\frac{144 \sigma^{2}}{\mu \eta}\right) ^\frac{1}{2}\left  (\frac{1}{{T}^\frac{1}{4}}+\frac{{b}^\frac{1}{4}}{T^\frac{1}{2}}\right)  \log^\frac{1}{2}   (T+b) \\
% &+\frac{120\sqrt{3} G d^\frac{1}{2}  (T^\frac{3}{4}+b^{\frac 14}T^\frac{1}{2}) \log^\frac{1}{2} \left  (\frac{1}{\delta }\right) }{ \eta^\frac{1}{2}  \mu^\frac{1}{2}   {\beta'}^\frac{1}{2}   (1-\beta' ) ^\frac{1}{2} n\epsilon}\\
% = &\frac{C_1\log^\frac{1}{2}   (T+b) }{{T}^\frac{1}{4}}+\frac{C_2\log^\frac{1}{2}   (T+b) }{T^\frac{1}{2}}\\
% &+\frac{   (C_3T^\frac{3}{4}+C_4T^\frac{1}{2}) d^\frac{1}{2}\log^\frac{1}{2} \left  (\frac{1}{\delta }\right) }{ n\epsilon}\,,
% \end{aligned}
% \end{equation}
where $C_1=\sqrt{3}\left  (\frac{64 Q_{0}}{\kappa}+\frac{144 \sigma^{2}}{\mu \eta}\right) ^\frac{1}{2}=\sqrt{3}\left  (\frac{64 Q_{0}}{\kappa}+\frac{288 G^{2}}{\mu \eta}\right) ^\frac{1}{2}$, $C_2=b^\frac{1}{4}C_1$, $C_3=\frac{120\sqrt{3}G}{  (\beta'  (1-\beta') \eta\mu) ^\frac{1}{2}}, C_4=b^\frac{1}{4}C_3$.

By the parameter setting in Theorem \ref{theorem: utility}, with $T = \frac{C_5n\epsilon}{\sqrt{d\log  (1/\delta) }}$
where $C_5$ is a constant, we can obtain the final bound is
\begin{align*}
\frac{1}{T} \sum_{t=0}^{T-1}& \mathbb{E}  \left(\frac{1}{\kappa}\left\|\widetilde{\theta}_{t+1}-\theta_{t}\right\|+\left\|p_{t}-\nabla_{\theta} F\left  (\theta_{t}, \omega_{t}\right) \right\| +L_{F}\left\|\omega_{t}-\omega^{*}\left  (\theta_{t}\right) \right\|\right) \\
\leq&\frac{ \sqrt[8]{d\log \left  (\frac{1}{\delta }\right) } }{\sqrt{C_5n\epsilon }}   (C_1 \sqrt[4]{C_5n\epsilon } \sqrt{\log \left  (b+\frac{C_5 n \epsilon }{\sqrt{d} \sqrt{\log \left  (\frac{1}{\delta }\right) }}\right) }+C_2 \sqrt[8]{d\log \left  (\frac{1}{\delta }\right) } \sqrt{\log \left  (b+\frac{C_5 n \epsilon }{\sqrt{d} \sqrt{\log \left  (\frac{1}{\delta }\right) }}\right) }\\
&+C_4 C_5 \sqrt[8]{d\log \left  (\frac{1}{\delta }\right) }+C_3 C_5^{5/4} \sqrt[4]{n\epsilon }) \,.
\end{align*}
% $$
% \begin{aligned}
% &\frac{1}{T} \sum_{t=0}^{T-1} \mathbb{E}  (\frac{1}{\kappa}\left\|\widetilde{\theta}_{t+1}-\theta_{t}\right\|+\left\|p_{t}-\nabla_{\theta} F\left  (\theta_{t}, \omega_{t}\right) \right\| \\
% &+L_{F}\left\|\omega_{t}-\omega^{*}\left  (\theta_{t}\right) \right\|) \\
% \leq&\frac{ \sqrt[8]{d\log \left  (\frac{1}{\delta }\right) } }{\sqrt{C_5n\epsilon }}   (C_1 \sqrt[4]{C_5n\epsilon } \sqrt{\log \left  (b+\frac{C_5 n \epsilon }{\sqrt{d} \sqrt{\log \left  (\frac{1}{\delta }\right) }}\right) }\\
% &+C_2 \sqrt[8]{d\log \left  (\frac{1}{\delta }\right) } \sqrt{\log \left  (b+\frac{C_5 n \epsilon }{\sqrt{d} \sqrt{\log \left  (\frac{1}{\delta }\right) }}\right) }\\
% &+C_4 C_5 \sqrt[8]{d\log \left  (\frac{1}{\delta }\right) }+C_3 C_5^{5/4} \sqrt[4]{n\epsilon }) 
% \end{aligned}
% $$
If we hide the factor $\sqrt{\log \left  (b+\frac{C_{5} n \epsilon}{\sqrt{d} \sqrt{\log \left  (\frac{1}{\delta}\right) }}\right) }$, \textit{i.e.}, $\sqrt{\log  (b+T) }$,  we obtain
\begin{align*}
\frac{1}{T} \sum_{t=0}^{T-1} \mathbb{E}  \left(\frac{1}{\kappa}\left\|\widetilde{\theta}_{t+1}-\theta_{t}\right\|+\left\|p_{t}-\nabla_{\theta} F\left  (\theta_{t}, \omega_{t}\right) \right\| +L_{F}\left\|\omega_{t}-\omega^{*}\left  (\theta_{t}\right) \right\|\right) 
\leq \widetilde{\mathcal{O}} \left (\frac{  (d\log  (1/\delta) ) ^\frac{1}{8}}{  (n\epsilon) ^\frac{1}{4}}\right) \,.
\end{align*}
\textbf{Gradient Complexity.} The gradient complexity is equal to $2(T+1)=\mathcal{O}\left(\frac{n\epsilon}{\sqrt{d\log  (1/\delta) }}\right)$ since Algorithm \ref{alg: DPTD} computes gradients for both the primal side and the dual side. 
% \begin{equation}
% \begin{aligned}
% &\frac{1}{T} \sum_{t=0}^{T-1} \mathbb{E}  (\frac{1}{\kappa}\left\|\widetilde{\theta}_{t+1}-\theta_{t}\right\|+\left\|p_{t}-\nabla_{\theta} F\left  (\theta_{t}, \omega_{t}\right) \right\| \\
% &+L_{F}\left\|\omega_{t}-\omega^{*}\left  (\theta_{t}\right) \right\|) \\
% \leq &\widetilde O  (\frac{  (d\log  (1/\delta) ) ^\frac{1}{8}}{  (n\epsilon) ^\frac{1}{4}}) \,.
% \end{aligned}
% \end{equation}
\end{proof}

% \newpage

% \subsection{Proof of Theorem \ref{theorem: utility in TDP}}

\section{Proof of Theorem \ref{theorem: privacy guarantee   (TDP) }}
In this section, we provide the proof of Theorem \ref{theorem: privacy guarantee   (TDP) }. 
\begin{proof}[Proof of Theorem \ref{theorem: privacy guarantee   (TDP) }]
According to the update rules in Algorithm \ref{alg: DPTD}, our mechanisms are constructed as 
\begin{equation}
\label{eq: random mechanism of p_t in TDP}
    \mathcal{M}_{t}^p = \left\{
    \begin{aligned}
    &  (1-\alpha\nu_{t-1}) p_{t-1} +\alpha\nu_{t-1}\nabla_\theta f  (\theta_{t},\omega_{t};\xi_{t}) +u_{t}^p,&t> 0\\
    &\nabla_\theta f  (\theta_{0},\omega_{0};\xi_{0}) +u_{0}^p,&t=0
    \end{aligned}\right.\,,
\end{equation}
and
\begin{equation}
\label{eq: random mechansim of d_t in TDP}
\mathcal{M}_{t}^d =\left\{
    \begin{aligned}
    &  (1-\beta\nu_{t-1}) d_{t-1} +\beta\nu_{t-1}\nabla_\omega f  (\theta_{t},\omega_{t};\xi_{t}) +u_{t}^d,&t>0\\
    &\nabla_\omega f  (\theta_{0},\omega_{0};\xi_{0}) +u_{0}^d,&t=0
    \end{aligned}\right. \,.
\end{equation}

% We prove the privacy guarantee of the mechanism on primal side and the proof of the privacy guarantee of the mechanism on dual side follows similarly.\zhao{20210825}
% Given two neighbouring datasets $S$ and $\hat{S}$, both of them consists of $m$ trajectories and they only differ in single trajectory. We aim to show the privacy guarantee of $M_t^p$ and $M_t^d$ for $t=0,1,..,T-1$. Similar to the proof of \ref{theorem: privacy guarantee}, we discuss two cases, $t=0$ and $t>0$.
We aim to show the privacy guarantee of $\mathcal{M}_t^p$ and $\mathcal{M}_t^d$ for $t=0,1,..,T-1$. 
We prove the privacy guarantee of the mechanism on the primal side   (\textit{i.e.}, $\mathcal{M}_t^p$)  and the proof of the privacy guarantee of the mechanism on the dual side   (\textit{i.e.}, $\mathcal{M}_t^d$)  follows similarly.
% Similar to the proof of \ref{theorem: privacy guarantee}, we discuss two cases, $t=0$ and $t>0$.
Similar to the proof of \ref{theorem: privacy guarantee}, we start from the case when $t=0$ and then discuss the case when $t>0$.

\noindent\textbf{Case   (a) .} If $t=0$, we consider the following Gaussian mechanism
\begin{align*}
    \mathcal{G}_0^p = \nabla_\theta f  (\theta_{0},\omega_{0};\xi_{0}) +u_{0}^p\,,
\end{align*}
% $$
% \mathcal{G}_0^p = \nabla_\theta f  (\theta_{0},\omega_{0};\xi_{0}) +u_{0}^p\,,
% $$
% $$
% \mathcal{G}_0^d=\nabla_\omega f  (\theta_{0},\omega_{0};\xi_{0}) +u_{0}^d\,,
% $$
% where $u_{0}^p\sim N  (0,\sigma_0^2\mathbf{I}_d) $, $u_{0}^d\sim N  (0,\sigma_0^2\mathbf{I}_d) $. Then we consider the Gaussian mechanisms without subsampling, which means we have the  access to the full dataset. Denoting $tr_{i}$ as the $i$th trajectory in $S$, $\mathrm{len}  (tr_i) $ as the length of $tr_i$ and $\xi_{i,j}$ as the $j$th triple in $tr_i$,  we obtain
where $u_{0}^p\sim N  (0,\sigma_0^2\mathbf{I}_d) $. To provide the privacy guarantee of the above mechanism, we first prove the privacy guarantee of the following Gaussian mechanisms $\widetilde{\mathcal{G}}_0^p$ without subsampling, which means we have the  access to the full dataset. Denote by $\tau_i$ the $i$-th trajectory in dataset $S$ and by $\xi_{i,j}$ the $j$-th triple in $\tau_i$. Specifically, $\widetilde{\mathcal{G}}_0^p$ is constructed as
\begin{align*}
    \widetilde{\mathcal{G}}_0^p =\sum_{i=0}^{m-1} \sum_{j=0}^{|\tau_i|-1}\nabla_\theta f  (\theta_{0},\omega_{0};\xi_{i,j}) +u_{0}^p\,.
\end{align*}
% $$
% \widetilde{\mathcal{G}_0^p} =\sum_{i=0}^{m-1} \sum_{j=0}^{\mathrm{len}  (tr_i) -1}\nabla_\theta f  (\theta_{0},\omega_{0};\xi_{i,j}) +u_{0}^p\,,
% $$
% $$
% \widetilde{\mathcal{G}_0^d}=\sum_{i=0}^{m-1}\sum_{j=0}^{\mathrm{len}  (tr_i) -1}\nabla_\omega f  (\theta_{0},\omega_{0};\xi_{i,j}) +u_{0}^d\,,
% $$

\noindent\textbf{Sensitivity.} Consider the query on the dataset $S$ as follows
\begin{align*}
    \widetilde {q}_0^p  (S) =\sum_{i=0}^{m-1} \sum_{j=0}^{|\tau_i|-1}\nabla_\theta f  (\theta_{0},\omega_{0};\xi_{i,j}) \,.
\end{align*}

% $$
% \widetilde {q_0^p}  (S) =\sum_{i=0}^{m-1} \sum_{j=0}^{\mathrm{len}  (tr_i) -1}\nabla_\theta f  (\theta_{0},\omega_{0};\xi_{i,j}) \,,
% $$
% $$
% \widetilde{q_0^d}  (S) =\sum_{i=0}^{m-1}\sum_{j=0}^{\mathrm{len}  (tr_i) -1}\nabla_\omega f  (\theta_{0},\omega_{0};\xi_{i,j}) \,.
% $$
% Similarly we can get $\widetilde {q}_0^p  (S') $ and $\widetilde{q}_0^d  (S') $. To get the sensitivity, we first obtain
Similarly we can get $\widetilde {q}_0^p  (\hat{S}) $. 
% To get the sensitivity, we first obtain
% \begin{align*}
%     \widetilde{q}_0^p  (S)  - \widetilde{q}_0^p  (\hat{S}) 
% =\sum_{j=0}^{|\tau_i|-1}\nabla_\theta f  (\theta_{0},\omega_{0};\xi_{i,j})  - \sum_{j=0}^{|\hat{\tau}_i|-1}\nabla_\theta f  (\theta_{0},\omega_{0};\hat{\xi}_{i,j}) \,.
% \end{align*}
% $$
% \begin{aligned}
% &\widetilde{q_0^p}  (S)  - \widetilde{q_0^p}  (S') \\
% =&\sum_{j=0}^{\mathrm{len}  (tr_i) -1}\nabla_\theta f  (\theta_{0},\omega_{0};\xi_{i,j})  - \sum_{j=0}^{\mathrm{len}  (tr_{i'}) -1}\nabla_\theta f  (\theta_{0},\omega_{0};\xi_{i',j}) 
% \end{aligned}
% $$
% $$
% \begin{aligned}
% &\widetilde{q_0^d}  (S)  - \widetilde{q_0^d}  (S') \\
% =&\sum_{j=0}^{\mathrm{len}  (tr_i) -1}\nabla_\omega f  (\theta_{0},\omega_{0};\xi_{i,j})  - \sum_{j=0}^{\mathrm{len}  (tr_{i'}) -1}\nabla_\omega f  (\theta_{0},\omega_{0};\xi_{i',j}) 
% \end{aligned}
% $$
% Thus we bound $\ell_2$-sensitivity as follows
Then the $\ell_2$-sensitivity could be bounded as follows
\begin{align*}
\widetilde{\Delta}_t^p =& \left\|\widetilde{q}_0^p  (S)  - \widetilde{q}_0^p  (\hat{S})  \right\| \\
=& \left\|\sum_{j=0}^{|\tau_i|-1}\nabla_\theta f  (\theta_{0},\omega_{0};\xi_{i,j})  - \sum_{j=0}^{|\hat{\tau}_i|-1}\nabla_\theta f  (\theta_{0},\omega_{0};\hat{\xi}_{i,j})  \right\| \\
\leq & \left\|\sum_{j=0}^{|\tau_i|-1}\nabla_\theta f  (\theta_{0},\omega_{0};\xi_{i,j}) \right\| +\left\| \sum_{j=0}^{|\hat{\tau}_i|-1}\nabla_\theta f  (\theta_{0},\omega_{0};\hat{\xi}_{i,j}) \right\|\\
\leq & 2nG
\end{align*}
The last inequality is because of Assumption \ref{asmp: upper bound of summation}. Similarly, we obtain $\widetilde{\Delta}_0^d\leq2nG $.
\noindent\textbf{Privacy guarantee of $\mathcal{G}_t^p$.} 
By Lemma \ref{lemma: RDP subsampling transformation}, if the Gaussian noise $u_t^p$ has the following variance
\begin{align*}
    \sigma^2_t = \frac{14n^2G^2T\alpha'}{m^2  \left(\epsilon - \frac{\log  (1/\delta) }{\alpha'-1}\right) }\,,
\end{align*}
% $$
% \sigma^2_t = \frac{14n^2  (D_\theta+\epsilon_E) ^2T\alpha'}{m^2  (\epsilon - \frac{\log  (1/\delta) }{\alpha'-1}) }
% $$
where $\sigma'^2=\frac{\sigma_t^2}{4n^2G^2}\geq 0.7$, $\alpha'=\frac{\log  (1/\delta) }{  (1-\beta') \epsilon}+1\leq 2\sigma^2\log  (\frac{n}{\alpha'  (1+\sigma'^2) }) /3+1$ and $\beta'\in   (0,1) $, then our mechanism $\mathcal{G}_t^p$ will satisfy $  \left(\alpha', \frac{14\alpha'n^2G^2}{m^2\sigma_t^2}\right) $-RDP. 
% \begin{align*}
    
% \end{align*}
% $$
%   (\alpha', \frac{14\alpha'n^2  (D_\theta+\epsilon_E) ^2}{m^2\sigma_t^2}) \text{-RDP}
% $$
% Similarly, if the Gaussian noises $u_t^d$ has the following variance
% $$
% \sigma^2_t = \frac{14n^2  (D_\omega+\epsilon_E) ^2T\alpha'}{m^2  (\epsilon - \frac{\log  (1/\delta) }{\alpha'-1}) }
% $$
% where $\sigma'^2=\frac{\sigma_t^2}{4n^2  (D_\omega+\epsilon_E) ^2}\geq 0.7$, $\alpha'=\frac{\log  (1/\delta) }{  (1-\beta') \epsilon}+1\leq 2\sigma^2\log  (\frac{n}{\alpha'  (1+\sigma'^2) }) /3+1$ and $\beta'\in   (0,1) $, then our mechanisms will satisfy the following
% $$
%   (\alpha', \frac{14\alpha'n^2  (D_\omega+\epsilon_E) ^2}{m^2\sigma_t^2}) \text{-RDP}
% $$
% If we take $D_\omega=D_\theta$, they will follow the same RDP.

\noindent\textbf{Case (b).} The sensitivity is bounded the same as in the Case (a). 
% Then we obtain that for each Gaussian mechanism, they are bounded as
Then one can see that Gaussian mechanism $\widetilde{\Delta}^p_t$ is bounded as 
\begin{align*}
    \widetilde{\Delta}^p_t\leq 2\alpha \nu_{t-1}nG\leq 2nG
\end{align*}
% $$
% \widetilde{\Delta^p_t}\leq 2\alpha \nu_{t-1}n  (D_\theta+\epsilon_E) \leq 2n  (D_\theta+\epsilon_E) 
% $$
% $$
% \widetilde{\Delta^d_t}\leq 2\beta \nu_{t-1}n  (D_\omega+\epsilon_E) \leq 2n  (D_\omega+\epsilon_E) 
% $$
with the probability $P_E$.

It is clear that the mechanisms in Case (a) and Case (b) are able to satisfy the same RDP under the same Gaussian noise since they have the same upper bound of the sensitivity.

% \noindent\textbf{Privacy guarantee of $\mathcal{M}^p_t$ and $\mathcal{M}^d_t$.} By the definition of $\mathcal{M}^p_t$ in equality \ref{eq: random mechanism of p_t in TDP} and $\mathcal{M}^d_t$ in equality \ref{eq: random mechansim of d_t in TDP}, they are composed of several Gaussian mechanisms, i.e. $\mathcal{M}_t^p=  (\mathcal{G}^p_0,...,\mathcal{G}^p_t) $ and $\mathcal{M}_t^d=  (\mathcal{G}^d_0,...,\mathcal{G}^d_t) $. According to the composition property of RDP, i.e. Lemma \ref{lemma: composition of RDP}, and setting $D_\theta=D_\omega$, we have $\mathcal{M}_t^p$ and $\mathcal{M}_t^d$ both satisfy $  (\alpha', \sum_{i=0}^t\frac{14\alpha'n^2  (D_\theta+\epsilon_E) ^2}{m^2\sigma_t^2}) $-RDP. Thus the output $\theta_T$ and $\omega_T$ of algorithm \ref{alg: DPTD} satisfy $  (\alpha', \sum_{i=0}^t\frac{14\alpha'n^2  (D_\theta+\epsilon_E) ^2}{m^2\sigma_t^2}) $-RDP. Finally, by using Lemma \ref{lemma: RDP_to_DP}, we transform RDP to DP and thus $\theta_T$ and $\omega_T$ satisfy the following 
% $$
%   (\sum_{i=0}^T\left  (\frac{14\alpha'n^2  (D_\theta+\epsilon_E) ^2}{m^2\sigma_i^2}\right) +\frac{log  (1/\delta) }{\alpha'-1}, \delta) \text{-DP}\,.
% $$
% Plugging in the value of $\sigma_t$, we can simplify it as 
% $$
%   (\epsilon, \delta) \text{-DP}\,,
% $$
% which is $  (\epsilon, \delta) $-TDP under this setting.
\noindent\textbf{Privacy guarantee of $\mathcal{M}^p_t$.} Due to the definition of $\mathcal{M}^p_t$ in Eq. (\ref{eq: random mechanism of p_t in TDP}), $\mathcal{M}^p_t$ is composed of several Gaussian mechanisms, \textit{i.e.}, $\mathcal{M}_t^p=  (\mathcal{G}^p_0,...,\mathcal{G}^p_t) $. Then Lemma \ref{lemma: composition of RDP} implies that  $\mathcal{M}_t^p$ and the output on the primal side satisfies $  \left(\alpha', \sum_{i=0}^T\frac{14\alpha'n^2  G ^2}{m^2\sigma_i^2}\right) $-RDP.
% Thus the output $\theta_T$ and $\omega_T$ of algorithm \ref{alg: DPTD} satisfy $  (\alpha', \sum_{i=0}^t\frac{14\alpha'n^2  (D_\theta+\epsilon_E) ^2}{m^2\sigma_t^2}) $-RDP. 
% Finally, by using Lemma \ref{lemma: RDP_to_DP}, we transform RDP to DP and thus $\theta_T$ and $\omega_T$ satisfy the following
Applying Lemma \ref{lemma: RDP_to_DP} shows that the output satisfies 
$\left  (\sum_{i=0}^T\left  (\frac{14\alpha'n^2  G ^2}{m^2\sigma_i^2}\right) +\frac{log  (1/\delta) }{\alpha'-1}, \delta\right) $-DP.
Substituting the value of $\sigma_t$ simplifies it as $  (\epsilon, \delta) $-DP under trajectory which concludes the proof.
% which is $  (\epsilon, \delta) $-TDP under this setting.
\end{proof}
% ---------------------------------------------- Proof of thm 5.2 ----------------------------------------

% \subsection{Proof of Theorem \ref{theorem: utility}}
\section{Proof of Theorem \ref{theorem: utility in TDP}}
In this section, we provide the proof of Theorem \ref{theorem: utility in TDP}  , which gives the utility of Algorithm \ref{alg: DPTD} when achieving $  (\epsilon,\delta) $-DP under trajectory. 
\begin{proof}[Proof of Theorem \ref{theorem: utility in TDP}]
The main proof is similar and the difference  lies in that we inject different Gaussian noises and the variance of gradient is $\sigma^2$. The variance of the Gaussian noise is
\begin{align*}
    \sigma^2_t = \frac{14n^2  G ^2T\alpha'}{m^2  (\epsilon - \frac{\log  (1/\delta) }{\alpha'-1}) },\ \forall t\geq 0\,,
\end{align*}
% $$
% \sigma^2_t = \frac{14n^2  (D_\theta+\epsilon_E) ^2T\alpha'}{m^2  (\epsilon - \frac{\log  (1/\delta) }{\alpha'-1}) }, for\ t\geq 0\,,
% $$
Thus, we start from rebounding the LHS of Eq. (\ref{ineq: main proof 9}) as follows 
\begin{align}\label{ineq: utility with TDP 1}
\frac{576 \sigma^{2}  \sum_{t=0}^{T-1} \nu_{t}^{2}+64d\sum_{t=1}^T\sigma_t^2}{\mu \eta T\nu_T}+\frac{16Q_{0}}{\kappa \nu_T T}
\leq &\frac{144 \sigma^{2}}{\mu \eta T}   (T+b) ^\frac{1}{2}\log  (T+b) \notag\\ &+\frac{3584 n^2  G ^2d T   (T+b) ^{\frac 12} \left  (\frac{\log \left  (\frac{1}{\delta }\right) }{  (1-\beta' )  \epsilon }+1\right) }{\beta'  \eta  \mu  m^2 \epsilon } +\frac{64 Q_{0}  (T+b) ^{1 / 2}}{\kappa T}\notag\\
\leq &\left  (\frac{64 Q_{0}}{\kappa}+\frac{144 \sigma^{2}}{\mu \eta}\right) \left  (\frac{1}{\sqrt{T}}+\frac{\sqrt{b}}{T}\right)  \log   (T+b) \notag\\
&+\frac{3584 n^2  G ^2d T   (T^\frac{1}{2}+b^{\frac 12})  \left  (\frac{\log \left  (\frac{1}{\delta }\right) }{  (1-\beta' )  \epsilon }+1\right) }{\beta'  \eta  \mu  m^2 \epsilon }\,,
\end{align}
% \begin{equation}
% \label{ineq: utility with TDP 1}
% \begin{aligned}
% &\frac{576 \sigma^{2}  \sum_{t=0}^{T-1} \nu_{t}^{2}+64d\sum_{t=1}^T\sigma_t^2}{\mu \eta T\nu_T}+\frac{16Q_{0}}{\kappa \nu_T T}\\
% \leq &\frac{144 \sigma^{2}}{\mu \eta T}   (T+b) ^\frac{1}{2}\log  (T+b)  \\
% &+\frac{3584 n^2  (D_\theta+\epsilon_E) ^2d T   (T+b) ^{\frac 12} \left  (\frac{\log \left  (\frac{1}{\delta }\right) }{  (1-\beta' )  \epsilon }+1\right) }{\beta'  \eta  \mu  m^2 \epsilon }\\ 
% &+\frac{64 Q_{0}  (T+b) ^{1 / 2}}{\kappa T}\\
% \leq &\left  (\frac{64 Q_{0}}{\kappa}+\frac{144 \sigma^{2}}{\mu \eta}\right) \left  (\frac{1}{\sqrt{T}}+\frac{\sqrt{b}}{T}\right)  \log   (T+b) \\
% &+\frac{3584 n^2  (D_\theta+\epsilon_E) ^2d T   (T^\frac{1}{2}+b^{\frac 12})  \left  (\frac{\log \left  (\frac{1}{\delta }\right) }{  (1-\beta' )  \epsilon }+1\right) }{\beta'  \eta  \mu  m^2 \epsilon }\,,
% \end{aligned}
% \end{equation}
where the last inequality is due to $\sqrt{x+y}\leq \sqrt{x}+\sqrt{y}$ for $x\geq 0$, $y\geq 0$.

Combining Eq.  (\ref{ineq: main proof 7}), Eq.  (\ref{ineq: main proof 8})  and Eq.  (\ref{ineq: utility with TDP 1})  shows that
\begin{align*}
\frac{1}{T} \sum_{t=0}^{T-1}& \mathbb{E}  \left(\frac{1}{\kappa}\left\|\widetilde{\theta}_{t+1}-\theta_{t}\right\|+\left\|p_{t}-\nabla_{\theta} F\left  (\theta_{t}, \omega_{t}\right) \right\| +L_{F}\left\|\omega_{t}-\omega^{*}\left  (\theta_{t}\right) \right\|\right) \\
\leq & \sqrt{3}\left  (\frac{576 \sigma^{2}  \sum_{t=0}^{T-1} \nu_{t}^{2}+64d\sum_{t=1}^T\sigma_t^2}{\mu \eta T\nu_T}+\frac{16Q_{0}}{\mu \nu_T T}\right) ^\frac{1}{2}\\
\leq &\sqrt{3}\left  (\frac{64 Q_{0}}{\kappa}+\frac{144 \sigma^{2}}{\mu \eta}\right) ^\frac{1}{2}\left  (\frac{1}{{T}^\frac{1}{4}}+\frac{{b}^\frac{1}{4}}{T^\frac{1}{2}}\right)  \log^\frac{1}{2}   (T+b) +\frac{60\sqrt{3} n  G d^\frac{1}{2}   (T^\frac{3}{4}+b^{\frac 14}T^\frac{1}{2}) \left  (\frac{\log^\frac{1}{2} \left  (\frac{1}{\delta }\right) }{  (1-\beta' ) ^\frac{1}{2} \epsilon^\frac{1}{2} }+1\right) }{{\beta'}^\frac{1}{2}  \eta^\frac{1}{2}  \mu^\frac{1}{2}  m \epsilon^\frac{1}{2}}\\
\leq &\sqrt{3}\left  (\frac{64 Q_{0}}{\kappa}+\frac{144 \sigma^{2}}{\mu \eta}\right) ^\frac{1}{2}\left  (\frac{1}{{T}^\frac{1}{4}}+\frac{{b}^\frac{1}{4}}{T^\frac{1}{2}}\right)  \log^\frac{1}{2}   (T+b) +\frac{120\sqrt{3} n  G  d^\frac{1}{2}  (T^\frac{3}{4}+b^{\frac 14}T^\frac{1}{2}) \log^\frac{1}{2} \left  (\frac{1}{\delta }\right) }{ \eta^\frac{1}{2}  \mu^\frac{1}{2}   {\beta'}^\frac{1}{2}   (1-\beta' ) ^\frac{1}{2} m\epsilon}
\,,
\end{align*}
where the second inequality is due to $\sqrt{x+y}\leq \sqrt{x}+\sqrt{y}$ for $x\geq 0$, $y\geq 0$.

\noindent Simplifying Eq.  (\ref{ineq: main proof 10})  by replacing the constant coefficients with $C_i$ leads to 
\begin{align*}
\frac{1}{T} \sum_{t=0}^{T-1}& \mathbb{E}  \left(\frac{1}{\kappa}\left\|\widetilde{\theta}_{t+1}-\theta_{t}\right\|+\left\|p_{t}-\nabla_{\theta} F\left  (\theta_{t}, \omega_{t}\right) \right\| +L_{F}\left\|\omega_{t}-\omega^{*}\left  (\theta_{t}\right) \right\|\right) \\
\leq &\sqrt{3}\left  (\frac{64 Q_{0}}{\kappa}+\frac{144 \sigma^{2}}{\mu \eta}\right) ^\frac{1}{2}\left  (\frac{1}{{T}^\frac{1}{4}}+\frac{{b}^\frac{1}{4}}{T^\frac{1}{2}}\right)  \log^\frac{1}{2}   (T+b) +\frac{120\sqrt{3} n  G  d^\frac{1}{2}  (T^\frac{3}{4}+b^{\frac 14}T^\frac{1}{2}) \log^\frac{1}{2} \left  (\frac{1}{\delta }\right) }{ \eta^\frac{1}{2}  \mu^\frac{1}{2}   {\beta'}^\frac{1}{2}   (1-\beta' ) ^\frac{1}{2} m\epsilon}\\
= &\frac{C_1\log^\frac{1}{2}   (T+b) }{{T}^\frac{1}{4}}+\frac{C_2\log^\frac{1}{2}   (T+b) }{T^\frac{1}{2}}+\frac{   (C_3T^\frac{3}{4}+C_4T^\frac{1}{2}) nd^\frac{1}{2}\log^\frac{1}{2} \left  (\frac{1}{\delta }\right) }{ m\epsilon}\\
=&\widetilde{\mathcal{O}} \left(\frac{1}{T^\frac{1}{4}}\right) +\mathcal{O}  \left(\frac{T^\frac{3}{4}nd^\frac{1}{2}\log^\frac{1}{2} \left  (1/\delta\right) }{m\epsilon}\right) \,,
\end{align*}
% \begin{equation}
% \begin{aligned}
% &\frac{1}{T} \sum_{t=0}^{T-1} \mathbb{E}  (\frac{1}{\kappa}\left\|\widetilde{\theta}_{t+1}-\theta_{t}\right\|+\left\|p_{t}-\nabla_{\theta} F\left  (\theta_{t}, \omega_{t}\right) \right\| \\
% &+L_{F}\left\|\omega_{t}-\omega^{*}\left  (\theta_{t}\right) \right\|) \\
% \leq &\sqrt{3}\left  (\frac{64 Q_{0}}{\kappa}+\frac{144 \sigma^{2}}{\mu \eta}\right) ^\frac{1}{2}\left  (\frac{1}{{T}^\frac{1}{4}}+\frac{{b}^\frac{1}{4}}{T^\frac{1}{2}}\right)  \log^\frac{1}{2}   (T+b) \\
% &+\frac{120\sqrt{3} n  (D_\theta+\epsilon_E)  d^\frac{1}{2}  (T^\frac{3}{4}+b^{\frac 14}T^\frac{1}{2}) \log^\frac{1}{2} \left  (\frac{1}{\delta }\right) }{ \eta^\frac{1}{2}  \mu^\frac{1}{2}   {\beta'}^\frac{1}{2}   (1-\beta' ) ^\frac{1}{2} m\epsilon}\\
% = &\frac{C_1\log^\frac{1}{2}   (T+b) }{{T}^\frac{1}{4}}+\frac{C_2\log^\frac{1}{2}   (T+b) }{T^\frac{1}{2}}\\
% &+\frac{   (C_3T^\frac{3}{4}+C_4T^\frac{1}{2}) nd^\frac{1}{2}\log^\frac{1}{2} \left  (\frac{1}{\delta }\right) }{ m\epsilon}\\
% =&\widetilde O  (\frac{1}{T^\frac{1}{4}}) +O  (\frac{T^\frac{3}{4}nd^\frac{1}{2}\log^\frac{1}{2} \left  (1/\delta\right) }{m\epsilon}) \,,
% \end{aligned}
% \end{equation}
where $C_1=\sqrt{3}\left  (\frac{64 Q_{0}}{\kappa}+\frac{144 \sigma^{2}}{\mu \eta}\right) ^\frac{1}{2}=\sqrt{3}\left  (\frac{64 Q_{0}}{\kappa}+\frac{288 G^{2}}{\mu \eta}\right) ^\frac{1}{2}$, $C_2=b^\frac{1}{4}C_1$, $C_3=\frac{120\sqrt{3}  G }{  (\beta'  (1-\beta') \eta\mu) ^\frac{1}{2}}, C_4=b^\frac{1}{4}C_3$.

By the parameter setting in Theorem \ref{theorem: utility in TDP}, if we set $T$ as follows
\begin{align*}
T = \frac{C_5m\epsilon}{n\sqrt{d\log  (1/\delta) }}\,,
\end{align*}
% $$
% T = \frac{C_5mn\epsilon}{\sqrt{d\log  (1/\delta) }}\,,
% $$
and hide the factor $\sqrt{\log  (b+T) }$ , the same as in Theorem \ref{theorem: utility}, we will obtain
\begin{align*}
\frac{1}{T} \sum_{t=0}^{T-1} \mathbb{E}\|\mathfrak{M}_t \|
\leq \widetilde{\mathcal{O}}  \left(\frac{n^\frac{1}{4}  (d\log  (1/\delta) ) ^\frac{1}{8}}{  (m\epsilon) ^\frac{1}{4}}\right) \,.
\end{align*}
% \begin{equation}
% \begin{aligned}
% \frac{1}{T} \sum_{t=0}^{T-1} \mathbb{E}\|\mathfrak{M}_t \|
% \leq \widetilde O  (\frac{n^\frac{7}{4}  (d\log  (1/\delta) ) ^\frac{1}{8}}{  (m\epsilon) ^\frac{1}{4}}) \,.
% \end{aligned}
% \end{equation}

\noindent\textbf{Gradient Complexity.} The gradient complexity is equal to $2(T+1)=\mathcal{O}\left(\frac{m\epsilon}{n\sqrt{d\log  (1/\delta) }}\right)$ since Algorithm \ref{alg: DPTD} computes gradients for both the primal side and the dual side. 
\end{proof}

% \subsection{Proof of Lemma \ref{lemma: bounded variance}}
\section{Proof of Technical Lemmas}\label{sec:app:pf_tech_lem}
In this section, we give the detailed proof of several technical lemmas.
\subsection{Proof of Lemma \ref{lemma: bounded variance}}
\textbf{Lemma \ref{lemma: bounded variance}}[Bounded  Variance]
Under Assumption \ref{asmp: G-lispchitz}, the variance of the stochastic gradient $\nabla f  (\theta, \omega; \xi) =\left  (\nabla_{\theta} f  (\theta, \omega; \xi) , \nabla_{\omega} f  (\theta, \omega; \xi) \right) $ is bounded as $\mathbb{E}_{\xi \sim \Xi }\|\nabla f  (\theta, \omega; \xi) -\nabla F  (\theta, \omega) \|^{2} \leq \sigma^{2}$, where $\sigma^2=4G^2$.
% $$
% T = \frac{C_5mn\epsilon}{\sqrt{d\log  (1/\delta) }}\,,
% $$
\begin{proof}
By Assumption \ref{asmp: G-lispchitz}, we have that for any $\theta\in \Theta,\omega\in \Omega$, $\left\|\nabla_\theta f  (\theta, \omega;\xi) \right\|\leq G$ and $\left\|\nabla_\omega f  (\theta, \omega;\xi) \right\|\leq G$.
% \zhao{what does $\nabla_\theta^Tf  (\theta,\omega;\xi) $ stand for? T seems redundant}
We start directly from the LHS,
\begin{align*}
\mathbb{E}_{\xi \sim \Xi }\|\nabla_\theta f  (\theta, \omega; \xi) -\nabla_\theta F  (\theta, \omega) \|^{2}
=&\mathbb{E}_{\xi \sim \Xi }[ \left\| \nabla_\theta f  (\theta,\omega;\xi)  \right\|^2 +\left\| \nabla_\theta F  (\theta,\omega)  \right\|^2 - 2\nabla_\theta^\top f  (\theta,\omega;\xi) \nabla_\theta F  (\theta,\omega)  ]\\
\leq& 2G^2-2\mathbb{E}_{\xi \sim \Xi }\left[  \nabla_\theta^\top f  (\theta,\omega;\xi) \nabla_\theta F  (\theta,\omega) \right]\\
=&2G^2-2\mathbb{E}_{\xi\sim\Xi}[\nabla_\theta^\top f  (\theta,\omega;\xi) ]\nabla_\theta F  (\theta,\omega) \\
\leq&2G^2\,,
\end{align*}
where $\sigma^2=2G^2$.
% $$
% \begin{aligned}
% &\mathbb{E}_{\xi \sim \Xi }\|\nabla_\theta f  (\theta, \omega; \xi) -\nabla_\theta F  (\theta, \omega) \|^{2}\\
% =&\mathbb{E}_{\xi \sim \Xi }[ \left\| \nabla_\theta f  (\theta,\omega;\xi)  \right\|^2 +\left\| \nabla_\theta F  (\theta,\omega)  \right\|^2  \\
% &- 2\nabla_\theta^Tf  (\theta,\omega;\xi) \nabla_\theta F  (\theta,\omega)  ]\\
% \leq& G^2+D_\theta^2-2\mathbb{E}_{\xi \sim \Xi }\left[  \nabla_\theta^Tf  (\theta,\omega;\xi) \nabla_\theta F  (\theta,\omega) \right]\\
% =&G^2+D_\theta^2-2\mathbb{E}_{\xi\sim\Xi}[\nabla_\theta^Tf  (\theta,\omega;\xi) ]\nabla_\theta F  (\theta,\omega) \\
% \leq&G^2+D_\theta^2\,,
% \end{aligned}
% $$
Similarly, $\mathbb{E}_{\xi \sim \Xi }\|\nabla_\omega f  (\theta, \omega; \xi) -\nabla_\omega F  (\theta, \omega) \|^{2}\leq 2G^2$ also holds.
% $$
% \begin{aligned}
% &\mathbb{E}_{\xi \sim \Xi }\|\nabla_\omega f  (\theta, \omega; \xi) -\nabla_\omega F  (\theta, \omega) \|^{2}\\
% =&\mathbb{E}_{\xi \sim \Xi }[ \left\| \nabla_\omega f  (\theta,\omega;\xi)  \right\|^2 +\left\| \nabla_\omega F  (\theta,\omega)  \right\|^2  \\
% &- 2\nabla_\omega^Tf  (\theta,\omega;\xi) \nabla_\omega F  (\theta,\omega)  ]\\
% \leq& G^2+D_\omega^2-2\mathbb{E}_{\xi \sim \Xi }\left[  \nabla_\omega^Tf  (\theta,\omega;\xi) \nabla_\omega F  (\theta,\omega) \right]\\
% =&G^2+D_\omega^2-2\mathbb{E}_{\xi\sim\Xi}[\nabla_\omega^Tf  (\theta,\omega;\xi) ]\nabla_\omega F  (\theta,\omega) \\
% \leq&G^2+D_\omega^2\,,
% \end{aligned}
% $$
Thus, if we set $\sigma^2=2G^2$, we have $\mathbb{E}_{\xi \sim \Xi }\|\nabla f  (\theta, \omega; \xi) -\nabla F  (\theta, \omega) \|^{2} \leq \sigma^{2}$.
\end{proof}

% \subsection{Proof of Lemma \ref{lemma: forth in main proof}}
\subsection{Proof of Lemma \ref{lemma: forth in main proof}}
\textbf{Lemma \ref{lemma: forth in main proof}}[With Bounded  Variance]

\noindent Under Assumptions \ref{asmp:wlipschitz}, \ref{asmp: feasible convex set}, \ref{asmp: strongly concave}, letting $0<\nu_t\leq   (8\alpha) ^{-1}$ and $0<\eta\leq  (4L_F) ^{-1}$, with the updating rules shown in Algorithm \ref{alg: DPTD}, we have
\begin{align}\label{ineq: 1 in proof of lemma}
 \mathbb{E}\left\|\nabla_{\theta} F\left  (\theta_{t+1}, \omega_{t+1}\right) -p_{t+1}\right\|^{2}  \leq& \left  (1-\alpha \nu_{t}\right)  \mathbb{E}\left\|\nabla_{\theta} F\left  (\theta_{t}, \omega_{t}\right) -p_{t}\right\|^{2} +\frac{9 \nu_{t} L_{F}^{2}}{8 \alpha} \mathbb{E}\left  (\left\|\widetilde{\theta}_{t+1}-\theta_{t}\right\|^{2}+\left\|\widetilde{\omega}_{t+1}-\omega_{t}\right\|^{2}\right) \notag\\
 &+\alpha^{2} \nu_{t}^{2} \sigma^{2} + d\sigma_{t+1}^2\,,
\end{align}
and
% \begin{equation}
% \label{ineq: 1 in proof of lemma}
% \begin{aligned}
% & \mathbb{E}\left\|\nabla_{\theta} F\left  (\theta_{t+1}, \omega_{t+1}\right) -p_{t+1}\right\|^{2} \\ \leq& \left  (1-\alpha \nu_{t}\right)  \mathbb{E}\left\|\nabla_{\theta} F\left  (\theta_{t}, \omega_{t}\right) -p_{t}\right\|^{2} \\
% &+\frac{9 \nu_{t} L_{F}^{2}}{8 \alpha} \mathbb{E}\left  (\left\|\widetilde{\theta}_{t+1}-\theta_{t}\right\|^{2}+\left\|\widetilde{\omega}_{t+1}-\omega_{t}\right\|^{2}\right) \\
% &+\alpha^{2} \nu_{t}^{2} \sigma^{2} + d\sigma_{t+1}^2\,,
% \end{aligned}
% \end{equation}
\begin{align}\label{ineq: 2 in proof of lemma}
\mathbb{E}\left\|\nabla_{\omega} F\left  (\theta_{t+1}, \omega_{t+1}\right) -d_{t+1}\right\|^{2}  \leq &\left  (1-\alpha \nu_{t}\right)  \mathbb{E}\left\|\nabla_{\omega} F\left  (\theta_{t}, \omega_{t}\right) -d_{t}\right\|^{2}+\frac{9 \nu_{t} L_{F}^{2}}{8 \alpha} \mathbb{E}\left  (\left\|\widetilde{\theta}_{t+1}-\theta_{t}\right\|^{2}+\left\|\widetilde{\omega}_{t+1}-\omega_{t}\right\|^{2}\right) \notag\\
&+\alpha^{2} \nu_{t}^{2} \sigma^{2} + d\sigma_{t+1}^2\,.
\end{align}
% \begin{equation}
% \label{ineq: 2 in proof of lemma}
% \begin{aligned}
% &\mathbb{E}\left\|\nabla_{\omega} F\left  (\theta_{t+1}, \omega_{t+1}\right) -d_{t+1}\right\|^{2} \\ \leq &\left  (1-\alpha \nu_{t}\right)  \mathbb{E}\left\|\nabla_{\omega} F\left  (\theta_{t}, \omega_{t}\right) -d_{t}\right\|^{2} \\
% &+\frac{9 \nu_{t} L_{F}^{2}}{8 \alpha} \mathbb{E}\left  (\left\|\widetilde{\theta}_{t+1}-\theta_{t}\right\|^{2}+\left\|\widetilde{\omega}_{t+1}-\omega_{t}\right\|^{2}\right) \\
% &+\alpha^{2} \nu_{t}^{2} \sigma^{2} + d\sigma_{t+1}^2\,.
% \end{aligned}
% \end{equation}

\begin{proof}
We first show the detailed proof for Eq. (\ref{ineq: 1 in proof of lemma}) in the lemma, and for the proof of Eq. (\ref{ineq: 2 in proof of lemma}) we will only give a proof sketch since the proof of the two inequalities is similar.

We start from the LHS of Eq. (\ref{ineq: 1 in proof of lemma}). Decompose the term $\nabla_{\theta} F\left  (\theta_{t+1}, \omega_{t+1}\right) -p_{t+1}$ as follows
\begin{align*}
\nabla_{\theta} F\left  (\theta_{t+1}, \omega_{t+1}\right) -p_{t+1}
=& \nabla_{\theta} F\left  (\theta_{t+1}, \omega_{t+1}\right) -\left  (1-\alpha \nu_{t}\right)  p_{t}-u_{t+1}^p-\alpha \nu_{t} \nabla_{\theta} f\left  (\theta_{t+1}, \omega_{t+1} ; \xi_{t+1}\right)  \\
=& \left  (1-\alpha \nu_{t}\right) \left[\nabla_{\theta} F\left  (\theta_{t+1}, \omega_{t+1}\right) -p_{t}\right]-u_{t+1}^p+\alpha \nu_{t}\left[\nabla_{\theta} F\left  (\theta_{t+1}, \omega_{t+1}\right) -\nabla_{\theta} f\left  (\theta_{t+1}, \omega_{t+1} ; \xi_{t+1}\right) \right]\,,
\end{align*}
% $$
% \begin{aligned}
% &\nabla_{\theta} F\left  (\theta_{t+1}, \omega_{t+1}\right) -p_{t+1} \\
% =& \nabla_{\theta} F\left  (\theta_{t+1}, \omega_{t+1}\right) -\left  (1-\alpha \nu_{t}\right)  p_{t}-u_{t+1}^p\\
% &-\alpha \nu_{t} \nabla_{\theta} f\left  (\theta_{t+1}, \omega_{t+1} ; \xi_{t+1}\right)  \\
% =& \left  (1-\alpha \nu_{t}\right) \left[\nabla_{\theta} F\left  (\theta_{t+1}, \omega_{t+1}\right) -p_{t}\right]-u_{t+1}^p\\
% &+\alpha \nu_{t}\left[\nabla_{\theta} F\left  (\theta_{t+1}, \omega_{t+1}\right) -\nabla_{\theta} f\left  (\theta_{t+1}, \omega_{t+1} ; \xi_{t+1}\right) \right]\,,
% \end{aligned}
% $$
where we use the updating rule $p_{t+1} =    (1-\alpha\nu_t) p_t +\alpha\nu_t\nabla_\theta f  (\theta_{t+1},\omega_{t+1};\xi_{t+1}) +u_{t+1}^p$ in Algorithm \ref{alg: DPTD}. 

Taking expectation of the square of the norm on both sides leads to
\begin{align}\label{eq: 3 in proof of lemma}
&\mathbb{E}\left\|\nabla_{\theta} F\left  (\theta_{t+1}, \omega_{t+1}\right) -p_{t+1}\right\|^{2}\notag\\
=&\left  (1-\alpha \nu_{t}\right) ^{2} \mathbb{E}\left\|\nabla_{\theta} F\left  (\theta_{t+1}, \omega_{t+1}\right) -p_{t}\right\|^{2}+\alpha^{2} \nu_{t}^{2} \mathbb{E}\left\|\nabla_{\theta} F\left  (\theta_{t+1}, \omega_{t+1}\right) -\nabla_{\theta} f\left  (\theta_{t+1}, \omega_{t+1} ; \xi_{t+1}\right) \right\|^{2}\notag\\
&+  (2\alpha-2\alpha^2 \nu_{t})   \nu_{t}\times \mathbb{E}\left\langle\nabla_{\theta} F  (\theta_{t+1}, \omega_{t+1}) -p_{t}, \nabla_{\theta} F  (\theta_{t+1}, \omega_{t+1}) -\nabla_{\theta} f  (\theta_{t+1}, \omega_{t+1} ; \xi_{t+1}) \right\rangle\notag\\
&+\mathbb{E}\left\|u_{t+1}^p \right\|^2+2\mathbb{E}\left\langle   (1-\alpha\nu_t) p_t+\alpha\nu_t\nabla_\theta f  (\theta_{t+1},\omega_{t+1};\xi_{t+1}) ,u_{t+1}^p \right\rangle\,.
\end{align}
% \begin{equation}
% \label{eq: 3 in proof of lemma}
% \begin{aligned}
% &\mathbb{E}\left\|\nabla_{\theta} F\left  (\theta_{t+1}, \omega_{t+1}\right) -p_{t+1}\right\|^{2}\\
% =&\left  (1-\alpha \nu_{t}\right) ^{2} \mathbb{E}\left\|\nabla_{\theta} F\left  (\theta_{t+1}, \omega_{t+1}\right) -p_{t}\right\|^{2}\\
% &+\alpha^{2} \nu_{t}^{2} \mathbb{E}\left\|\nabla_{\theta} F\left  (\theta_{t+1}, \omega_{t+1}\right) -\nabla_{\theta} f\left  (\theta_{t+1}, \omega_{t+1} ; \xi_{t+1}\right) \right\|^{2}\\
% &+  (2\alpha-2\alpha^2 \nu_{t})   \nu_{t}\times\\ &\mathbb{E}\langle\nabla_{\theta} F  (\theta_{t+1}, \omega_{t+1}) -p_{t}, \nabla_{\theta} F  (\theta_{t+1}, \omega_{t+1}) -\nabla_{\theta} f  (\theta_{t+1}, \omega_{t+1} ; \xi_{t+1}) \rangle\\
% &+\mathbb{E}\left\|u_{t+1}^p \right\|^2\\
% &+2\mathbb{E}\langle   (1-\alpha\nu_t) p_t+\alpha\nu_t\nabla_\theta f  (\theta_{t+1},\omega_{t+1};\xi_{t+1}) ,u_{t+1}^p \rangle
% \end{aligned}
% \end{equation}
% Note that the third term in the RHS of the above equality is too long to be fit in a single row, thus it is written in two lines.

% Then we aim to simplify the RHS of Eq. (\ref{eq: 3 in proof of lemma}).
It remains to simplify the RHS of Eq. (\ref{eq: 3 in proof of lemma}).
By the Gaussian noise we define in Algorithm \ref{alg: DPTD}, we have
\begin{equation}
\label{eq: 4 in proof of lemma}
\mathbb{E}\left\|u_{t+1}^p \right\|^2 = d\sigma_{t+1}^2\,,
\end{equation}
and
\begin{equation}
\label{eq: 5 in proof of lemma}
    \mathbb{E}\langle   (1-\alpha\nu_t) p_t+\alpha\nu_t\nabla_\theta f  (\theta_{t+1},\omega_{t+1};\xi_{t+1}) ,u_{t+1}^p \rangle = 0\,,
\end{equation}
where the second equality is because $\mathbb{E}[u_{t+1}^p ]=0$. 
% Additionally, by chain rule of expectation \zhao{maybe tower rule of conditional expectation?}, we can obtain

By the tower rule of conditional expectation, one can see that
\begin{align}\label{eq: 6 in proof of lemma}
&\mathbb{E}\left\langle\nabla_{\theta} F\left  (\theta_{t+1}, \omega_{t+1}\right) -p_{t}, \nabla_{\theta} F\left  (\theta_{t+1}, \omega_{t+1}\right) -\nabla_{\theta} f\left  (\theta_{t+1}, \omega_{t+1} ; \xi_{t+1}\right) \right\rangle\notag\\
=& \mathbb{E}\left[\mathbb{E}_{\xi_{t+1}}\left\langle\nabla_{\theta} F\left  (\theta_{t+1}, \omega_{t+1}\right) -p_{t},\nabla_{\theta} F\left  (\theta_{t+1}, \omega_{t+1}\right) -\nabla_{\theta} f\left  (\theta_{t+1}, \omega_{t+1} ; \xi_{t+1}\right) \right\rangle\right ] \notag\\
=& 0\,.
\end{align}
% \begin{equation}
% \label{eq: 6 in proof of lemma}
% \begin{aligned}
% &\mathbb{E}\langle\nabla_{\theta} F\left  (\theta_{t+1}, \omega_{t+1}\right) -p_{t}, \\
% &\nabla_{\theta} F\left  (\theta_{t+1}, \omega_{t+1}\right) -\nabla_{\theta} f\left  (\theta_{t+1}, \omega_{t+1} ; \xi_{t+1}\right) \rangle \\
% =& \mathbb{E}\{\mathbb{E}_{\xi_{t+1}}\langle\nabla_{\theta} F\left  (\theta_{t+1}, \omega_{t+1}\right) -p_{t},\\
% &\nabla_{\theta} F\left  (\theta_{t+1}, \omega_{t+1}\right) -\nabla_{\theta} f\left  (\theta_{t+1}, \omega_{t+1} ; \xi_{t+1}\right) \rangle\} \\
% =& 0\,.
% \end{aligned}
% \end{equation}
Combining Eq. (\ref{eq: 4 in proof of lemma}), Eq. (\ref{eq: 5 in proof of lemma}), Eq. (\ref{eq: 6 in proof of lemma}), we simplify Eq. (\ref{eq: 3 in proof of lemma}) as
\begin{align}\label{eq: 7 in proof of lemma}
\mathbb{E}\left\|\nabla_{\theta} F\left  (\theta_{t+1}, \omega_{t+1}\right) -p_{t+1}\right\|^{2}
=&\left  (1-\alpha \nu_{t}\right) ^{2} \mathbb{E}\left\|\nabla_{\theta} F\left  (\theta_{t+1}, \omega_{t+1}\right) -p_{t}\right\|^{2}+d\sigma_{t+1}^2\notag\\
&+\alpha^{2} \nu_{t}^{2} \mathbb{E}\left\|\nabla_{\theta} F\left  (\theta_{t+1}, \omega_{t+1}\right) -\nabla_{\theta} f\left  (\theta_{t+1}, \omega_{t+1} ; \xi_{t+1}\right) \right\|^{2}\,.
\end{align}
% \begin{equation}
% \label{eq: 7 in proof of lemma}
% \begin{aligned}
% &\mathbb{E}\left\|\nabla_{\theta} F\left  (\theta_{t+1}, \omega_{t+1}\right) -p_{t+1}\right\|^{2}\\
% =&\left  (1-\alpha \nu_{t}\right) ^{2} \mathbb{E}\left\|\nabla_{\theta} F\left  (\theta_{t+1}, \omega_{t+1}\right) -p_{t}\right\|^{2}+d\sigma_{t+1}^2\\
% &+\alpha^{2} \nu_{t}^{2} \mathbb{E}\left\|\nabla_{\theta} F\left  (\theta_{t+1}, \omega_{t+1}\right) -\nabla_{\theta} f\left  (\theta_{t+1}, \omega_{t+1} ; \xi_{t+1}\right) \right\|^{2}\,.
% \end{aligned}
% \end{equation}
Now we bound the first term $\left  (1-\alpha \nu_{t}\right) ^{2} \mathbb{E}\left\|\nabla_{\theta} F\left  (\theta_{t+1}, \omega_{t+1}\right) -p_{t}\right\|^{2}$ in Eq.  (\ref{eq: 7 in proof of lemma})  as follows
\begin{align}\label{eq: 8 in proof of lemma}
&\left  (1-\alpha \nu_{t}\right) ^{2} \mathbb{E}\left\|\nabla_{\theta} F\left  (\theta_{t+1}, \omega_{t+1}\right) -p_{t}\right\|^{2}\notag\\
=&\left  (1-\alpha \nu_{t}\right) ^{2} \mathbb{E}\left\|\nabla_{\theta} F\left  (\theta_{t+1}, \omega_{t+1}\right) -\nabla_{\theta} F\left  (\theta_{t}, \omega_{t}\right) +\nabla_{\theta} F\left  (\theta_{t}, \omega_{t}\right) -p_{t}\right\|^{2}\notag\\
\leq &\left  (1-\alpha \nu_{t}\right) ^{2}\left  (1+\frac{1}{\alpha \nu_{t}}\right)  \mathbb{E}\left\|\nabla_{\theta} F\left  (\theta_{t+1}, \omega_{t+1}\right) -\nabla_{\theta} F\left  (\theta_{t}, \omega_{t}\right) \right\|^{2}+\left  (1-\alpha \nu_{t}\right) ^{2}\left  (1+\alpha \nu_{t}\right)  \mathbb{E}\left\|\nabla_{\theta} F\left  (\theta_{t}, \omega_{t}\right) -p_{t}\right\|^{2}\notag\\
\leq& \frac{9}{8 \alpha \nu_{t}} \mathbb{E}\left\|\nabla_{\theta} F\left  (\theta_{t+1}, \omega_{t+1}\right) -\nabla_{\theta} F\left  (\theta_{t}, \omega_{t}\right) \right\|^{2}+\left  (1-\alpha \nu_{t}\right)  \mathbb{E}\left\|\nabla_{\theta} F\left  (\theta_{t}, \omega_{t}\right) -p_{t}\right\|^{2}\,,
\end{align}
% \begin{equation}
% \label{eq: 8 in proof of lemma}
% \begin{aligned}
% &\left  (1-\alpha \nu_{t}\right) ^{2} \mathbb{E}\left\|\nabla_{\theta} F\left  (\theta_{t+1}, \omega_{t+1}\right) -p_{t}\right\|^{2}\\
% =&\left  (1-\alpha \nu_{t}\right) ^{2} \mathbb{E}\left\|\nabla_{\theta} F\left  (\theta_{t+1}, \omega_{t+1}\right) -\nabla_{\theta} F\left  (\theta_{t}, \omega_{t}\right) +\nabla_{\theta} F\left  (\theta_{t}, \omega_{t}\right) -p_{t}\right\|^{2}\\
% \leq &\left  (1-\alpha \nu_{t}\right) ^{2}\left  (1+\frac{1}{\alpha \nu_{t}}\right)  \mathbb{E}\left\|\nabla_{\theta} F\left  (\theta_{t+1}, \omega_{t+1}\right) -\nabla_{\theta} F\left  (\theta_{t}, \omega_{t}\right) \right\|^{2} \\
% &+\left  (1-\alpha \nu_{t}\right) ^{2}\left  (1+\alpha \nu_{t}\right)  \mathbb{E}\left\|\nabla_{\theta} F\left  (\theta_{t}, \omega_{t}\right) -p_{t}\right\|^{2}\\
% \leq& \frac{9}{8 \alpha \nu_{t}} \mathbb{E}\left\|\nabla_{\theta} F\left  (\theta_{t+1}, \omega_{t+1}\right) -\nabla_{\theta} F\left  (\theta_{t}, \omega_{t}\right) \right\|^{2}\\
% &+\left  (1-\alpha \nu_{t}\right)  \mathbb{E}\left\|\nabla_{\theta} F\left  (\theta_{t}, \omega_{t}\right) -p_{t}\right\|^{2}\,,
% \end{aligned}
% \end{equation}
where the first inequality is by Young's inequality $\|x+y\|^{2} \leq  (1+\lambda) \|x\|^{2}+\left  (1+\lambda^{-1}\right) \|y\|^{2}$ with $\lambda=\alpha \nu_t$, and the second inequality is due to the condition $0<\nu_t\leq  (8\alpha) ^{-1}$ and then
$$
\left  (1-\alpha \nu_{t}\right) ^{2}\left  (1+\frac{1}{\alpha \nu_{t}}\right)  \leq 1+\frac{1}{\alpha \nu_{t}} \leq \frac{9}{8 \alpha \nu_{t}}\,,
$$
and
$$
\left  (1-\alpha \nu_{t}\right) ^{2}\left  (1+\alpha \nu_{t}\right) =1-\alpha \nu_{t}-\alpha^{2} \nu_{t}^{2}+\alpha^{3} \nu_{t}^{3} \leq 1-\alpha \nu_{t}\,.
$$
% Furthermore, by Assumption \ref{asmp:wlipschitz} we have that $\nabla_\theta F  (\theta,\omega) $ is Lipschiz continuous, and the updating rule is $\theta_{t+1}=\theta_t+\nu_t  (\widetilde\theta_t-\theta_t) $ and $\omega_{t+1}=\omega_t+\nu_t  (\widetilde\omega_{t}-\omega_t) $, we obtain
Furthermore, Assumption \ref{asmp:wlipschitz} implies that $\nabla_\theta F  (\theta,\omega) $ is Lipschiz continuous. Recall that the update rule in Alogrithm \ref{alg: DPTD} is $\theta_{t+1}=\theta_t+\nu_t  (\widetilde\theta_t-\theta_t) $ and $\omega_{t+1}=\omega_t+\nu_t  (\widetilde\omega_{t}-\omega_t) $, which further leads to
\begin{align}\label{eq: 9 in proof of lemma}
\frac{9}{8 \alpha \nu_{t}} \mathbb{E}\left\|\nabla_{\theta} F\left  (\theta_{t+1}, \omega_{t+1}\right) -\nabla_{\theta} F\left  (\theta_{t}, \omega_{t}\right) \right\|^{2}\leq& \frac{9 L_{F}^{2}}{8 \alpha \nu_{t}}\left  (\left\|\theta_{t+1}-\theta_{t}\right\|^{2}+\left\|\omega_{t+1}-\omega_{t}\right\|^{2}\right) \notag\\
\leq &\frac{9 L_{F}^{2} \nu_{t}}{8 \alpha}\left  (\left\|\tilde{\theta}_{t+1}-\theta_{t}\right\|^{2}+\left\|\widetilde{\omega}_{t+1}-\omega_{t}\right\|^{2}\right) \,.
\end{align}
% \begin{equation}
% \label{eq: 9 in proof of lemma}
% \begin{aligned}
% &\frac{9}{8 \alpha \nu_{t}} \mathbb{E}\left\|\nabla_{\theta} F\left  (\theta_{t+1}, \omega_{t+1}\right) -\nabla_{\theta} F\left  (\theta_{t}, \omega_{t}\right) \right\|^{2} \\
% \leq& \frac{9 L_{F}^{2}}{8 \alpha \nu_{t}}\left  (\left\|\theta_{t+1}-\theta_{t}\right\|^{2}+\left\|\omega_{t+1}-\omega_{t}\right\|^{2}\right) \\
% \leq &\frac{9 L_{F}^{2} \nu_{t}}{8 \alpha}\left  (\left\|\tilde{\theta}_{t+1}-\theta_{t}\right\|^{2}+\left\|\widetilde{\omega}_{t+1}-\omega_{t}\right\|^{2}\right) \,.
% \end{aligned}
% \end{equation}
Combining Eq.  (\ref{eq: 8 in proof of lemma})  and Eq.  (\ref{eq: 9 in proof of lemma})  leads to
\begin{align}\label{eq: 10 in proof of lemma}
\left  (1-\alpha \nu_{t}\right) ^{2} \mathbb{E}\left\|\nabla_{\theta} F\left  (\theta_{t+1}, \omega_{t+1}\right) -p_{t}\right\|^{2}
\leq& \frac{9 L_{F}^{2}\nu_{t} }{8 \alpha}\left\|\tilde{\theta}_{t+1}-\theta_{t}\right\|^{2}+\frac{9 \nu_{t} L_{F}^{2}}{8 \alpha}\left\|\widetilde{\omega}_{t+1}-\omega_{t}\right\|^{2}\notag\\
&+\left  (1-\alpha \nu_{t}\right)  \mathbb{E}\left\|\nabla_{\theta} F\left  (\theta_{t}, \omega_{t}\right) -p_{t}\right\|^{2}\,,
\end{align}
% \begin{equation}
% \label{eq: 10 in proof of lemma}
% \begin{aligned}
% &\left  (1-\alpha \nu_{t}\right) ^{2} \mathbb{E}\left\|\nabla_{\theta} F\left  (\theta_{t+1}, \omega_{t+1}\right) -p_{t}\right\|^{2} \\
% \leq& \frac{9 L_{F}^{2}\nu_{t} }{8 \alpha}\left\|\tilde{\theta}_{t+1}-\theta_{t}\right\|^{2}+\frac{9 \nu_{t} L_{F}^{2}}{8 \alpha}\left\|\widetilde{\omega}_{t+1}-\omega_{t}\right\|^{2}\\
% &+\left  (1-\alpha \nu_{t}\right)  \mathbb{E}\left\|\nabla_{\theta} F\left  (\theta_{t}, \omega_{t}\right) -p_{t}\right\|^{2}\,,
% \end{aligned}
% \end{equation}
% which provides a contraction for the first term in Eq. (\ref{eq: 7 in proof of lemma}).
which upper bounds the first term of the RHS in Eq. (\ref{eq: 7 in proof of lemma}).

% For the  third term in Eq. (\ref{eq: 7 in proof of lemma}), due to the bounded variance given by Lemma \ref{lemma: bounded variance}, we obtain
For the  third term of the RHS in Eq. (\ref{eq: 7 in proof of lemma}), Lemma \ref{lemma: bounded variance} implies that
\begin{equation}
\label{eq: 11 in proof of lemma}
\alpha^{2} \nu_{t}^{2} \mathbb{E}\left\|\nabla_{\theta} F\left  (\theta_{t+1}, \omega_{t+1}\right) -\nabla_{\theta} f\left  (\theta_{t+1}, \omega_{t+1} ; \xi_{t+1}\right) \right\|^{2} \leq \alpha^{2} \nu_{t}^{2} \sigma^{2}\,.
\end{equation}
% Combining Eq.  (\ref{eq: 7 in proof of lemma}) , Eq.  (\ref{eq: 10 in proof of lemma})  and Eq.  (\ref{eq: 11 in proof of lemma})  shows that
% \begin{align}\label{eq: 1 result in lemma}
% \mathbb{E}\left\|\nabla_{\theta} F\left  (\theta_{t+1}, \omega_{t+1}\right) -p_{t+1}\right\|^{2} \leq&\left  (1-\alpha \nu_{t}\right)  \mathbb{E}\left\|\nabla_{\theta} F\left  (\theta_{t}, \omega_{t}\right) -p_{t}\right\|^{2}+\frac{9 \nu_{t} L_{F}^{2}}{8 \alpha} \mathbb{E}\left  (\left\|\widetilde{\theta}_{t+1}-\theta_{t}\right\|^{2}+\left\|\widetilde{\omega}_{t+1}-\omega_{t}\right\|^{2}\right) \notag\\
% &+d\sigma^2_{t+1} + \alpha^{2} \nu_{t}^{2} \sigma^{2}\,.
% \end{align}
% \begin{equation}
% \label{eq: 1 result in lemma}
% \begin{aligned}
% &\mathbb{E}\left\|\nabla_{\theta} F\left  (\theta_{t+1}, \omega_{t+1}\right) -p_{t+1}\right\|^{2}\\ \leq&\left  (1-\alpha \nu_{t}\right)  \mathbb{E}\left\|\nabla_{\theta} F\left  (\theta_{t}, \omega_{t}\right) -p_{t}\right\|^{2}\\
% &+\frac{9 \nu_{t} L_{F}^{2}}{8 \alpha} \mathbb{E}\left  (\left\|\widetilde{\theta}_{t+1}-\theta_{t}\right\|^{2}+\left\|\widetilde{\omega}_{t+1}-\omega_{t}\right\|^{2}\right) \\
% &+d\sigma^2_{t+1} + \alpha^{2} \nu_{t}^{2} \sigma^{2}\,.
% \end{aligned}
% \end{equation}
% Eq.  (\ref{eq: 1 result in lemma})  is the first equality  we aim to get. For Eq.  (\ref{ineq: 2 in proof of lemma}) , the proof is very similar. Thus we will not repeat here.
% Eq.  (\ref{eq: 1 result in lemma})  is the first equality  we aim to get. 
% For Eq.  (\ref{ineq: 2 in proof of lemma}) , the proof is very similar. Thus we will not repeat here.
Eq. (\ref{ineq: 1 in proof of lemma}) follows from combining Eq.  (\ref{eq: 7 in proof of lemma}) , Eq.  (\ref{eq: 10 in proof of lemma})  and Eq.  (\ref{eq: 11 in proof of lemma}).
The proof of Eq.  (\ref{ineq: 2 in proof of lemma}) is similar to that of Eq. (\ref{ineq: 1 in proof of lemma}) and is omitted here.
\end{proof}

\section{Implementation Details}\label{app:sec:exp_imple}
The parameters of all the algorithms are introduced as follows. The value function is parameterized by a two-layer  fully-connected neural network with  $50$ hidden neurons and ELU activation function \cite{djork2016elu}.
The discount factor $\gamma$ is set to $0.95$ as in \cite{WaiHYWT19}. We set the feasible sets as  $\Theta=[-1,1]^d$ and $\Omega =[-1,1]^d$ where $d$ is the dimension of the neural network's parameters. For DPTD and TD, we set $\alpha=3$, $\beta=3$, $\kappa=2$, $\eta=2$, $\nu_t=\frac{1}{4(t+3)^{1/2}}$ as suggested in Theorem \ref{theorem: utility}.
The step sizes of DPGLD and DPSRM are also taken as the suggested theoretical values in their original papers. For SGD, the step size 
is maintained in the same order with other algorithms, ranging from $10^{-3}$ to $10^{-4}$.
We implement all the algorithms in PyTorch 1.5.1  \cite{paszke2019pytorch} with Ubuntu 18.04 and an NVIDIA GTX 2080Ti GPU. 

\clearpage
\section{Additional Experiments}\label{app:sec:exp_varying_eps}
We conduct experiments to show the impact of different privacy budgets on the convergence of DPTD. Specifically, we plot the utilities of DPTD under $\epsilon=100.0$, 
$\epsilon=10.0$, $\epsilon=1.0$ and $\epsilon=0.1$ respectively.
As $\epsilon$ increases, DPTD will have more privacy budgets and
the primal and dual gradients of DPTD will be perturbed by Gaussian noises with smaller variance. Therefore, DPTD under larger $\epsilon$ will have better utility, which is validated in Figure \ref{fig: DP}, where DPTD under $\epsilon=100.0$
has the best utility in all three tasks.
Furthermore, 
the utility of DPTD will degrade as $\epsilon$ decreases
and DPTD under $\epsilon=0.1$ has the worst utility.

% ---------------------------------------  fig 2 ---------------------------------------------------
\begin{figure*}[!htbp]
\begin{minipage}[t]{0.33\textwidth}
  \includegraphics[width=0.95\linewidth]{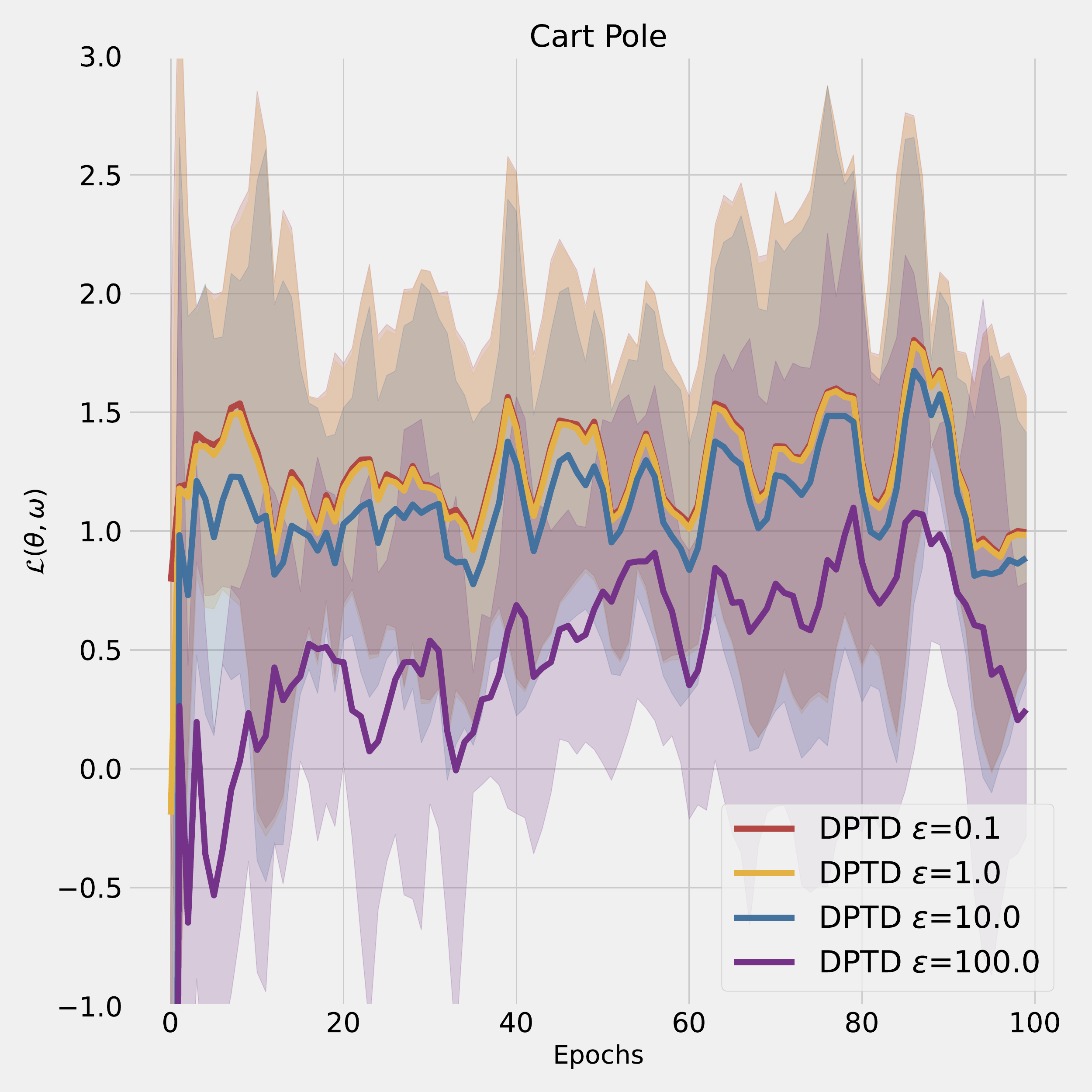}
%   \subcaption{}
   \label{fig: DP mountain car}
\end{minipage}%
\hfill
\begin{minipage}[t]{0.33\textwidth}

  \includegraphics[width=0.95\linewidth]{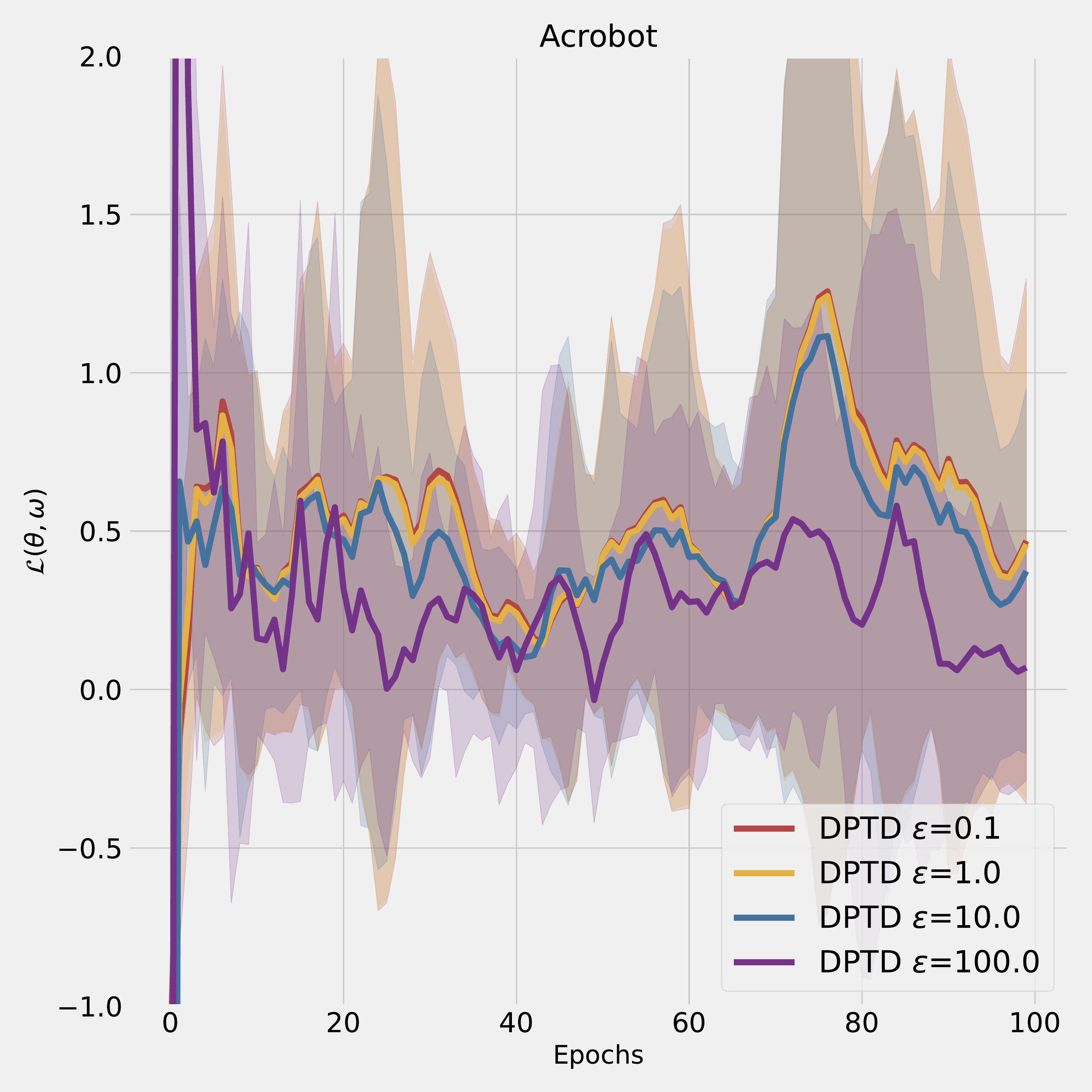}
%   \subcaption{}
   \label{fig: DP cart pole}
\end{minipage}%
\hfill
\begin{minipage}[t]{0.33\textwidth}
  \includegraphics[width=0.95\linewidth]{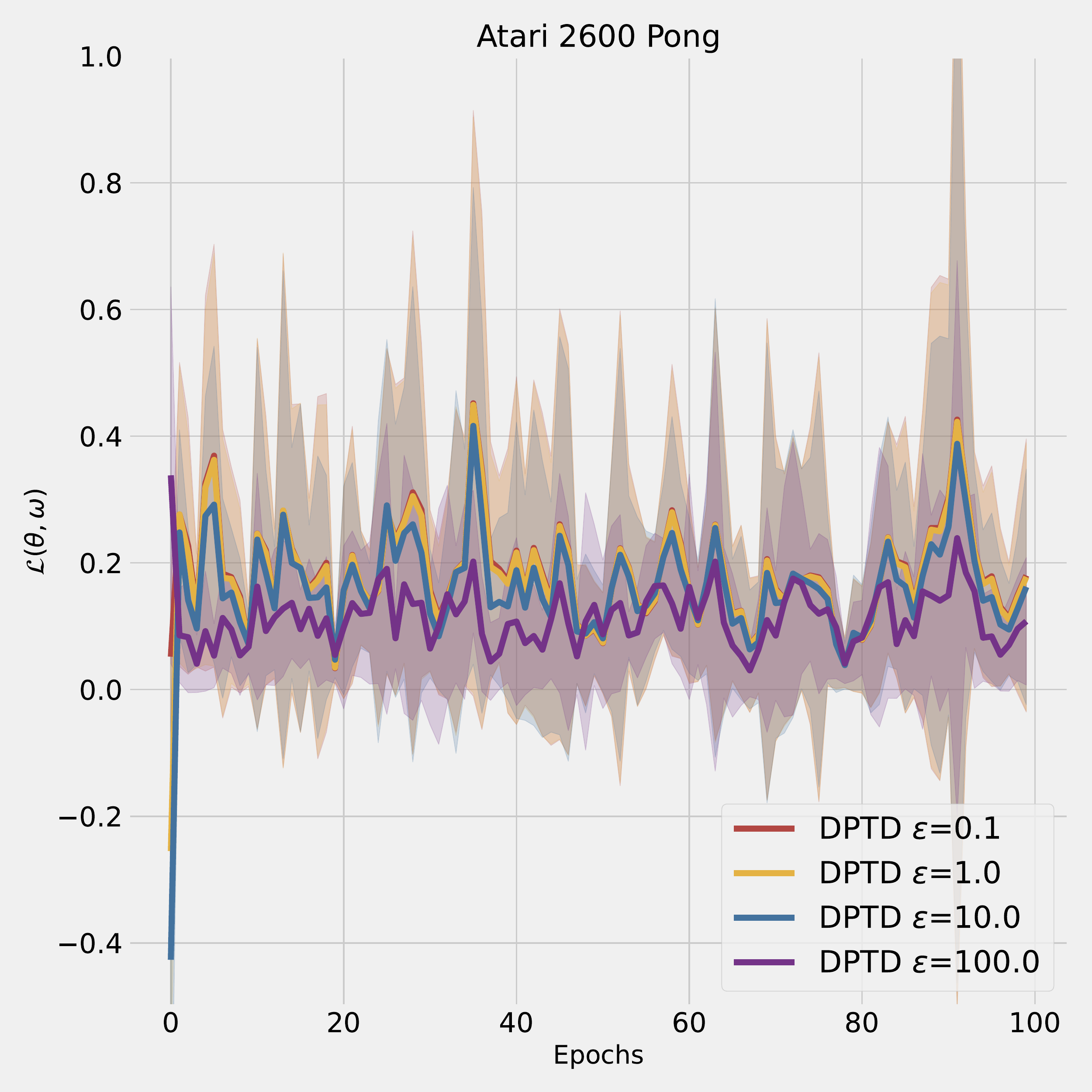}
%   \subcaption{}
   \label{fig: DP acrobot}

\end{minipage}
\caption{
% Compare DPTD under different privacy parameter $\epsilon$, including $\epsilon=0.1, 1.0, 10.0$. Each epoch has $5$ finite trajectories. The shadow denotes $1$-std. The learning curves are averaged over 10 random seeds. The curves are generated without smoothing.
Utilities of DPTD under different privacy parameter $\epsilon$, including $\epsilon=0.1, 1.0, 10.0, 100.0$. Each epoch has $5$ finite trajectories. The shadow denotes $1$-std. The learning curves are averaged over 10 random seeds and are generated without smoothing.
}
\label{fig: DP}
\end{figure*}
% ---------------------------------------  fig 2 ---------------------------------------------------

\end{document}